%% file: morwani22.tex
\documentclass[final,12pt]{alt2022} % Anonymized submission 
\input{0_macros.tex}
\title[Inductive Bias of Weight Normalization]{Inductive Bias of Gradient Descent for Weight Normalized Smooth Homogeneous Neural Nets}
\usepackage{times}
\usepackage{enumitem}
\altauthor{%
 \Name{Depen Morwani} \Email{depenmorwani@gmail.com}\\
 \Name{Harish G. Ramaswamy} \Email{hariguru@cse.iitm.ac.in}\\
 \addr {Department of Computer Science and Engineering, and RBCDSAI,  \\
 Indian Institute of Technology Madras, India.}
}

\begin{document}

\maketitle

\begin{abstract}%
  We analyze the inductive bias of gradient descent for weight normalized smooth homogeneous neural nets, when trained on exponential or cross-entropy loss. We analyse both standard weight normalization (SWN) and exponential weight normalization (EWN), and show that the gradient flow path with EWN is equivalent to gradient flow on standard networks with an adaptive learning rate. We extend these results to gradient descent, and establish asymptotic relations between weights and gradients  for both SWN and EWN. We also show that EWN causes weights to be updated in a way that prefers asymptotic relative sparsity. For EWN, we  provide a finite-time convergence rate of the loss with gradient flow and a tight asymptotic convergence rate with gradient descent. We demonstrate our results for SWN and EWN on synthetic data sets. Experimental results on simple datasets support our claim on sparse EWN solutions, even with SGD. This demonstrates its potential applications in learning neural networks amenable to pruning.
\end{abstract}


\input{1_intro}
\input{2_problemsetup}
\input{3_sparsitybias}
\input{4_convergencerates}
\input{5_experiments}
\input{6_proofsketch}
\input{7_conclusion}

\bibliography{alt2022-sample}

\appendix

\input{8_appendix}
\end{document}

%% file: 0_macros.tex
\newtheorem*{lemma*}{Lemma}
\newtheorem*{definition*}{Definition}
\newtheorem*{theorem*}{Theorem}
\newtheorem*{corollary*}{Corollary}
\newtheorem*{note*}{Note}
\newtheorem*{property*}{Property}
\newtheorem*{assumption*}{Assumption}
\newtheorem*{caution*}{Caution}
\newtheorem*{proposition*}{Proposition}

\newcommand{\w}{\mathbf{w}}
\newcommand{\g}{\mathbf{g}}
\newcommand{\bv}{\mathbf{v}}

\newcommand{\x}{\mathbf{x}}

\newcommand{\br}{\mathbf{r}}
\newcommand{\ba}{\mathbf{a}}

\newcommand{\0}{\mathbf{0}}
\newcommand{\z}{\mathbf{z}}

\newcommand{\R}{\mathbb{R}}

\newcommand{\C}{\mathcal C}
\renewcommand{\L}{\mathcal L}

\newcommand{\balpha}{\boldsymbol \alpha}

\newcommand{\bell}{\boldsymbol{\ell}}
\newcommand{\btheta}{\boldsymbol \theta}

%% file: 1_intro.tex
\section{Introduction}

The prevailing hypothesis for explaining the generalization ability of deep neural nets, despite their ability to fit even random labels \citep{Zhang+17}, is that the optimisation/training algorithms such as gradient descent have a `bias' towards `simple' solutions. This property is often called inductive bias, and has been an active research area over the past few years.

% This has led to an interest in understanding the inductive bias\citep{Neyshabur+15} of the training algorithm that drives the network towards low complexity solutions. The inductive bias of gradient descent (GD) and stochastic gradient descent (SGD) for standard logistic regression have been explored in great detail \citep{Soudry+18, Gunasekar+18a, Nacson+19c}. These results have  been extended partially to deep linear nets \citep{JiTelgrasky19a} and deep homogeneous nets \citep{LyuLi20}. 

It has been shown that gradient descent does indeed seem to prefer `simpler' solutions over more `complex' solutions, where the notion of complexity is often problem/architecture specific. The predominant line of work typically shows that gradient descent prefers a least norm solution in some variant of the $L_2$-norm. This is satisfying, as gradient descent over the parameters abides by the rules of $L_2$ geometry, i.e.  the weight vector moves along direction of steepest descent, with length measured using the Euclidean norm. However, there is nothing special about the Euclidean norm in the parameter space, and hence several other notions of `length' and `steepness' are equally valid. In recent years, several alternative parameterizations of the weight vector, such as Batch normalization and Weight normalization, have seen immense success and these do not seem to respect $L_2$ geometry in the `weight space'. We pose the question of inductive bias of gradient descent for some of these parameterizations, and demonstrate interesting inductive biases. In particular, it can still be argued that gradient descent with these reparameterizations prefers simpler solutions, but the notion of complexity is different. 

\textbf{Our Contributions}.
The main contributions of the paper are as follows.
\begin{itemize}
    \item We establish that the gradient flow path with exponential weight normalization is equal to the gradient flow path of an unnormalized network using an adaptive neuron dependent learning rate.  This provides a crisp description of the difference between exponential weight normalized networks and unnormalized networks.
    
    \item While most of the previous works on the inductive bias of non-linear deep learning architectures work under the assumption of directional convergence of weights and gradients, we show that gradient convergence implies weight convergence, even for gradient descent, for both standard and exponentially weight normalized network (which is not homogeneous in its parameters)\footnote{The assumptions about weights and gradients converging in direction have been recently shown to hold for gradient flow on homogeneous neural nets without normalization under some regularity conditions related to o-minimality of the architecture\citep{JiTelgarsky20}}.
    
    \item We establish the asymptotic relations between weights and gradients for gradient descent on  standard weight normalized and exponentially weight normalized networks and show that exponential weight normalization is likely to lead to asymptotic sparsity in weights. We demonstrate the relative sparsity of exponential weight normalization on MNIST dataset, by showing that it leads to networks with better pruning efficacy.
    
    \item We establish finite-time convergence rates for gradient flow and tight asymptotic convergence rates for gradient descent on exponentially weight normalized networks.
    
    % \item We provide tight asymptotic convergence rates for exponentially weight normalized networks.
\end{itemize}

\section{Related Work}
The literature most closely related to this paper can be broadly classified into two categories - the inductive biases established for neural networks, and the theoretical studies of normalization methods in 
deep learning.
\subsection{Inductive Bias}
 \citet{Soudry+18} showed that gradient descent(GD) on the logistic loss with linearly separable data converges to the $L_2$ maximum margin solution for almost all datasets. These results were extended to loss functions with super-polynomial tails in  \citet{Nacson+19a}. \citet{Nacson+19c} extended these results to hold for stochastic gradient descent(SGD) and \citet{Gunasekar+18a} extended the results for other optimization geometries. \citet{JiTelgarsky19c} provided tight convergence bounds in terms of dataset size as well as training time. \citet{JiTelgarsky19b} provide similar results when the data is not linearly separable.
 
 \citet{JiTelgrasky19a} showed that for deep linear nets, under certain conditions on the initialization, for almost all linearly separable datasets, the network, in function space, converges to the maximum margin solution. \citet{Gunasekar+18b} established that for linear convolutional nets, under certain assumptions regarding convergence of gradients etc, the function converges to a KKT point of the maximum margin problem in fourier space. \citet{Nacson+19b} shows that for smooth homogeneous nets, the network converges to a KKT point of the maximum margin problem in parameter space. \citet{LyuLi20} established these results with weaker assumptions and also provide asymptotic convergence rates for the loss. \citet{ChizatBach20} explore the inductive bias for a 2-layer infinitely wide ReLU neural net in function space and show that the function learnt is a max-margin classifier for variation norm. \cite{Moroshko+20} established the inductive bias for linear diagonal networks and showed that the network transitions between maximum $L_2$-norm and $L_1$-norm margin, depending on the relation between the initialization scale and training accuracy.
 
% \begin{figure}[t]
% \floatconts{neighborhoods}{\caption{$L_2$ neighborhoods with $\epsilon = 0.5$ radius in parameter space for different parameterizations. For EWN in the left (resp. SWN in the middle) the parameter $[\gamma,\bv]$ (the parameter $[\alpha, \bv]$ resp.) is restricted to a 3-d ball of radius $\epsilon$ and the values that the 2-d weight vector $\w$ takes is illustrated for 6 different centers.}}
% {\subfigure[EWN]{
%   \centering
%   \includegraphics[width=0.3\textwidth]{plots/rich_richer/exp-WN_0.5.pdf}
%   }
%   \subfigure[SWN]{
%   \centering
%   \includegraphics[width=0.3\textwidth]{plots/rich_richer/gamma-WN_0.5.pdf}
%   }
%   \subfigure[Unnorm]{
%   \centering
%   \includegraphics[width=0.3\textwidth]{plots/rich_richer/no-WN_0.5.pdf}
%   }
%   }
% \end{figure}

 \subsection{Normalization} 
 \citet{SalimansKnigma16} introduced weight normalization and demonstrated that it replicates the convergence speedup of BatchNorm. Similarly, other normalization techniques have been proposed as well\citep{Ba+16}\citep{Qiao+20}\citep{Li+19}, but only a few have been theoretically explored. \citet{Santurkar+18} demonstrated that batch normalization makes the loss surface smoother and $L_2$ normalization in batchnorm can even be replaced by $L_1$ and $L_\infty$ normalizations. \citet{Kohler+19} showed that for GD, batchnorm speeds up convergence in the case of GLM by splitting the optimization problem into learning the direction and the norm. \citet{Cai+19} analyzed GD on BN for squared loss and showed that it converges for a wide range of lr. \citet{Bjorck+18} showed that the primary reason BN allows networks to achieve higher accuracy is by enabling higher learning rates. \citet{Arora+19} showed that in case of GD or SGD with batchnorm, lr for scale-invariant parameters does not affect the convergence rate towards stationary points. \citet{Du+18} showed that for GD over one-hidden-layer weight normalized CNN, with a constant probability over initialization, iterates converge to global minima. \citet{Qiao+19} compared different normalization techniques from the perspective of whether they lead to points, where neurons are consistently deactivated. \citet{Wu+20} established the inductive bias of gradient flow with weight normalization for overparameterized least squares and showed that for a wider range of initializations as compared to normal parameterization, it converges to the minimum $L_2$ norm solution. \citet{Dukler+20} analyzed weight normalization for multilayer ReLU net in the infinite width regime and showed that it may speedup convergence. Some other papers\citep{Luo+19, Roburin+20} also provide other perspectives to think about normalization techniques.  

%% file: 2_problemsetup.tex
\section{Problem Setup}
We use a standard view of neural networks as a collection of nodes/neurons grouped by layers. Each node $u$ is associated with a weight vector $\w_u$, that represents the incoming weight vector for that node. In case of CNNs, weights can be shared across different nodes. $\w$ represents all the parameters of the network arranged in form of a vector (In general, for any vector $\bv$ associated with the entire network, $\bv_u$ represent its components corresponding to the node $u$). The training dataset consists $(\x_i, y_i)$ pairs with a total of $m$ points in the dataset. The function represented by the neural network is denoted by $\Phi(\w, .)$. The loss for a single data point $\x_i$ is given by $\ell(y_i, \Phi(\w, \x_i))$ and the loss vector is represented by $\bell$. The overall loss is represented by $\L(\w)$ and is given by $\L(\w) = \sum_{i=1}^m \ell(y_i, \Phi(\w, \x_i))$. We sometimes abbreviate $\L(\w(t))$ as $\L$ when the context is clear.

In standard weight normalisation (SWN), each weight vector $\w_u$ is reparameterized as $\gamma_u \frac{\bv_u}{\| \bv_u \|}$. This was proposed by \citet{SalimansKnigma16}, as a substitute for Batch Normalization and has been practically used in multiple papers such as \citet{Sokolic+17}, \citet{Dauphin+17}, \citet{Jin-Hwa+18} and \citet{Hiebler+18}. The corresponding update equations for gradient descent are given by
\begin{align}
\gamma_u(t+1) &= \gamma_u(t) - \eta(t)\frac{\bv_u(t)^\top \nabla_{\w_u} \L}{\| \bv_u(t) \|}  \label{eq:gamma:SWN:gd}\\
\bv_u(t+1) &= \bv_u(t) - \eta(t)\frac{\gamma_u(t)}{\| \bv_u(t) \|}\left(I - \frac{\bv_u(t)\bv_u(t)^\top}{\| \bv_u(t) \|^2}\right)\nabla_{\w_u} \L \label{eq:v:SWN:gd}
\end{align}
In exponential weight normalisation (EWN), each weight vector $\w_u$ is reparameterized as $e^{\alpha_u}\frac{\bv_u}{\| \bv_u \|}$. This was mentioned in \citet{SalimansKnigma16}, but to the best of our knowledge, has not been widely used. The corresponding update equations for gradient descent with learning rate $\eta(t)$ are given by
\begin{align}
\alpha_u(t+1) &= \alpha_u(t) - \eta(t)e^{\alpha_u(t)} \frac{\bv_u(t)^\top \nabla_{\w_u} \L}{\| \bv_u(t) \|} \label{eq:alpha:EWN:gd}\\
\bv_u(t+1) &= \bv_u(t) - \eta(t)\frac{e^{\alpha_u(t)}}{\| \bv_u(t) \|}\left(I - \frac{\bv_u(t)\bv_u(t)^\top}{\| \bv_u(t) \|^2}\right)\nabla_{\w_u} \L \label{eq:v:EWN:gd}
\end{align}

The update equations for gradient flow are the continuous counterparts for the same. In gradient flow, for both SWN and EWN, we set $\| \bv_u(0) \| = 1$, to simplify the update equations.

  %\caption{EWN}
% \end{subfigure}%
% \begin{subfigure}{.33\textwidth}
%   \centering
%   \includegraphics[width=\textwidth]{plots/rich_richer/gamma-WN_0.5.pdf}
%   \caption{SWN}
%   %\label{fig:test2}
% \end{subfigure}
% \begin{subfigure}{.33\textwidth}
%   \centering
%   \includegraphics[width=\textwidth]{plots/rich_richer/no-WN_0.5.pdf}
%   \caption{Unnorm}
%   %\label{fig:test2}
% \end{subfigure}
% \caption{$L_2$ neighborhoods with $\epsilon = 0.5$ radius in parameter space for different parameterizations. For EWN in the left (resp. SWN in the middle) the parameter $[\gamma,\bv]$ (the parameter $[\alpha, \bv]$ resp.) is restricted to a 3-d ball of radius $\epsilon$ and the values that the 2-d weight vector $\w$ takes is illustrated for 6 different centers.}
% %As we can see, with increasing $\| \w \|$, although the neighborhood for NWN is unaffected, but the neighborhood of EWN expands radially. This leads to a 'rich gets richer' phenomenon, where weights having higher norm tend to grow faster, leading to asymptotic relative sparsity} 
% \label{neighborhoods}
% \end{figure}

%% file: 3_sparsitybias.tex
\section{Inductive Bias of Weight Normalization}

In this section, we state our main results for weight normalized smooth homogeneous models on exponential loss($\ell(y_i, \Phi(\w, \x_i)) = e^{-y_i\Phi(\w, \x_i)}$). The results for cross-entropy loss and proofs have been deferred to the appendix due to space constraints. First, we state the main proposition that helps in establishing these results for EWN.

\begin{theorem}
\label{theorem:1}
 The gradient flow path with learning rate $\eta(t)$ for EWN and SWN are given as follows:
 \begin{align}
 \text{EWN: } \frac{d \w_u(t)}{dt} &= -\eta(t)\| \w_u(t) \|^2 \nabla_{\w_u} \L \label{eq:w:EWN:gf}\\
 \text{SWN: } \frac{d \w_u(t)}{dt}  &= -\eta(t)\left(\|\w_u(t)\|^2 \nabla_{\w_u} \L + \left(\frac{1 - \|\w_u(t)\|^2}{\|\w_u(t)\|^2}\right) (\w_u(t)^\top\nabla_{\w_u} \L) \w_u(t)\right) \label{eq:w:SWN:gf}
 \end{align}
\end{theorem}

Thus, the gradient flow path of EWN can be replicated by an adaptive (neuron dependent) learning rate given by $\eta(t) \| \w_u(t) \|^2$ on the unnormalized network(Unnorm). % These parameterizations also induce different neighborhoods in the parameter space, that have been shown in Figure~\ref{neighborhoods}. 

\subsection{Assumptions}\label{assumptions}
The assumptions in the paper can be broadly divided into loss function/architecture based assumptions and trajectory based assumptions. The loss functions/architecture based assumptions are shared across both gradient flow and gradient descent.

\textbf{Loss function/Architecture based assumptions}
\begin{enumerate}[label=\arabic*]
    \item $\ell(y_i, \Phi(\w, \x_i)) = e^{-y_i\Phi(\w, \x_i)}$
    \item $\Phi(., \x)$ is a $\C^1$ function (i.e. continuously differentiable), for any fixed $\x$
    \item $\Phi(\lambda \w, \x) = \lambda^L \Phi(\w, \x)$, for some $\lambda > 0$ and $L > 0$
\end{enumerate}

\textbf{Gradient flow}. For gradient flow, we make the following trajectory based assumptions
\begin{enumerate}[label=(A\arabic*)]
\item There exists a time $t_0$ such that $\L(\w(t_0)) < 1$.
\item $\lim_{t \to \infty} \frac{ -\nabla_\w \L(\w(t))}{\| \nabla_\w \L(\w(t)) \|} := \widetilde{\g}$.
    % \begin{minipage}{0.5\textwidth}
    % \item $\lim_{t \to \infty} \L(\w(t)) = 0$
    % \item $\lim_{t \to \infty} \frac{\w(t)}{\| \w(t) \|} := \widetilde \w$
    % \end{minipage}
    % \begin{minipage}{0.5\textwidth}
    % \item $\lim_{t \to \infty} \frac{ -\nabla_\w \L(\w(t))}{\| \nabla_\w \L(\w(t)) \|} := \widetilde{\g}$
    % % \item $\lim_{t \to \infty}\frac{\bell(\w(t))}{\| \bell(\w(t)) \|} := \widetilde{\bell}$
    % % \item Let $\rho = \min_i y_i\Phi(\widetilde \w, \x_i)$. Then $\rho > 0$. 
    % \end{minipage}
\end{enumerate}

The first trajectory assumption is simply a separability assumption and means that the network is able to correctly classify the dataset at some point during the training process. This is not a completely unreasonable assumption, given recent papers demonstrating neural networks with sufficient overparameterization can fit even random labels \citep{Zhang+17,Jacot+18}. The second assumption has been used in multiple previous works \citep{Gunasekar+18b,ChizatBach20,Nacson+19b}, and is standard in the literature related to the inductive bias of non-linear deep learning architectures. Moreover, we remove one of the assumptions related to directional convergence of weights used in these works, and instead show that it is implied by the directional convergence of gradients.

\textbf{Gradient Descent}. For gradient descent, we require the learning rate $\eta(t)$ to not grow too fast, and a slightly stronger assumption on loss.
\begin{enumerate}[label=(B\arabic*)]
    \begin{minipage}{0.5\textwidth}
    \item $\lim_{t \to \infty} \L(\w(t)) = 0$
    \end{minipage}
    \begin{minipage}{0.5\textwidth}
    \item $\lim_{t \to \infty} \frac{ -\nabla_\w \L(\w(t))}{\| \nabla_\w \L(\w(t)) \|} := \widetilde{\g}$
    % \item $\lim_{t \to \infty}\frac{\bell(\w(t))}{\| \bell(\w(t)) \|} := \widetilde{\bell}$
    % \item Let $\rho = \min_i y_i\Phi(\widetilde \w, \x_i)$. Then $\rho > 0$. 
    \end{minipage}
    \item $\lim_{t \to \infty} \eta(t) \| \w_u(t) \| \nabla_{\w_u} \L(\w(t)) \| = 0$ for all $u$ in the network.
\end{enumerate}

The assumption (B3) is mild, as the norm of the gradient of the exponential loss goes down exponentially fast as compared to norm of the weights. We demonstrate that these assumptions hold for multiple datasets including MNIST in Appendix \ref{assum:verif}.

\subsection{Asymptotic relations between weights and gradients}
This section contains the main theorems that establish asymptotic relations between weights and gradients for SWN and EWN. First, we will state a common proposition for both SWN and EWN.

\begin{proposition}
\label{prop:lzero:gf}
Under assumption (A1) for gradient flow, for both SWN and EWN, $\lim_{t \to \infty} \L(\w(t)) = 0$.
\end{proposition}
Although the above proposition was established for homogeneous nets by \citet{LyuLi20}\footnote{Homogeneous networks in the $\w$ space are also homogeneous in the $\gamma, \bv$ space. Therefore results  regarding convergence rates and monotonic margin hold from \cite{LyuLi20}. However, the results for convergence to a KKT point of the max margin problem do not hold. For details, refer Appendix \ref{swn:ll}}, we extend it for the non-homogeneous parameterization of EWN. Now, we provide one of our main theorem that establishes gradient convergence implies weight convergence.

\begin{theorem}
\label{thm:weightconv}
    Consider a node $u$ in the network with $\| \widetilde{\g}_u \| > 0$ and $\lim_{t \to \infty} \| \w_u(t) \| = \infty$. Under assumptions (A1), (A2) for gradient flow and (B1)-(B3) for gradient descent, for both SWN and EWN
    \begin{enumerate}[label=(\roman*)]
        \begin{minipage}{0.5\textwidth}
        \item $\lim_{t \to \infty} \frac{\w_u(t)}{\| \w_u(t) \|} := \widetilde{\w}_u$ exists.
        \end{minipage}
        \begin{minipage}{0.5\textwidth}
        \item $\widetilde{\w}_u = \lambda \widetilde \g_u$ for some $\lambda > 0$.
        \end{minipage}
    \end{enumerate}
    % Let $\widetilde \w_u=\lim_{t\to\infty} \frac{\w_u(t)}{\|\w(t)\|}$ and $\widetilde \g_u=\lim_{t\to\infty} \frac{-\nabla_{\w_u} \L(\w(t))}{\|\nabla_{\w} \L(\w(t))\|}$.
    % Under assumptions (A1)-(A3) for gradient flow and (A1)-(A4) for gradient descent, for both SWN and EWN, for all nodes $u$ in the network, $\| \widetilde \w_u \| > 0, \| \widetilde \g_u \| > 0 \implies \widetilde \w_u = \lambda \widetilde \g_u $ for some $\lambda > 0$. 
    % the following holds:
    % \begin{enumerate}[label=(\roman*)]
    %     \item $\lim_{t\to\infty} \frac{-\nabla_{\w} \L(\w(t))}{\|\nabla_{\w} \L(\w(t))\|} = \mu \sum_{i=1}^m \widetilde \ell_i y_i \nabla_\w \Phi(\widetilde \w, \x_i)
    % =  \widetilde \g  $, where $\mu > 0$. 
    
    % \item Let $\widetilde \w_u=\lim_{t\to\infty} \frac{\w_u(t)}{\|\w(t)\|}$ and $\widetilde \g_u=\lim_{t\to\infty} \frac{-\nabla_{\w_u} \L(\w(t))}{\|\nabla_{\w} \L(\w(t))\|}$. Then, $\widetilde \w_u = \lambda \widetilde \g_u $ for some $\lambda\geq 0$
    % \end{enumerate}
    % Additionally, for SWN, we can say 
    % \begin{enumerate}[label=3]
    %     \item $ \| \widetilde \w_u \| = 0, \| \widetilde \g_u \| >0 \implies \lim_{t \to \infty} \bv_u(t)^\top \widetilde \g_u = 0$
    % \end{enumerate}
\end{theorem}
The above theorem relaxes one of the assumptions regarding weight convergence used in many of the previous works, by showing that even for non-homogeneous parameterization under gradient descent, gradient convergence implies weight convergence. Moreover, it also shows, that weights and gradients eventually get aligned opposite to each other.
% The proposition basically states that if weights and gradients for a given node $u$ converge in direction, then eventually they get aligned opposite to each other. 
% We can prove a slightly stronger statement under a mild additional assumption.

% \begin{proposition}
% \label{prop:conv_dir:strong}
%     Assume that there exists at least one node $v$ in the network satisfying $\| \widetilde \w_v \| > 0$ and $\| \widetilde \g_v \| > 0$. Under assumptions (A1)-(A3) for gradient flow and (A1)-(A4) for gradient descent, for both SWN and EWN, $\widetilde \w_u = \lambda \widetilde \g_u$ for some $\lambda \geq 0$. 
%     % the following holds:
%     % \begin{enumerate}[label=(\roman*)]
%     %     \item $\lim_{t\to\infty} \frac{-\nabla_{\w} \L(\w(t))}{\|\nabla_{\w} \L(\w(t))\|} = \mu \sum_{i=1}^m \widetilde \ell_i y_i \nabla_\w \Phi(\widetilde \w, \x_i)
%     % =  \widetilde \g  $, where $\mu > 0$. 
    
%     % \item Let $\widetilde \w_u=\lim_{t\to\infty} \frac{\w_u(t)}{\|\w(t)\|}$ and $\widetilde \g_u=\lim_{t\to\infty} \frac{-\nabla_{\w_u} \L(\w(t))}{\|\nabla_{\w} \L(\w(t))\|}$. Then, $\widetilde \w_u = \lambda \widetilde \g_u $ for some $\lambda\geq 0$
%     % \end{enumerate}
%     % Additionally, for SWN, we can say 
%     % \begin{enumerate}[label=3]
%     %     \item $ \| \widetilde \w_u \| = 0, \| \widetilde \g_u \| >0 \implies \lim_{t \to \infty} \bv_u(t)^\top \widetilde \g_u = 0$
%     % \end{enumerate}
% \end{proposition}
\begin{figure}[t]
\floatconts{demons_lin_sep}{\caption{\textbf{Demonstration of Results for EWN in \texttt{Lin-Sep} experiment: } (a) Evolution of $\| \w_u \|$ - norm of the incoming weights for neuron $u$ (b) Cosine between weights and gradients for neurons 5, 7 and 8. (c) Weight and gradient norms for weights 5, 7 and 8.}}  
% The three graphs are plotted at loss values of $e^{-200}, e^{-250}$ and $e^{-300}$ respectively. At each loss value, for the 3 weights, $\log \| \w_u \| + \log \| \nabla_{\w_u} \L\|$ is approximately same.}}
{
% \subfigure{
%   \includegraphics[width=0.4\textwidth]{plots/lin_sep_dataset/Assumptions/dataset.pdf}
% }
\subfigure{
  \includegraphics[width=0.3\textwidth]{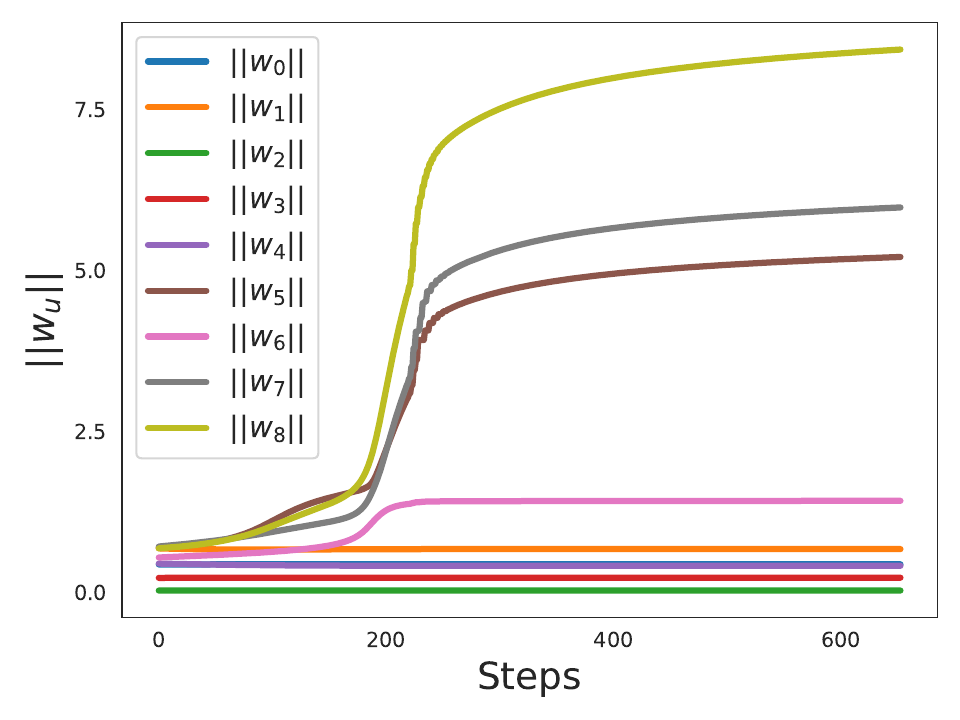}
}
% \subfigure{
%   \includegraphics[width=0.4\textwidth]{plots/lin_sep_dataset/Theorems/l_tilde_grad.pdf}
% }
\subfigure{
  \includegraphics[width=0.3\textwidth]{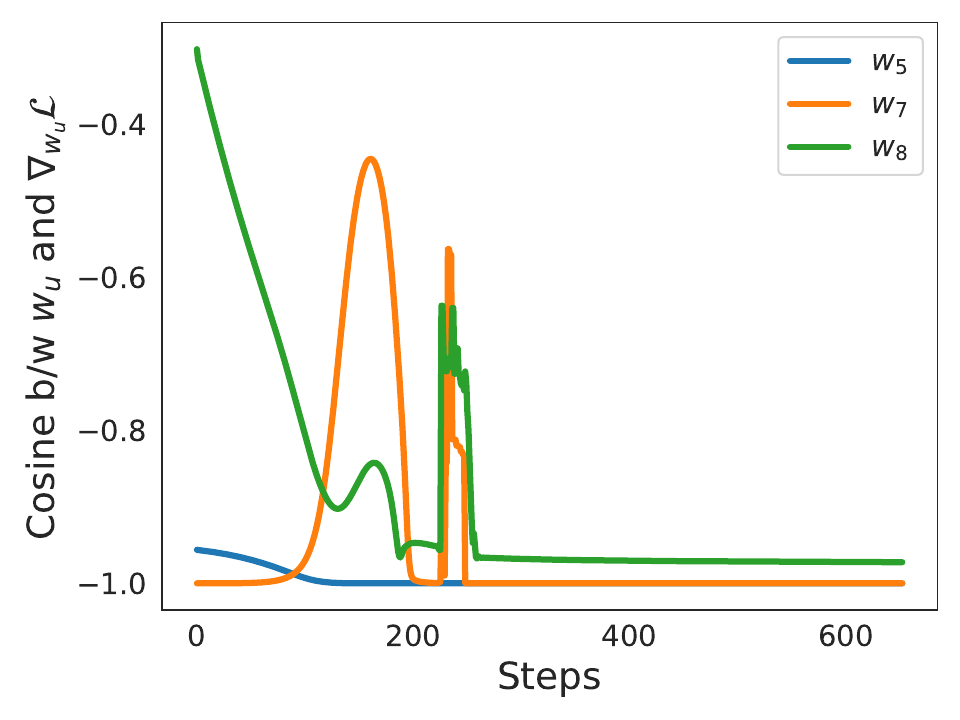}
}
% \subfigure{
%   \includegraphics[width=0.3\textwidth]{plots/lin_sep_dataset/Theorems/weight_grad_ratios_200.pdf}
% }
% \subfigure{
%   \includegraphics[width=0.3\textwidth]{plots/lin_sep_dataset/Theorems/weight_grad_ratios_250.pdf}
%}
\subfigure{
  \includegraphics[width=0.3\textwidth]{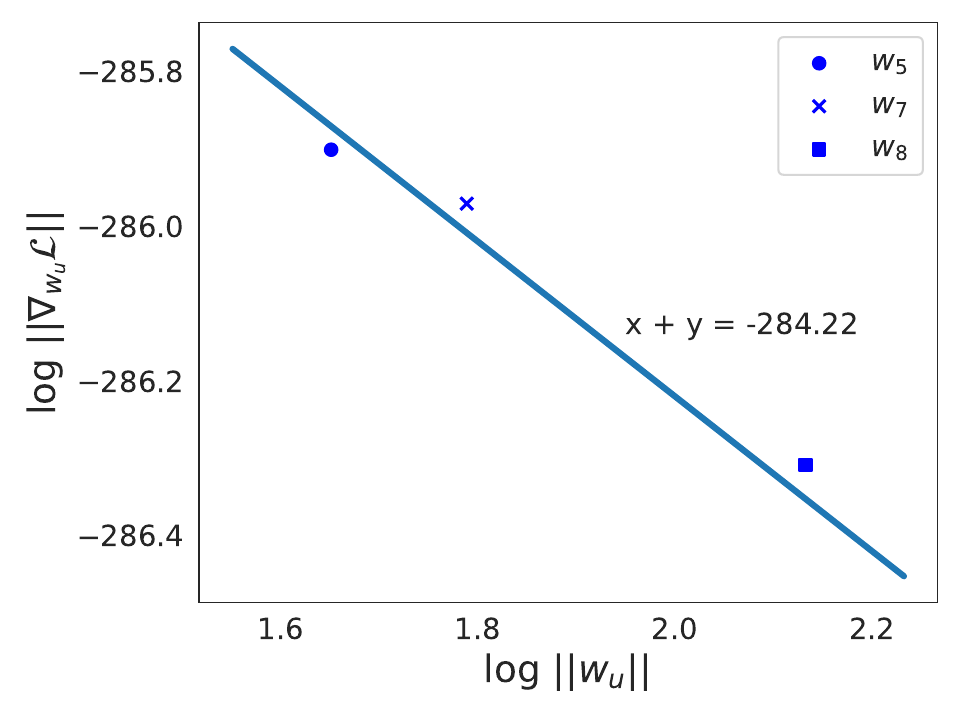}
}
}
\end{figure}

Now, we provide the main theorem that distinguishes SWN and EWN.

\begin{theorem}
\label{thm:weightnorm}
    Consider two nodes $u$ and $v$ in the network with $\| \widetilde \g_u \| \geq \|\widetilde \g_v \|>0, \lim_{t \to \infty} \| \w_u(t) \| = \infty$ and $\lim_{t \to \infty} \| \w_v(t) \| = \infty$. Let $\frac{\| \widetilde \g_u \|}{\| \widetilde \g_v \|}$ be denoted by $c$. Under assumptions (A1), (A2) for gradient flow and (B1)-(B3) for gradient descent,
    % Under assumptions (A1)-(A3) for gradient flow and (A1)-(A4) for gradient descent, for all nodes $u$ and $v$ in the network, the following hold
    \begin{enumerate}[label=(\roman*)]
        \item for SWN, $\lim_{t\to\infty} \frac{\|\w_u(t)\|}{\|\w_v(t)\|}=c $
        \item for EWN, $\lim_{t\to\infty} \frac{\|\w_u(t)\|}{\|\w_v(t)\|}$ is either $0, \infty$ or $\frac{1}{c}$
    \end{enumerate}
\end{theorem}

Thus, if $\frac{\|\w_u(t)\|}{\|\w_v(t)\|}$ converges to a finite non-zero value for EWN, then
$\| \w_u(t) \| \| \nabla_{\w_u} \L(\w(t)) \| = \| \w_v(t) \| \| \nabla_{\w_v} \L(\w(t)) \| = k_1(t)$ asymptotically. While, for SWN, $\frac{\| \w_u(t) \|}{\| \nabla_{\w_u} \L(\w(t)) \|} = \frac{\| \w_v(t) \|}{\| \nabla_{\w_v} \L(\w(t)) \|} = k_2(t)$ asymptotically, where $k_1(t)$ and $k_2(t)$ are independent of $u$ and $v$. The exact conditions under which $\frac{\|\w_u(t)\|}{\|\w_v(t)\|}$ tends to 0 or $\infty$ for EWN are provided in Proposition \ref{prop:3}.

We demonstrate Theorem \ref{thm:weightconv} and Theorem \ref{thm:weightnorm} for EWN on a linearly separable dataset (\texttt{Lin-Sep}) in Figure \ref{demons_lin_sep}. In this experiment, a 2-layered neural network, with 8 neurons in the hidden layer and a ReLU-squared activation function, is trained on a linearly separable dataset. The learning rate schedule used was $O\left(\frac{1}{\L^{0.97}}\right)$ and the network was trained till a loss of $e^{-300}$. As can be seen in Figure \ref{demons_lin_sep}, for weights 5, 7 and 8, whose norms keep on growing, weights and gradients eventually become oppositely aligned, and their norms are inversely proportional to each other. 
The results for SWN, along with results on other datasets including MNIST have been deferred to Appendix \ref{app:swn}.

\subsection{Sparsity Inductive Bias for Exponential Weight Normalisation} \label{subsec:sparsitybiasEWN}

The inverse relation between $\| \w_u(t) \|$ and $\| \nabla_{\w_u} \L(\w(t)) \|$ in the EWN trajectory results in an interesting inductive bias that favours movement along sparse directions. 
% We make this precise by making a a few extra stylistic assumptions in the proposition below.

% \begin{proposition}
% \label{prop:3}
%     Consider two nodes $u$ and $v$ in the network such that $\| \widetilde \g_u \| > 0, \| \widetilde \g_v \| > 0$, and $\| \w_u(t) \|, \| \w_v(t) \| \to \infty$. Let $\frac{\| \widetilde \g_u\|}{\| \widetilde \g_v \|}$ be denoted by $c$. Let $\epsilon, \delta > 0$, then, there always exists a time $t_1$, such that for all $t > t_1$ we have the following properties: Firstly, $\frac{\| \nabla_{\w_u} \L(\w(t))\|}{\| \nabla_{\w_v} \L(\w(t))\|} \in [c - \epsilon, c + \epsilon]$, secondly $\left( \frac{\w_u(t)}{\| \w_u(t) \|}\right)^\top \left(\frac{-\nabla_{\w_u} \L(\w(t))}{\| \nabla_{\w_u} \L(\w(t)) \|}\right) \geq \cos(\delta)$, and thirdly $\left( \frac{\w_v(t)}{\| \w_v(t) \|}\right)^\top \left(\frac{-\nabla_{\w_v} \L(\w(t))}{\| \nabla_{\w_v} \L(\w(t)) \|}\right) \geq \cos(\delta)$. Then the following holds:
%     \begin{itemize}
%         \item for SWN, $\lim_{t \to \infty} \frac{\| \w_u(t) \|}{\| \w_v(t) \|} = c$
%         \item for EWN, if at some time $t_2 > t_1$,
%         \begin{enumerate}[label=(\roman*)]
%             \item $\frac{\| \w_u(t_2) \|}{\| \w_v(t_2) \|} > \frac{1}{(c-\epsilon)\cos(\delta)} \implies \lim_{t \to \infty} \frac{\| \w_u(t) \|}{\| \w_v(t) \|} = \infty$
%             \item $\frac{\| \w_u(t_2) \|}{\| \w_v(t_2) \|} < \frac{\cos(\delta)}{c + \epsilon} \implies \lim_{t \to \infty} \frac{\| \w_u(t) \|}{\| \w_v(t) \|} = 0$
%         \end{enumerate}
%     \end{itemize}
% \end{proposition}

\begin{figure}[t]
\floatconts{1_dim_weight}{\caption{(a) Network architecture for the \texttt{Simple-Traj} experiment . (b) Trajectories of the two weights for EWN and Unnorm, starting from 5 different initialization points.}}  
% \centering
% \begin{subfigure}{0.4\textwidth}
%   \centering
%   \includegraphics[width=\textwidth]{plots/1_dim_weight/Fig1.pdf}
%   \label{1_dim_weight_a}
%   \caption{Network Architecture}
% \end{subfigure}
% \begin{subfigure}{0.4\textwidth}
%   \centering
%   \includegraphics[width=\textwidth]{plots/1_dim_weight/trajectories.pdf}
%   \label{1_dim_weight_b}
%   \caption{Weight trajectories}
% \end{subfigure}
% \caption{(a) Network architecture for the \texttt{Simple-Traj} experiment . (b) Trajectories of the two weights for EWN and Unnorm, starting from 5 different initialization points.}
% \label{1_dim_weight}
% \end{figure}
% \floatconts{1_dim_weight}{}
{
\subfigure[Network Architecture]{
  \label{1_dim_weight_a}
  \includegraphics[width=0.4\textwidth]{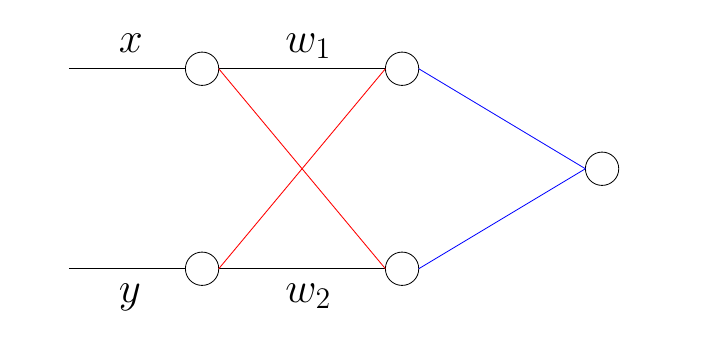}
}
\subfigure[Weight trajectories]{
  \includegraphics[width=0.4\textwidth]{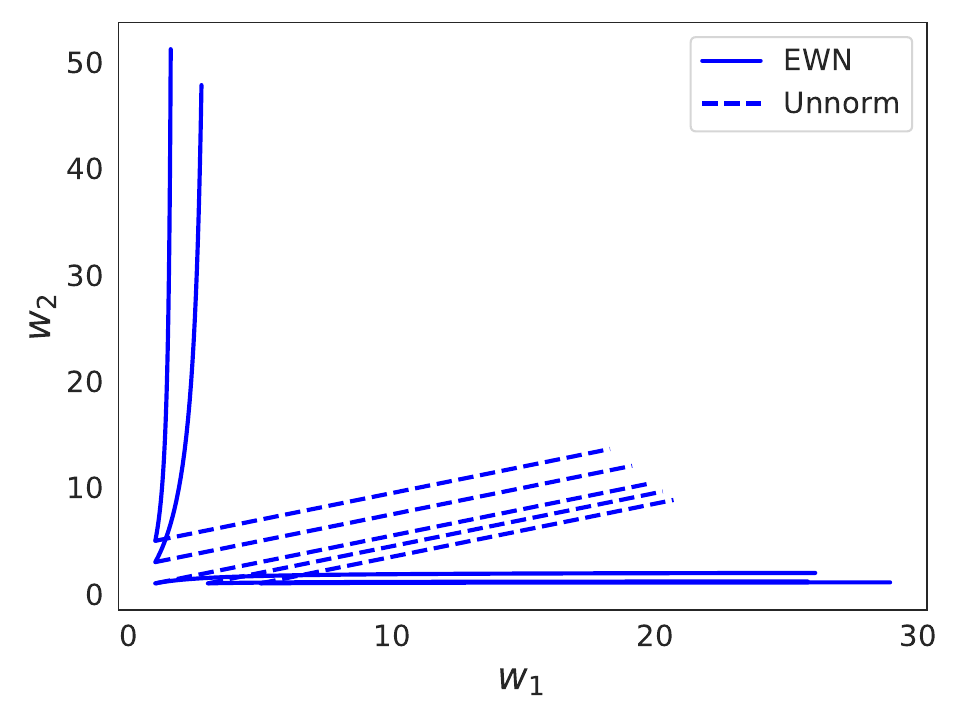}
  \label{1_dim_weight_b}
}
}
\end{figure}

\begin{proposition}
\label{prop:3}
    Consider two nodes $u$ and $v$ in the network such that $\| \widetilde \g_v \| \geq \| \widetilde \g_u \| > 0$ and $\| \w_u(t) \|, \| \w_v(t) \| \to \infty$. Let $\frac{\| \widetilde \g_u\|}{\| \widetilde \g_v \|}$ be denoted by $c$ (note that $\widetilde \g$ and $c$ will be different for SWN and EWN trajectory). Consider any $\epsilon, \delta$ such that $0< \epsilon < c$ and $0<\delta<\frac{\pi}{2}$.  Then, the following holds:
    \begin{enumerate}[label=(\roman*)]
        \item There exists a time $t_1$, such that for all $t > t_1$ both SWN and EWN trajectories have the following properties: 
        \begin{enumerate}
        \begin{minipage}{0.4\textwidth}
        \item $\frac{\| \nabla_{\w_u} \L(\w(t))\|}{\| \nabla_{\w_v} \L(\w(t))\|} \in [c - \epsilon, c + \epsilon]$
        \end{minipage}
        \begin{minipage}{0.5\textwidth}
        \item $\left( \frac{\w_u(t)}{\| \w_u(t) \|}\right)^\top \left(\frac{-\nabla_{\w_u} \L(\w(t))}{\| \nabla_{\w_u} \L(\w(t)) \|}\right) \geq \cos(\delta)$
        \end{minipage}
        \item $\left( \frac{\w_v(t)}{\| \w_v(t) \|}\right)^\top \left(\frac{-\nabla_{\w_v} \L(\w(t))}{\| \nabla_{\w_v} \L(\w(t)) \|}\right) \geq \cos(\delta)$.
        \end{enumerate}
        \item for SWN, $\lim_{t \to \infty} \frac{\| \w_u(t) \|}{\| \w_v(t) \|} = c$
        \item for EWN, if at some time $t_2 > t_1$,
        \begin{enumerate}
            \item $\frac{\| \w_u(t_2) \|}{\| \w_v(t_2) \|} > \frac{1}{(c-\epsilon)\cos(\delta)} \implies \lim_{t \to \infty} \frac{\| \w_u(t) \|}{\| \w_v(t) \|} = \infty$
            \item $\frac{\| \w_u(t_2) \|}{\| \w_v(t_2) \|} < \frac{\cos(\delta)}{c + \epsilon} \implies \lim_{t \to \infty} \frac{\| \w_u(t) \|}{\| \w_v(t) \|} = 0$
        \end{enumerate}
    \end{enumerate}
\end{proposition}
The above proposition shows that the limit property of the weights in Theorem \ref{thm:weightnorm}, makes non-sparse $\w$ an unstable convergent direction for EWN. But that is not the case for SWN. We demonstrate the relative sparsity between EWN, SWN and Unnorm through two toy experiments -- \texttt{Simple-Traj} and \texttt{XOR}. 

% \begin{figure}[t]
% % \begin{subfigure}{0.4\textwidth}
% %   \centering
% %   \includegraphics[width=\textwidth]{plots/XOR_dataset/dataset.pdf}
% %   \label{SWN:XOR_a}
% %   \caption{Dataset}
% % \end{subfigure}%
% \begin{subfigure}{0.3\textwidth}
%   \centering
%   \includegraphics[width=\textwidth]{plots/XOR_dataset/exp_weight_norm.pdf}
%   \label{SWN:XOR_b}
%   \caption{EWN}
% \end{subfigure}
% \begin{subfigure}{0.3\textwidth}
%   \centering
%   \includegraphics[width=\textwidth]{plots/XOR_dataset/no_weight_norm.pdf}
%   \label{SWN:XOR_c}
%   \caption{Unnorm}
% \end{subfigure}%
% \begin{subfigure}{0.3\textwidth}
%   \centering
%   \includegraphics[width=\textwidth]{plots/XOR_dataset/gamma_weight_norm.pdf}
%   \label{SWN:XOR_d}
%   \caption{SWN}
% \end{subfigure}%
% \caption{(a) shows the XOR dataset. (b), (c) and (d) demonstrate that EWN weights grow sparsely when compared to Unnorm and SWN}
% \label{SWN:XOR}
% \end{figure}

\begin{figure}[t]
\floatconts{SWN:XOR}{\caption{(a), (b) and (c) demonstrate the evolution of weight norms for each neuron in the \texttt{XOR} experiment. EWN weights grow sparsely when compared to Unnorm and SWN}}
{
% \subfigure[Dataset]{
%   \includegraphics[width=0.4\textwidth]{plots/XOR_dataset/dataset.pdf}
%   \label{SWN:XOR_a}
% }
\subfigure[EWN]{
  \includegraphics[width=0.3\textwidth]{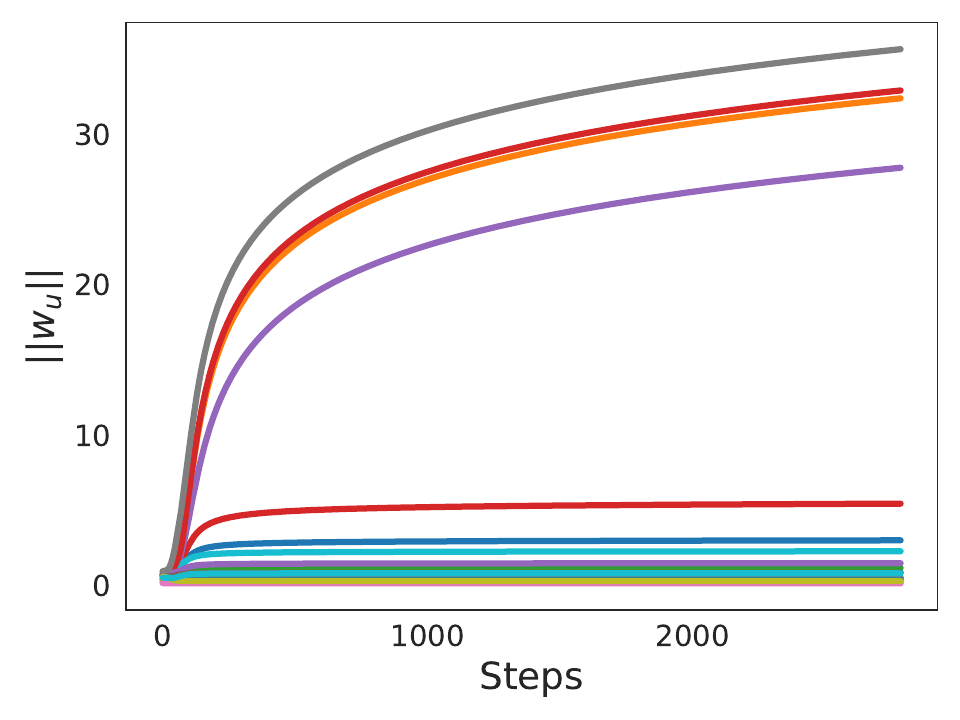}
  \label{SWN:XOR_b}
}
\subfigure[SWN]{
  \includegraphics[width=0.3\textwidth]{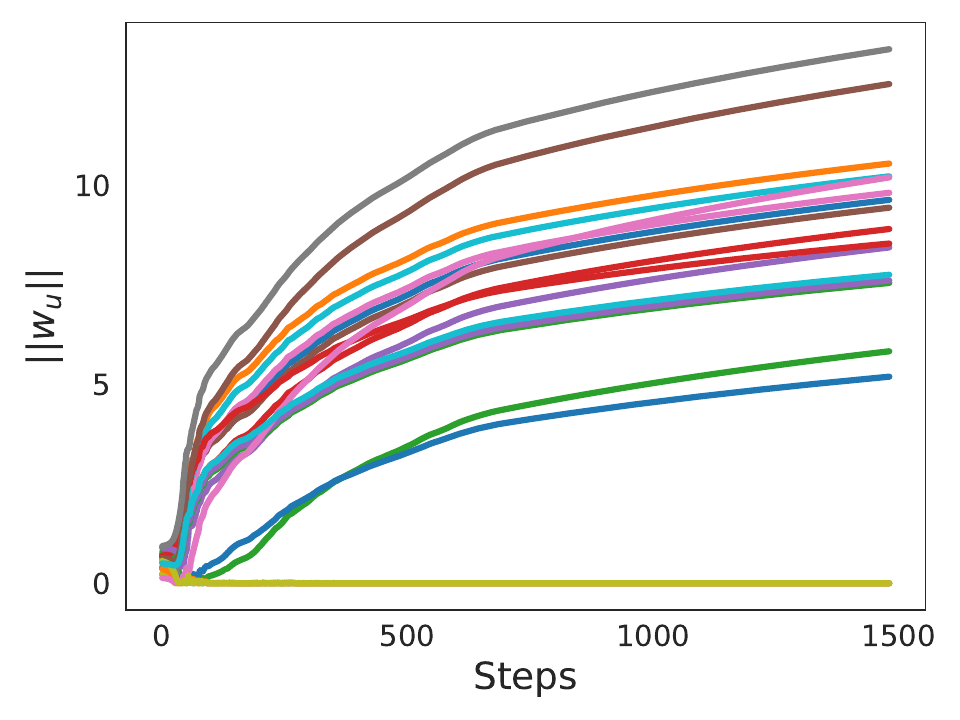}
  \label{SWN:XOR_c}
}
\subfigure[Unnorm]{
  \includegraphics[width=0.3\textwidth]{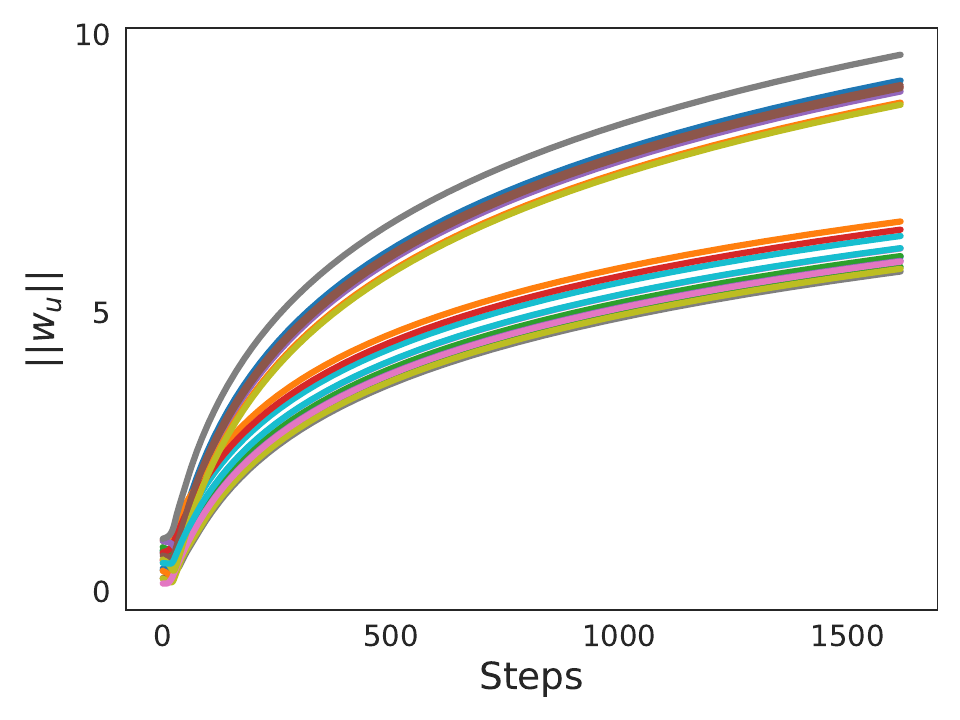}
  \label{SWN:XOR_d}
}
}
\end{figure}

% \begin{figure}[t]
% \begin{subfigure}{0.3\textwidth}
% \centering
%   \includegraphics[width=\textwidth]{plots/conv_rate_linear/WN_train_loss.pdf}
%   \caption{EWN}
% \end{subfigure}
% \begin{subfigure}{0.3\textwidth}
% \centering
%   \includegraphics[width=\textwidth]{plots/conv_rate_linear/gamma-WN_train_loss.pdf}
%   \caption{SWN}
% \end{subfigure}
% \begin{subfigure}{0.3\textwidth}
% \centering
%   \includegraphics[width=\textwidth]{plots/conv_rate_linear/no-WN_train_loss.pdf}
%   \caption{Unnorm}
% \end{subfigure}
% \caption{Variation of convergence rate of train loss with number of layers for multilayer linear nets}
% \label{conv_rate_variation}
% \end{figure}
    
% % \end{figure}
\begin{figure}[t]
\floatconts{conv_rate_variation}{\caption{Variation of convergence rate of train loss with number of layers for multilayer linear nets on a linearly separable dataset}}
{
\subfigure[EWN]{
  \includegraphics[width=0.3\textwidth]{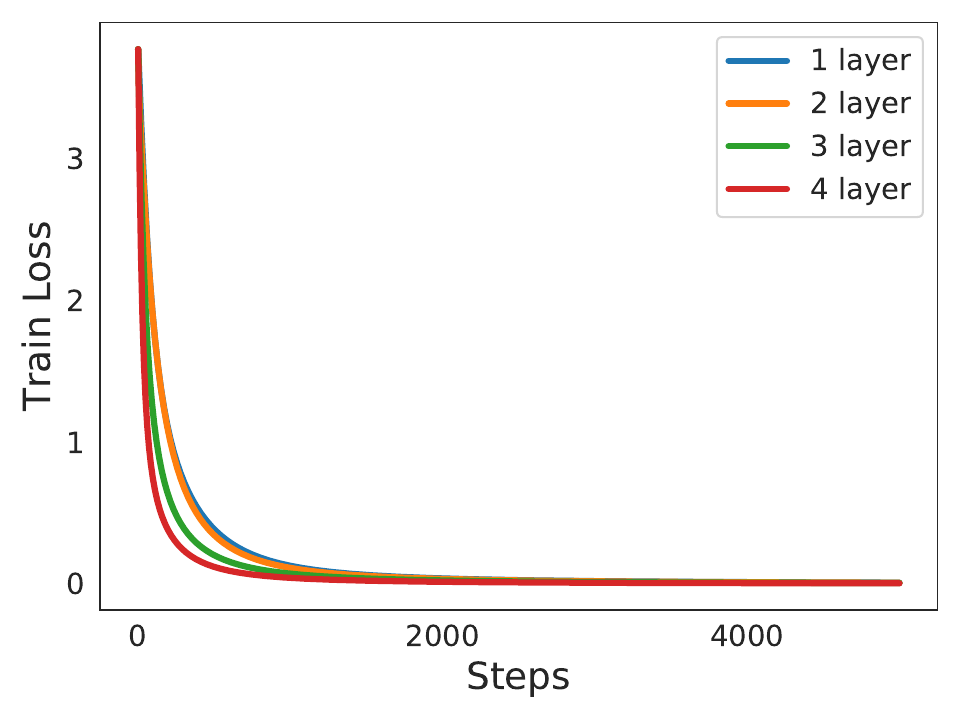}
}
\subfigure[SWN]{
  \includegraphics[width=0.3\textwidth]{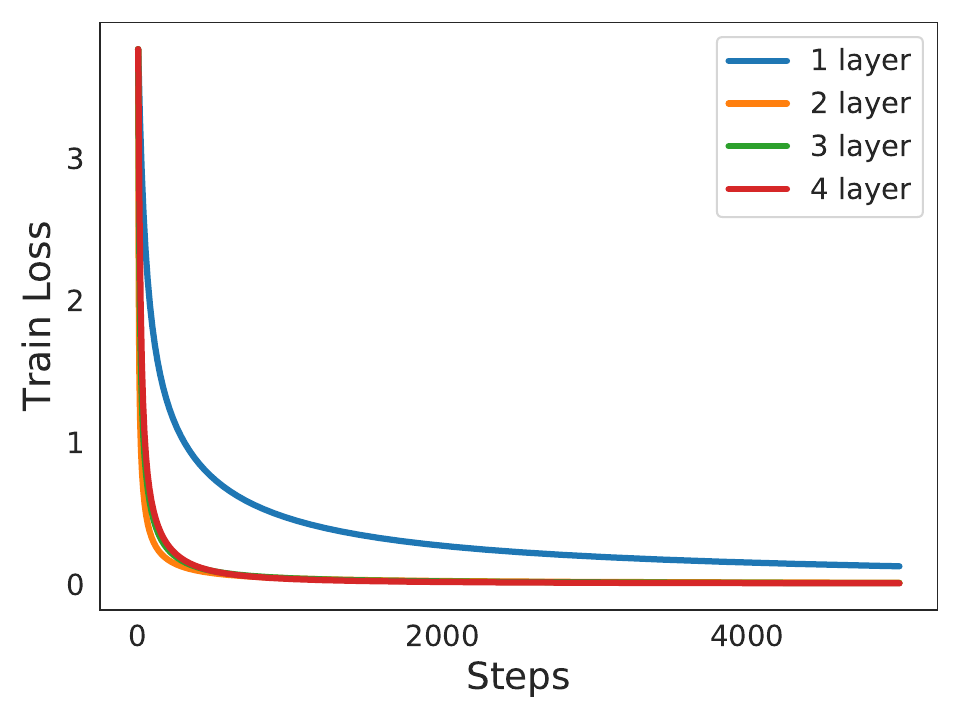}
}
\subfigure[Unnorm]{
  \includegraphics[width=0.3\textwidth]{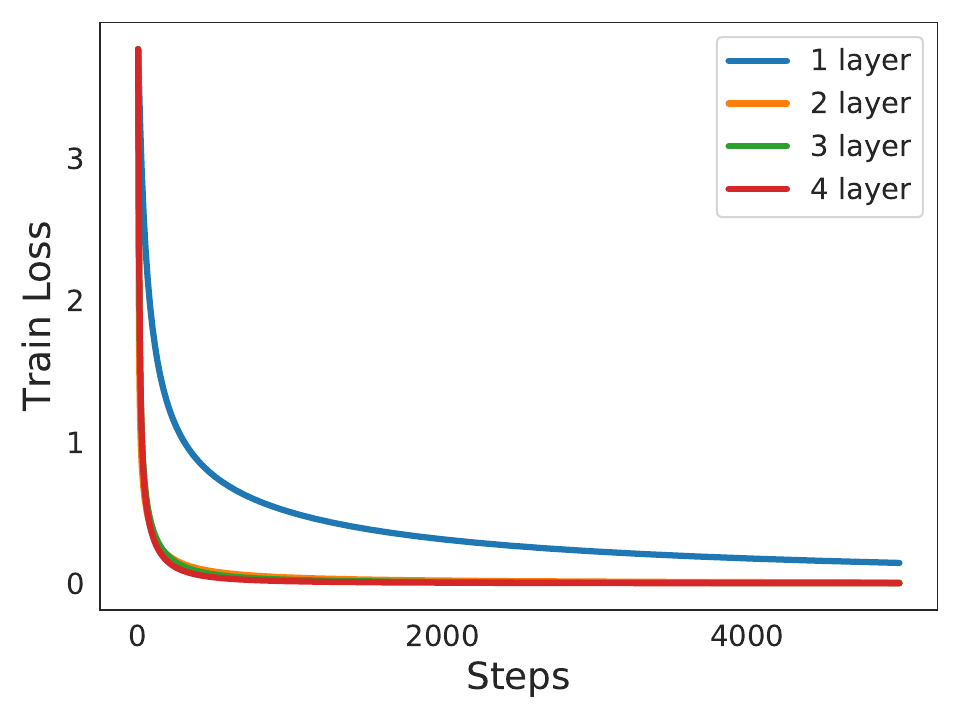}
}
}
\end{figure}

In the \texttt{Simple-Traj} experiment, we illustrate the notion of asymptotic relative sparsity occuring in the EWN parameterization. In this dataset, we have a  single data point at $(2,1)$, that is labelled positive and train a network with linear activations. The architecture is shown in Figure \ref{1_dim_weight}, where weights in blue and red are frozen to values $1$ and $0$ respectively. Thus, there are effectively only two scalar parameters- $w_1$ and $w_2$. The network is trained till a loss value of $e^{-50}$ starting from 5 different initialization points. The weight trajectories in Figure \ref{1_dim_weight} shows that EWN prefers to converge either along the x or y axis, and hence has an asymptotic relative sparsity property. We provide a theoretical proof for the same in the general d-dimensional case.

\begin{proposition}
\label{prop:quant_sparsity}
    Consider a linear model over $\R^d$ given by $f(\x) = \w^\top \x$, where each $w_i$ is further reparameterized as $e^{\alpha_i}$. Consider a dataset consisting of a single data point $\z \succ 0$, that is labelled as +1. According to the initialization of $\balpha$, define a relation $R$ on $\{1,\ldots,d\}$, given by $i \sim j$ if $w_i(0)z_i = w_j(0)z_j$. Then, R is an equivalence relation on $\{1,\ldots,d\}$.
    Let these equivalent sets be denoted by $I_1, I_2, ..., I_k$. Define a total order on these sets given by $I_a > I_b$ if $\exists i \in I_a, j \in I_b$ such that $w_i(0)z_i > w_j(0)z_j$. Let the maximum set according to this order be denoted by $I^*$. Then, for gradient flow on exponential loss, the following holds
    \begin{enumerate}[label=(\roman*)]
        \item For any $i \in I^*$, $\lim_{t \to \infty} w_i(t) = \infty$
        \item For $i, j \in I^*$, $\frac{w_i(t)}{w_j(t)} = \frac{x_j}{x_i}$.
        \item For any $i \notin I^*$, $\lim_{t \to \infty} w_i(t) = \left(\frac{1}{w_i(0)} - \frac{x_i}{w_j(0)x_j}\right)^{-1}$, where $j$ is any element in $I^*$.
    \end{enumerate}
\end{proposition}
%Note that the weights for the unnormalized network move along the vector (2,1) regardless of initialization, while the EWN convergent direction depends heavily on the initialization.
Thus, if $\w$ is initialized from a continuous distribution, then with probability 1, the cardinality of $I^*$ is 1 and hence $\w(t)/\|\w(t)\|$ will approach a sparse vector.

In the \texttt{XOR} experiment, we train a 2-layer ReLU network with 20 hidden neurons on XOR dataset, till a loss value of $e^{-50}$. The second layer is fixed to the values 1 or -1 randomly. For attaining 100\% accuracy on this dataset with this architecture, at least 4 hidden units are needed. As can be seen in Figure \ref{SWN:XOR}, EWN asymptotically uses exactly 4 neurons out of 20, while Unnorm and SWN use almost all the 20 neurons.

% \begin{figure}[t]
% \floatconts{XOR}{\caption{(a) Training data for the \texttt{XOR} experiment. (b, c) Norm of the incoming neuron weights for the EWN and unnormalized architectures.}}
% {
% \subfigure[Dataset]{
%   \includegraphics[width=0.3\textwidth]{plots/XOR_dataset/dataset.pdf}
%   \label{XOR_a}
% }
% \subfigure[EWN]{
%   \includegraphics[width=0.3\textwidth]{plots/XOR_dataset/exp_weight_norm.pdf}
%   \label{XOR_b}
% }
% \subfigure[Unnorm]{
%   \includegraphics[width=0.3\textwidth]{plots/XOR_dataset/no_weight_norm.pdf}
%   \label{XOR_c}
% }
% }
% \end{figure}

%% file: 4_convergencerates.tex
\section{Convergence Rates}
In this section, we provide convergence rate of loss for EWN.

\paragraph{Gradient Flow:} We provide a finite-time convergence rate of loss for gradient flow in case of EWN.
\begin{theorem}
\label{theorem:ewn:gf:rate}
For Exponential Weight Normalization, under assumption (A1), the following hold for $t > t_0$ in case of gradient flow
\begin{enumerate}[label=(\roman*)]
    \begin{minipage}{0.5\textwidth}
    \item $\| \w(t) \|$ grows with $t$ as $o((\log t)^{\frac{1}{L}})$
    \end{minipage}
    \begin{minipage}{0.5\textwidth}
    \item $\L(t)$ goes down with $t$ as $O\left(\frac{1}{t}\right)$
    \end{minipage}
\end{enumerate}
\end{theorem}

\paragraph{Gradient Descent:} For establishing convergence rates for gradient descent, we are going to make an additional assumption that the overall weight vector converges in direction, i.e, $\lim_{t \to \infty} \frac{\w(t)}{\| \w(t) \|}$ exists (\textbf{B4}). Although we have already shown this is indeed true for nodes with $\| \widetilde \g_u \| > 0$, we need this assumption to take into account the nodes with $\| \widetilde \g_u \| = 0$. Under this assumption, $\w$ can be represented as $\w = g(t)\widetilde \w + \br(t)$, where $\lim_{t \to \infty} \frac{\| \br(t) \|}{g(t)} = 0$. Let $d:\mathbb{N} \to \mathbb{R}$, given by $d(t) = \sum_{\tau = 0}^{t-1} \eta(\tau)$ denote total step size. Let $\rho = \min_i y_i\Phi(\widetilde \w, \x_i)$ be the normalized margin at convergence.

The asymptotic convergence rate of loss for SWN and Unnorm have already been established in \citet{LyuLi20} as $\Theta\left(\frac{1}{d(t)(\log d(t))^{2 - \frac{2}{L}}}\right)$. For EWN, the corresponding theorem is provided below
\begin{theorem}
\label{theorem:3}
For Exponential Weight Normalization, under Assumptions (B1)-(B4), $\rho > 0$, $\eta(t) = O\left(\left(\log\frac{1}{\L}\right)^c\right)$ for $c < 1$ and $\lim_{t \to \infty} \frac{\| \br(t+1) - \br(t) \|}{g(t+1) - g(t)} = 0$, the following hold
\begin{enumerate}[label=(\roman*)]
    \item $\| \w(t) \|$ asymptotically grows with $t$ as $\Theta\left((\log d(t))^{\frac{1}{L}}\right)$
    \item $\L(\w(t))$ asymptotically goes down with $t$ as $\Theta\left(\frac{1}{d(t)(\log d(t))^2}\right)$.
\end{enumerate}
\end{theorem}

Although the additional assumption $\lim_{t \to \infty} \frac{\| \br(t+1) - \br(t) \|}{g(t+1) - g(t)} = 0$ is not standard, we empirically demonstrate that, for EWN, the convergence rate is almost independent of the number of layers. Moreover, the learning rate assumption used still covers the constant $\eta(t)$ case, that is generally used in practice.

For multilayer linear nets, the variation of convergence rate with number of layers for a linearly separable dataset is illustrated in Figure \ref{conv_rate_variation}. All of these networks were explicitly initialized to represent the same point in function space. It can be seen that EWN, SWN and unnormalized networks all converge faster with more layers, but the effect is much less pronounced for EWN. 

%% file: 5_experiments.tex
\section{Pruning Experiments}
As EWN leads to asymptotically sparse solutions, it is likely that a sufficiently trained EWN network would be comparatively robust to pruning. In this section, we compare the pruning efficacy of EWN, SWN and Unnorm on MNIST \citep{Lecun+10} dataset. We use a 2-layer ReLU network with 1024 neurons in the hidden layer. In case of EWN and SWN, only the first layer is weight normalized as only this layer needs to be pruned.

\paragraph{Pruning Strategy:} 
The natural pruning strategy of removing neurons $u$ with least $\|\w_u\|$ gives inordinate importance to the initialisation and the initial optimization epochs. In order to minimize the effect of initialization and initial movement of weights, we prune according to the weight norm increase from a reference point. For example, when pruning at a loss value of $e^{-300}$ we consider 4 reference points - $\0$, weight at initialization, weight when $\L = e^{-10}$ and $\L = e^{-100}$. We then choose the pruning strategy that gives maximum testing accuracy for a given level of pruning. Similarly, for pruning at a loss value of $e^{-100}$, we consider three reference points - $\0$, weight at initialization and weight when $\L= e^{-10}$, and for pruning at a loss value of $e^{-10}$, we consider $\0$, weight at initialization and weight when $\L= e^{-5}$. More detailed description of the pruning strategy is provided in Appendix \ref{pruning:algo}.
% The variation of norm of the weight vectors with gradient descent steps for neurons in the first layer  has been deferred to Figure \ref{norm_dist_MNIST} in the appendix. 

The pruning graphs for MNIST at different loss values averaged across multiple seeds are shown in Figure \ref{pruning}. It can be seen that when the loss levels are sufficiently low, the EWN network becomes better adapted for pruning, significantly outperforming SWN and the unnormalized network in terms of test accuracy for a given level of pruning. Further details along with convergence rate plots are provided in Appendix \ref{pruning:conv}.

% \begin{figure}[t]
% \centering
% \begin{subfigure}{0.3\textwidth}
% \centering
%   \includegraphics[width=\textwidth]{plots/pruning/pruning_10_max.pdf}
%   \caption{$\L = e^{-10}$}
% \end{subfigure}
% \begin{subfigure}{0.3\textwidth}
% \centering
%   \includegraphics[width=\textwidth]{plots/pruning/pruning_100_max.pdf}
%   \caption{$\L = e^{-100}$}
% \end{subfigure}
% \begin{subfigure}{0.3\textwidth}
% \centering
%   \includegraphics[width=\textwidth]{plots/pruning/pruning_300_max.pdf}
%   \caption{$\L = e^{-300}$}
% \end{subfigure}
% \caption{Variation of test accuracy vs percentage of neurons pruned in first layer at different loss values for MNIST experiment}
% \label{pruning}
% \end{figure}

% \begin{figure}[t]
% \centering
% \begin{subfigure}{0.3\textwidth}
% \centering
%   \includegraphics[width=\textwidth]{plots/CIFAR_pruning/pruning_10_max.pdf}
%   \caption{$\L = e^{-10}$}
% \end{subfigure}
% \begin{subfigure}{0.3\textwidth}
% \centering
%   \includegraphics[width=\textwidth]{plots/CIFAR_pruning/pruning_100_max.pdf}
%   \caption{$\L = e^{-100}$}
% \end{subfigure}
% \begin{subfigure}{0.3\textwidth}
% \centering
%   \includegraphics[width=\textwidth]{plots/CIFAR_pruning/pruning_300_max.pdf}
%   \caption{$\L = e^{-300}$}
% \end{subfigure}
% \caption{Variation of test accuracy vs percentage of neurons pruned at different loss values for VGG-13 on CIFAR-10 dataset}
% \label{pruning:CIFAR}
% \end{figure}

\begin{figure}[t]
\floatconts{pruning}{\caption{Variation of test accuracy vs percentage of neurons pruned in first layer at different loss values for MNIST experiment}}
{
\subfigure[$\L = e^{-10}$]{
  \includegraphics[width=0.3\textwidth]{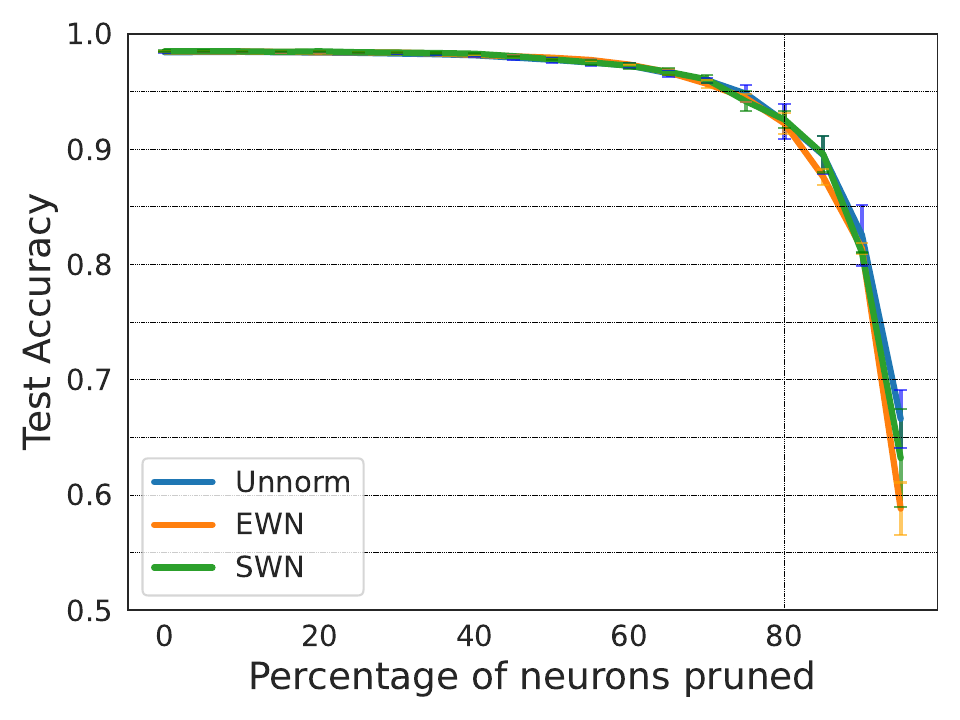}
}
\subfigure[$\L = e^{-100}$]{
  \includegraphics[width=0.3\textwidth]{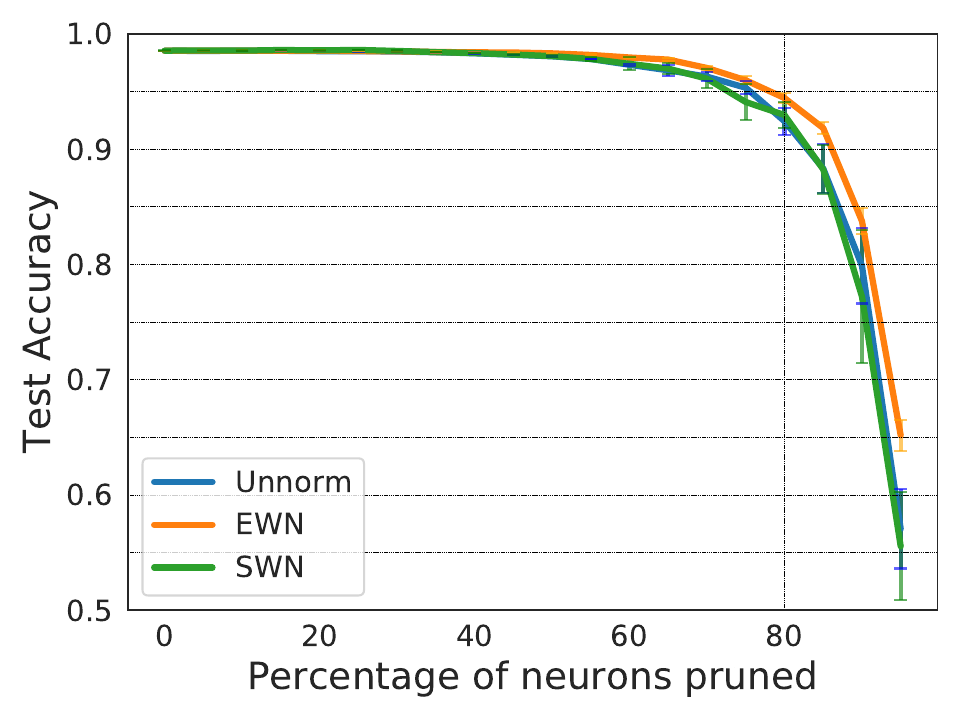}
}
\subfigure[$\L = e^{-300}$]{
  \includegraphics[width=0.3\textwidth]{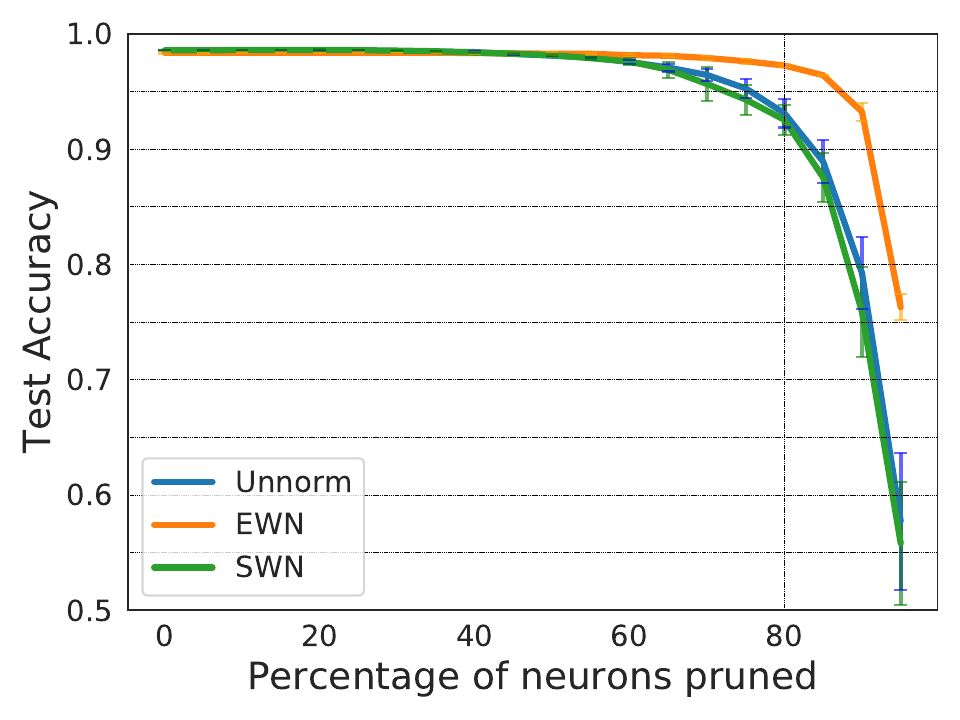}
}
}
\end{figure}

% \begin{figure}[t]
% \floatconts{pruning:CIFAR}{\caption{Variation of test accuracy vs percentage of neurons pruned at different loss values for VGG-13 on CIFAR-10 dataset}}
% {
% \subfigure[$\L = e^{-10}$]{
%   \includegraphics[width=0.3\textwidth]{plots/CIFAR_pruning/pruning_new_2_-10_-5_max.pdf}
% }
% \subfigure[$\L = e^{-100}$]{
%   \includegraphics[width=0.3\textwidth]{plots/CIFAR_pruning/pruning_new_2_-100_max.pdf}
% }
% \subfigure[$\L = e^{-300}$]{
%   \includegraphics[width=0.3\textwidth]{plots/CIFAR_pruning/pruning_new_2_-300_max.pdf}
% }
% }
% \end{figure}

\begin{figure}[t]
\floatconts{EWN:sparse:pruning}{\caption{(a) Variation of test accuracy vs percentage of neurons pruned on MNIST dataset after training the network initially with $\ell_{2,1}$ regularizer and later on continuing training with either EWN, SWN or unnormalized net till $\L = e^{-30}$ (b) Zoomed in view on EWN}}
{
\subfigure{
  \includegraphics[width=0.4\textwidth]{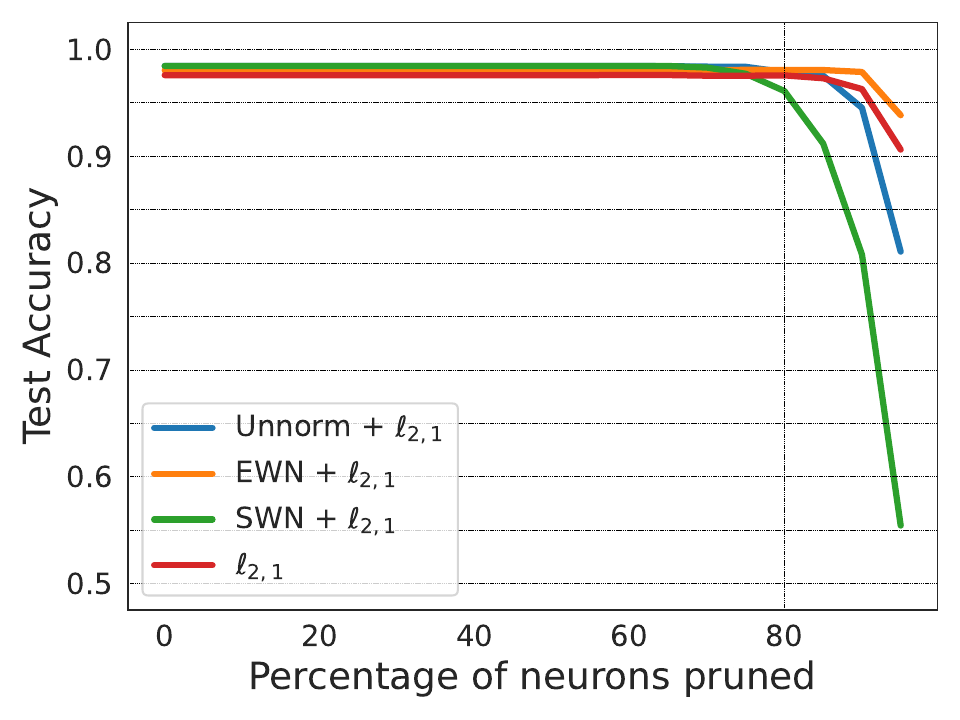}
}
\subfigure{
  \includegraphics[width=0.4\textwidth]{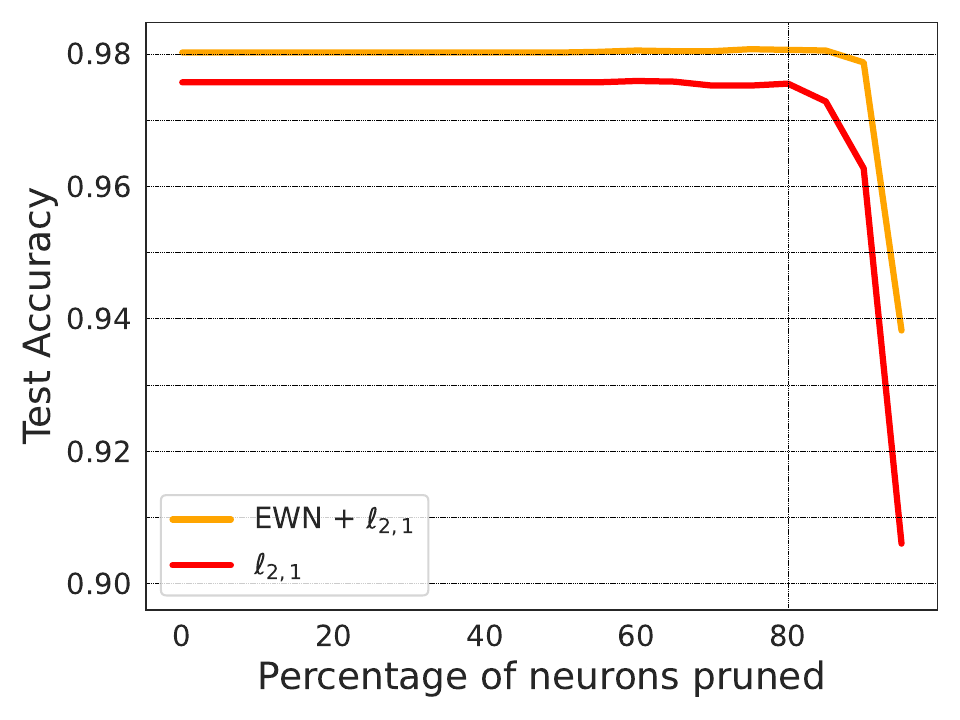}
}
}
\end{figure}

\subsection{Combining EWN with other sparsity regularizers}
EWN by itself has a sparsifying effect, and in addition it can be combined with other sparsity regularizers. e.g. the neural net can be trained using any existing sparsity regularizer with standard parameterization for the initial few epochs, and use EWN (without the regularizer) for the later phase. 

We conducted experiment on MNIST dataset, by initially training the network till convergence with $\ell_{2,1}$ group sparsity regularizer ($\sum_{u \in \text{nodes of network}} \| \w_u \|$), and later on using EWN till a loss value of $e^{-30}$. In this case, the pruning strategy is based on the final norms of the weights as the sparsity regularizer already induces an asymmetry among different neurons. The pruning results are shown in Figure \ref{EWN:sparse:pruning}. As can be seen, training further with EWN improves the pruning efficacy of the network.

%% file: 6_proofsketch.tex
\section{Proof Sketch}
In this section, we provide the proof for part (iii)a of Proposition \ref{prop:3} for gradient flow, i.e, for EWN, $\frac{\| \w_u(t_2) \|}{\| \w_v(t_2) \|} > \frac{1}{(c-\epsilon)\cos(\delta)} \implies \lim_{t \to \infty} \frac{\| \w_u(t) \|}{\| \w_v(t) \|} = \infty$. Remaining proofs are given in appendix.

The update equations for $\alpha_u$ and $\bv_u$ in case of gradient flow for EWN are given by
\begin{align}
\label{EWN:alpha:gf:main}
\frac{d\alpha_u(t)}{dt} &= -\eta(t)e^{\alpha_u(t)} \frac{\bv_u(t)^\top \nabla_{\w_u} \L}{\| \bv_u(t) \|} \\
\label{EWN:v:gf:main}
\frac{d \bv_u(t)}{dt} &= -\eta(t)\frac{e^{\alpha_u(t)}}{\| \bv_u(t) \|}\left(I - \frac{\bv_u(t)\bv_u(t)^\top}{\| \bv_u(t) \|^2}\right)\nabla_{\w_u} \L
\end{align}

Using Equation (\ref{EWN:alpha:gf:main}) (along with the fact that $\| \bv_u(t) \|$ does not change with time and $\| \bv_u(0) \| = 1$),
\begin{equation}
\label{eq:main:wnorm:EWN:gf}
\frac{d \| \w_u(t) \| }{dt} = \frac{d e^{\alpha_u(t)}}{dt} = -\eta(t) \| \w_u(t) \|^2 (\bv_u(t)^\top \nabla_{\w_u} \L(\w(t))
\end{equation}

Using Equation (\ref{eq:main:wnorm:EWN:gf}) and part 1 of Proposition \ref{prop:3}, we can say for $t > t_1$,
\begin{align}
    \label{eq:main:wratio:lowbound:EWN:gf}
    \frac{d \frac{\| \w_u(t) \|}{\| \w_v(t) \|}}{dt} &= \frac{\| \w_v(t) \|\frac{d \| \w_u(t) \|}{dt} -   \| \w_u(t) \|\frac{d \| \w_v(t) \|}{dt}}{\| \w_v(t) \|^2} \nonumber \\
    &\geq \eta(t)\frac{\| \w_u(t) \|}{\| \w_v(t) \|}(\|\w_u(t)\|\|\nabla_{\w_u} \L(\w(t)) \| \cos(\delta) - \|\w_v(t)\|\|\nabla_{\w_v} \L(\w(t)) \|) \nonumber\\
    &\geq \eta(t)\| \w_u(t) \|\|\nabla_{\w_u} \L(\w(t)) \|\left(\frac{\| \w_u(t) \|}{\| \w_v(t) \|} \cos(\delta) - \frac{1}{c - \epsilon}\right)
\end{align}

In this case, using Equation (\ref{eq:main:wratio:lowbound:EWN:gf}), we can see $\frac{d \frac{\| \w_u(t) \|}{\| \w_v(t) \|}}{dt} > 0$ at $t_2$. Thus, $\frac{\| \w_u(t) \|}{\| \w_v(t) \|}$ always remains greater than $\frac{1}{(c - \epsilon) \cos(\delta)}$ and keeps on increasing. Let's denote $\frac{\|\w_u(t_2)\|}{\|\w_v(t_2)\|}$ by $\Delta$. Then, for $t > t_2$,
\[ \frac{d \frac{\| \w_u(t) \|}{\| \w_v(t) \|}}{dt} \geq \left(\Delta\cos(\delta) - \frac{1}{c - \epsilon}\right) \eta(t) \|\w_u(t)\|\|\nabla_{\w_u} \L(\w(t)) \| \]
As $\alpha_u \to \infty$, $\int_{t_2}^\infty \eta(t) \|\w_u(t)\|\|\nabla_{\w_u} \L(\w(t)) \| dt = \infty$ using Equation (\ref{EWN:alpha:gf:main}). Thus, integrating both the sides of the equation above from $t_2$ to $\infty$, we get
\[ \int_{t_2}^\infty \frac{d \frac{\| \w_u(t) \|}{\| \w_v(t) \|}}{dt} dt \geq \infty \]
Thus $\lim_{t \to \infty} \frac{\| \w_u(t) \|}{\| \w_v(t) \|} = \infty$. A similar proof works for the part (iii)b of Proposition \ref{prop:3} as well.

%% file: 7_conclusion.tex
% \newpage
\section{Conclusion}
In this paper, we analyze the inductive bias of weight normalization for smooth homogeneous neural nets and show that exponential weight normalization is likely to lead to asymptotically sparse solutions and has a faster convergence rate than unnormalized or standard weight normalized networks.

The smooth homogeneity assumptions made in the paper are satisfied by any positive power of ReLU greater than 1. The primary issue with ReLU is that the assumption of gradients converging in direction does not make sense for the non-smooth case. However, as our experiments demonstrate, the implication of relative sparsity holds for ReLU activation as well. Therefore, extending the results in the paper for the non-smooth case is a promising research direction.

% The sparsity theorem established in the paper is a qualitative result, not a quantitative one. Understanding how the sparsity in weights is affected by network architecture and dataset is an intriguing research direction.

% While experimenting with XOR dataset, we found that increasing the powers of ReLU promotes sparsity even for SWN and Unnorm. However, this is a slightly different kind of sparsity than implied by Proposition \ref{prop:3}. Proposition \ref{prop:3} states that even when the gradients have a non-sparse direction of convergence, EWN is likely to have a sparse convergent direction for weights. However, with increasing powers of ReLU, we observed that the gradients itself start converging to a sparse direction. These experiments are provided in the Appendix. Understanding this phenomenon is also an interesting research direction.
Although the trajectory based assumptions have been shown to hold for gradient flow on unnormalized nets under certain regularity conditions, establishing similar conditions for weight normalized networks remains an open question. Moreover, extending the directional convergence results from gradient flow to gradient descent is also an interesting research direction.

%% file: 8_appendix.tex
\section{Proof of Theorem \ref{theorem:1}}
\begin{theorem*}
The gradient flow path with learning rate $\eta(t)$ for EWN and SWN are given as follows:
 \begin{align*}
 \text{EWN: } \frac{d \w_u(t)}{dt} &= -\eta(t)\| \w_u(t) \|^2 \nabla_{\w_u} \L \\
 \text{SWN: } \frac{d \w_u(t)}{dt}  &= -\eta(t)\left(\|\w_u(t)\|^2 \nabla_{\w_u} \L + \left(\frac{1 - \|\w_u(t)\|^2}{\|\w_u(t)\|^2}\right) (\w_u(t)^\top\nabla_{\w_u} \L) \w_u(t)\right)
 \end{align*}
\end{theorem*}
The proof for the two parts will be provided in different subsections, where the corresponding part will be restated for ease of the reader.
\subsection{Exponential Weight Normalization}
\begin{theorem*}
The gradient flow path with learning rate $\eta(t)$ for EWN is given by:
 \begin{align*}
 \frac{d \w_u(t)}{dt} &= -\eta(t)\| \w_u(t) \|^2 \nabla_{\w_u} \L \\
 \end{align*}
\end{theorem*}
\begin{proof}
In case of EWN, weights are reparameterized as $\w_u = e^{\alpha_u} \frac{\bv_u}{\| \bv_u\|}$. Then
\[ \nabla_{\alpha_u} \L = e^{\alpha_u} \frac{\bv_u^\top \nabla_{\w_u} \L}{\| \bv_u \|} \]
\[ \nabla_{\bv_u} \L = \frac{e^{\alpha_u}}{\| \bv_u \|}\left(I - \frac{\bv_u\bv_u^\top}{\| \bv_u \|^2}\right)\nabla_{\w_u} \L \]
Now, in case of gradient flow with learning rate $\eta(t)$, we can say
\[ \frac{d\alpha_u(t)}{dt} = -\eta(t)\nabla_{\alpha_u} \L = -\eta(t)e^{\alpha_u(t)} \frac{\bv_u(t)^\top \nabla_{\w_u} \L}{\| \bv_u(t) \|} \]
\[ \frac{d \bv_u(t)}{dt} = -\eta(t)\nabla_{\bv_u} \L = -\eta(t)\frac{e^{\alpha_u(t)}}{\| \bv_u(t) \|}\left(I - \frac{\bv_u(t)\bv_u(t)^\top}{\| \bv_u(t) \|^2}\right)\nabla_{\w_u} \L \]
Now, using these equations, we can say
\[ \frac{d \| \bv_u(t) \|^2}{dt} = 2\bv_u(t)^\top\left(\frac{d \bv_u(t)}{dt}\right) = 0 \]
Thus, $\| \bv_u(t) \|$ does not change with time. As we assumed $\| \bv_u(0) \|$ to be 1, therefore for any t, $\| \bv_u(t) \| = 1$. Using this simplification, we can write
\begin{align*}
    \frac{d \w_u(t)}{dt} &= \frac{d(e^{\alpha_u(t)}\bv_u(t))}{dt} \\
    &= e^{\alpha_u(t)}(-\eta(t)e^{\alpha_u(t)}(I - \bv_u(t)\bv_u(t)^\top)\nabla_{\w_u} \L) - \eta(t)e^{2\alpha_u(t)}(\bv_u(t)^\top \nabla_{\w_u} \L) \bv_u(t)\\
    &= -\eta(t) e^{2\alpha_u(t)} \nabla_{\w_u} \L
\end{align*}
Thus, the gradient flow path with exponential weight normalization can be replicated by an adaptive learning rate given by $\eta(t) \| \w_u(t) \|^2$.
\end{proof}
\subsection{Standard Weight Normalization}
\begin{theorem*}
The gradient flow path with learning rate $\eta(t)$ for SWN is given by:
 \begin{align*}
 \frac{d \w_u(t)}{dt}  &= -\eta(t)\left(\|\w_u(t)\|^2 \nabla_{\w_u} \L + \left(\frac{1 - \|\w_u(t)\|^2}{\|\w_u(t)\|^2}\right) (\w_u(t)^\top\nabla_{\w_u} \L) \w_u(t)\right)
 \end{align*}
\end{theorem*}
\begin{proof}
In case of SWN, weights are reparameterized as $\w_u = \gamma_u \frac{\bv_u}{\| \bv_u\|}$. Then
\[ \nabla_{\gamma_u} \L = \frac{\bv_u^\top \nabla_{\w_u} \L}{\| \bv_u \|} \]
\[ \nabla_{\bv_u} \L = \frac{\gamma_u}{\| \bv_u \|}\left(I - \frac{\bv_u\bv_u^\top}{\| \bv_u \|^2}\right)\nabla_{\w_u} \L \]
Now, in case of gradient flow with learning rate $\eta(t)$, we can say
\[ \frac{d\gamma_u(t)}{dt} = -\eta(t)\nabla_{\alpha_u} \L = -\eta(t) \frac{\bv_u(t)^\top \nabla_{\w_u} \L}{\| \bv_u(t) \|} \]
\[ \frac{d \bv_u(t)}{dt} = -\eta(t)\nabla_{\bv_u} \L = -\eta(t)\frac{\gamma_u(t)}{\| \bv_u(t) \|}\left(I - \frac{\bv_u(t)\bv_u(t)^\top}{\| \bv_u(t) \|^2}\right)\nabla_{\w_u} \L \]
Now, similar to EWN, $\| \bv_u(t) \|$ does not change with time. Using the fact that $\| \bv_u(t) \| =1$ for all $t$, we can say
\begin{align*}
    \frac{d \w_u(t)}{dt} &= \frac{d(\gamma_u(t) \bv_u(t))}{dt} \\
    &= \gamma_u(t)(-\eta(t)\gamma_u(t)(I - \bv_u(t)\bv_u(t)^\top)\nabla_{\w_u} \L) - \eta(t)(\bv_u(t)^\top \nabla_{\w_u} \L) \bv_u(t)\\
    &= -\eta(t)(\gamma_u(t)^2 \nabla_{\w_u} \L + (1 - \gamma_u(t)^2)(\bv_u(t)^\top\nabla_{\w_u} \L) \bv_u(t))\\
    &= -\eta(t)\left(\gamma_u(t)^2 \nabla_{\w_u} \L + \left(\frac{1 - \gamma_u(t)^2}{\gamma_u(t)^2}\right)(\w_u(t)^\top\nabla_{\w_u} \L) \w_u(t)\right)
\end{align*}
Replacing $\gamma_u(t)$ by $\| \w_u(t) \|$ gives the required expression.
\end{proof}

\section{Proof of Proposition \ref{prop:lzero:gf}} \label{proof:prop:lzero:gf}
\begin{proposition*}
Under assumption (A1) for gradient flow, for both SWN and EWN, $\lim_{t \to \infty} \L(\w(t)) = 0$.
\end{proposition*}

The proof for SWN, as it is homogeneous in its parameters, was provided by \citet{LyuLi20}. We provide the proof for EWN.

First of all, for exponential loss
\begin{equation}
\label{eq:exp:loss:grad}
\frac{d\L(t)}{d\w} = -\sum_i e^{-y_i\Phi(\w(t), \x_i)} y_i\nabla_\w \Phi(\w(t), \x_i)
\end{equation}
Now, using Theorem \ref{theorem:1},
\[ \frac{d\L(t)}{dt} = \left(\frac{d\L(t)}{d\w}\right)^\top \frac{d\w(t)}{dt} = -\sum_u \| \w_u(t) \|^2 \left\| \frac{d\L(t)}{d\w_u} \right\|^2 \]
Let $k$ be the total number of neurons in the network. Then using the elementary inequality, $(\sum_{i=1}^n a_i)^2 \leq n\sum_{i=1}^n a_i^2$, we get
\[ \frac{d\L(t)}{dt} \leq -\frac{1}{k} \left(\sum_u \| \w_u(t) \| \left\| \frac{d\L(t)}{d\w_u} \right\| \right)^2 \]
Again using the fact that $\left|\w(t)^\top \frac{d\L(t)}{d\w}\right| \leq \sum_u \| \w_u(t) \| \left\| \frac{d\L(t)}{d\w_u} \right\|$, we get
\begin{equation}
\label{eq:dL:dt:upp:bound}
\frac{d\L(t)}{dt} \leq -\frac{1}{k} \left(\w(t)^\top \frac{d\L(t)}{d\w}\right)^2    
\end{equation}
Taking the dot product with $\w$ on both sides of Equation (\ref{eq:exp:loss:grad}) and using $\w^\top \nabla_\w \Phi(\w, \x_i) = L \Phi(\w, \x_i)$ (Euler's homogeneity theorem), we get
\[ \w(t)^\top \frac{d\L(t)}{d\w} = -L\sum_i e^{-y_i\Phi(\w(t), \x_i)} y_i\Phi(\w(t), \x_i) \]
Now, using the fact, that at time $t_0$, $\L(t_0) < 1$, which means $\min_i y_i\Phi(\w(t_0), \x_i) = \epsilon > 0$. Also, as we know, for gradient flow, the loss cannot go up, therefore, for any time $t > t_0$, $\min_i y_i\Phi(\w(t), \x_i) > \epsilon > 0$. Using this, we can say, for any $t > t_0$,
\[ \w(t)^\top \frac{d\L(t)}{d\w} \leq -L\epsilon \L(t)) \]
Substituting this in Equation (\ref{eq:dL:dt:upp:bound}), we get
\[ \frac{d\L(t)}{dt} \leq - \frac{L^2\epsilon^2}{k} \L(t)^2 \]
Integrating this equation from $t_0$ to $t$, we get
\begin{equation}
\label{eq:finiteconvrate:gf}
\frac{1}{\L(t)} \geq \frac{1}{\L(t_0)} + \frac{L^2\epsilon^2}{k} (t - t_0)
\end{equation}
Clearly as t tends to $\infty$, RHS tends to $\infty$ and thus $\L$ tends to 0.
\section{Proof of Theorem \ref{thm:weightconv}} \label{proof:thm:weightconv}
\begin{theorem*}
    Consider a node $u$ in the network with $\| \widetilde{\g}_u \| > 0$ and $\lim_{t \to \infty} \| \w_u(t) \| = \infty$. Under assumptions (A1), (A2) for gradient flow and (B1)-(B3) for gradient descent, for both SWN and EWN
    \begin{enumerate}[label=(\roman*)]
        \begin{minipage}{0.5\textwidth}
        \item $\lim_{t \to \infty} \frac{\w_u(t)}{\| \w_u(t) \|} := \widetilde{\w}_u$ exists.
        \end{minipage}
        \begin{minipage}{0.5\textwidth}
        \item $\widetilde{\w}_u = \lambda \widetilde \g_u$ for some $\lambda > 0$.
        \end{minipage}
    \end{enumerate}
\end{theorem*}
The proof for different cases will be split into different subsections and corresponding theorem will be stated there for ease of the reader. The proof will depend on the Stolz Cesaro theorem(stated in Appendix \ref{intstolz}),Integral Form of Stolz-Cesaro Theorem(stated and proved in Appendix \ref{intstolz}) and following lemmas that have been proved in Appendix \ref{pflemmas}. 

% \begin{lemma}
% \label{lemma:1}
% Under assumptions (A1)-(A2) for gradient flow and (A1)-(A4) for gradient descent, for both SWN and EWN, $\widetilde \w_u^\top \widetilde \g_u \geq 0$ for all nodes $u$ in the network.
% \end{lemma}

% \begin{lemma}
% \label{lemma:3}
%     Consider two unit vectors $\ba$ and $\bb$ satisfying $\ba^{\top}\bb \geq 0$ and $\ba^{\top}\bb < 1$. Then, there exists a small enough $\epsilon > 0$, such that for any unit vector $\bc$ satisfying $\bc^\top \ba \geq \cos(\epsilon)$ and any unit vector $\bd$ satisfying $\bd^\top \bb \geq \cos(\epsilon)$, $\bb^\top(I - \bc\bc^\top)\bd \geq \epsilon$.
% \end{lemma}

\begin{lemma}
\label{lemma:4}
Consider sequence a satisfying the following properties
\begin{enumerate}
    \item $a_k > 0$
    \item $\sum_{k=0}^\infty a_k = \infty$
    \item $\lim_{k \to \infty} a_k = 0$
\end{enumerate}
Then $\sum_{k=0}^{\infty} \frac{a_k}{\sqrt{\sum_{j=0}^{k} a_j^2}} = \infty$
\end{lemma}

\begin{lemma}
\label{lemma:5}
Consider two sequences a and b satisfying the following properties
\begin{enumerate}
    \item $a_k > 0, \sum_{k=0}^\infty a_k = \infty$ and $\lim_{k \to \infty} a_k = 0$
    \item $b_0 > 0$, $b$ is increasing and $b_{k+1}^2 \leq b_k^2 + \left(\frac{a_k}{b_k}\right)^2$
\end{enumerate}
Then $\sum_{k=0}^\infty \frac{a_k}{b_k} = \infty$.
\end{lemma}

% \begin{lemma}
% \label{lemma:6}
%     Consider two sequences a and b satisfying the following properties
% \begin{enumerate}
%     \item $a_k > 0$ and $\sum_{k=0}^\infty a_k = \infty$
%     \item $b_k > 0$ and $\sum_{k=0}^\infty b_k = \infty$
%     \item $\sum_{k=0}^\infty (a_k - b_k)$ converges to a finite value
%     \item $lim_{k \to \infty} \frac{a_k}{b_k}$ exists
% \end{enumerate}
% Then $\lim_{k \to \infty} \frac{a_k}{b_k} = 1$.
%\end{lemma}

\subsection{Exponential Weight Normalization} \label{proof:ewn:thm:weightconv}
In this section, we will use $e^{\alpha_u(t)}$ and $\| \w_u(t) \|$ interchangeably.
\subsubsection{Gradient Flow} \label{proof:ewn:gf:thm:weightconv}
\begin{theorem*}
    Consider a node $u$ in the network with $\| \widetilde{\g}_u \| > 0$ and $\lim_{t \to \infty} \| \w_u(t) \| = \infty$. Under assumptions (A1), (A2) for gradient flow, for EWN
    \begin{enumerate}[label=(\roman*)]
        \begin{minipage}{0.5\textwidth}
        \item $\lim_{t \to \infty} \frac{\w_u(t)}{\| \w_u(t) \|} := \widetilde{\w}_u$ exists.
        \end{minipage}
        \begin{minipage}{0.5\textwidth}
        \item $\widetilde{\w}_u = \lambda \widetilde \g_u$ for some $\lambda > 0$.
        \end{minipage}
    \end{enumerate}
\end{theorem*}
\textbf{Update Equations:}
\begin{equation}
    \label{eq:alpha:EWN:gf}
    \frac{d\alpha_u(t)}{dt} = -\eta(t)e^{\alpha_u(t)} (\bv_u(t)^\top \nabla_{\w_u} \L(\w(t)))
\end{equation}

\begin{equation}
    \label{eq:v:EWN:gf}
    \frac{d \bv_u(t)}{dt} = -\eta(t)e^{\alpha_u(t)}(I - \bv_u(t)\bv_u(t)^\top)\nabla_{\w_u} \L(\w(t))
\end{equation}

\begin{proof}
As $\| \widetilde \g_u \| > 0$, therefore $\nabla_{\w_u} \L(\w(t))$ converges in direction. Therefore, for every $\tau$ satisfying $0 < \tau < 2\pi$, there exists a time $t_1(\tau)$, such that for $t > t_1$,  $\left( \frac{-\nabla_{\w_u} \L(\w(t))}{\| \nabla_{\w_u} \L(\w(t))\|} \right)^\top \left(\frac{\widetilde \g_u}{\| \widetilde \g_u \|}\right) \geq \cos(\tau)$. Now, Let's assume that $\w_u(t)$ does not converge in the direction of $\widetilde \g_u$. Then, there must exist a $\tau$ satisfying $0 < \tau < 2\pi$, such that for this $\tau$, there exists a time $t_2 > t_1(\tau)$ satisfying $\bv_u(t_2)^\top \left(\frac{\widetilde \g_u}{\| \widetilde \g_u \|}\right) = \cos(\Delta)$, where $\Delta > \tau$.

Now, we are going to show that for any $\kappa$ satisfying $\tau < \kappa < \Delta$, there exists a time $t_3 > t_2$ such that $\bv_u(t_3)^\top \left(\frac{\widetilde \g_u}{\| \widetilde \g_u \|}\right) > \cos(\kappa)$. Let's say for a given $\kappa$, no such $t_3$ exists. Then, taking dot product with $\frac{\widetilde \g_u}{\| \widetilde \g_u \|}$ on both sides of Equation (\ref{eq:v:EWN:gf}), we can say
\[ \left( \frac{\widetilde \g_u}{\| \widetilde \g_u \|} \right)^\top \frac{d \bv_u(t)}{dt} = \eta(t) e^{\alpha_u(t)} \| \nabla_{\w_u} \L(\w(t)) \| \left( \frac{\widetilde \g_u}{\| \widetilde \g_u \|} \right)^\top (I - \bv_u(t)\bv_u(t)^\top) \left( \frac{-\nabla_{\w_u} \L(\w(t))}{\| \nabla_{\w_u} \L(\w(t)) \|} \right) \]
Now, as $\left( \frac{\widetilde \g_u}{\| \widetilde \g_u \|} \right)^\top \left( \frac{-\nabla_{\w_u} \L(\w(t))}{\| \nabla_{\w_u} \L(\w(t)) \|} \right) \geq \cos(\tau)$ and $\left( \frac{\widetilde \g_u}{\| \widetilde \g_u \|} \right)^\top \bv_u \leq \cos(\kappa)$, we can say
\begin{equation}
\label{eq:ewn:gf:w:conv}
 \left( \frac{\widetilde \g_u}{\| \widetilde \g_u \|} \right)^\top \frac{d \bv_u(t)}{dt} \geq \eta(t) e^{\alpha_u(t)} \| \nabla_{\w_u} \L(\w(t)) \| (\cos(\tau) - \cos(\kappa))   
\end{equation}
Now, using the fact that $\alpha_u \to \infty$ and using Equation (\ref{eq:alpha:EWN:gf}), we can say
\[ \int_{t = t_2}^\infty \eta(t) e^{\alpha_u(t)} \| \nabla_{\w_u} \L(\w(t)) \| dt = \infty \]
Using this fact and integrating the Equation (\ref{eq:ewn:gf:w:conv}) on both the sides from $t_2$ to $\infty$, we get a contradiction as vectors on LHS have a finite norm while RHS tends to $\infty$. Thus, for every $\kappa$ between $\tau$ and $\Delta$, there must exist a $t_3$, such that $\bv_u(t_3)^\top \left(\frac{\widetilde \g_u}{\| \widetilde \g_u \|}\right) > \cos(\kappa)$.

Now, we are going to show for all $t \geq t_3$, $\bv_u(t)^\top \left(\frac{\widetilde \g_u}{\| \widetilde \g_u \|}\right) > \cos(\kappa)$. Now, consider any $\beta$ such that $\tau < \beta < \kappa$. Using similar argument as in Equation (\ref{eq:ewn:gf:w:conv}), we can say, if for any $t_4 > t_3$, $\bv_u(t_4)^\top \left(\frac{\widetilde \g_u}{\| \widetilde \g_u \|}\right) < \cos(\beta)$, then 
\begin{equation}
\label{eq:ewn:gf:w:conv:2}
 \left( \frac{\widetilde \g_u}{\| \widetilde \g_u \|} \right)^\top \frac{d \bv_u(t_4)}{dt} \geq \eta(t_4) e^{\alpha_u(t_4)} \| \nabla_{\w_u} \L(\w(t_4)) \| (\cos(\tau) - \cos(\beta))   
\end{equation}
This means that the dot product between $\left( \frac{\widetilde \g_u}{\| \widetilde \g_u \|} \right)$ and $\bv_u(t)$ goes up, whenever $\left( \frac{\widetilde \g_u}{\| \widetilde \g_u \|} \right)^\top \bv_u(t) < \cos(\tau)$. Therefore, its not possible that $\bv_u(t)^\top \left(\frac{\widetilde \g_u}{\| \widetilde \g_u \|}\right) \leq \cos(\kappa)$ for any $t > t_3$. As $\kappa$ can be arbitrarily chosen between $\tau$ and $\Delta$, $\w_u(t)$ converges in the direction of $\widetilde \g_u$.
\end{proof}
\subsubsection{Gradient Descent} \label{proof:ewn:gd:thm:weightconv}
\begin{theorem*}
    Consider a node $u$ in the network with $\| \widetilde{\g}_u \| > 0$ and $\lim_{t \to \infty} \| \w_u(t) \| = \infty$. Under assumptions (B1)-(B3) for gradient descent, for EWN
    \begin{enumerate}[label=(\roman*)]
        \begin{minipage}{0.5\textwidth}
        \item $\lim_{t \to \infty} \frac{\w_u(t)}{\| \w_u(t) \|} := \widetilde{\w}_u$ exists.
        \end{minipage}
        \begin{minipage}{0.5\textwidth}
        \item $\widetilde{\w}_u = \lambda \widetilde \g_u$ for some $\lambda > 0$.
        \end{minipage}
    \end{enumerate}
\end{theorem*}
\textbf{Update Equations:}
\begin{equation}
    \label{eq:alpha:EWN:gd:2}
    \alpha_u(t+1) = \alpha_u(t) - \eta(t)e^{\alpha_u(t)} \frac{\bv_u(t)^\top \nabla_{\w_u} \L(\w(t))}{\| \bv_u(t) \|}
\end{equation}
\begin{equation}
    \label{eq:v:EWN:gd:2}
    \bv_u(t+1) = \bv_u(t) - \eta(t)\frac{e^{\alpha_u(t)}}{\| \bv_u(t) \|}\left(I - \frac{\bv_u(t)\bv_u(t)^\top}{\| \bv_u(t) \|^2}\right)\nabla_{\w_u} \L(\w(t))
\end{equation}

\begin{proof}

As $\| \widetilde \g_u \| > 0$, therefore $\nabla_{\w_u} \L(\w(t))$ converges in direction. Therefore, for every $\tau$ satisfying $0 < \tau < 2\pi$, there exists a time $t_1(\tau)$, such that for $t > t_1(\tau)$,  $\left( \frac{-\nabla_{\w_u} \L(\w(t))}{\| \nabla_{\w_u} \L(\w(t))\|} \right)^\top \left(\frac{\widetilde \g_u}{\| \widetilde \g_u \|}\right) \geq \cos(\tau)$. Now, Let's assume that $\w_u(t)$ does not converge in the direction of $\widetilde \g_u$. Then, there must exist a $\tau$ satisfying $0 < \tau < 2\pi$, such that for this $\tau$, there exists a time $t_2 > t_1(\tau)$ satisfying $\bv_u(t_2)^\top \left(\frac{\widetilde \g_u}{\| \widetilde \g_u \|}\right) = \cos(\Delta)$, where $\Delta > \tau$.

Now, we are going to show that for any $\kappa$ satisfying $\tau < \kappa < \Delta$, there exists a time $t_3 > t_2$ such that $\left( \frac{\bv_u(t_3)}{\| \bv_u(t_3) \|} \right)^\top \left(\frac{\widetilde \g_u}{\| \widetilde \g_u \|}\right) > \cos(\kappa)$. Let's say for a given $\kappa$, no such $t_3$ exists. Then, taking dot product with $\frac{\widetilde \g_u}{\| \widetilde \g_u \|}$ on both sides of Equation (\ref{eq:v:EWN:gd:2}), we can say
\begin{align*}
\frac{\bv_u(t+1)^\top\widetilde \g_u}{\| \widetilde \g_u \|} = &\frac{\bv_u(t)^\top\widetilde \g_u}{\| \widetilde \g_u \|} + \\
&\eta(t)\frac{e^{\alpha_u(t)}}{\| \bv_u(t) \|} \| \nabla_{\w_u} \L(\w(t)) \|\left(\frac{\widetilde \g_u}{\| \widetilde \g_u \|}\right)^\top \left(I - \frac{\bv_u(t)\bv_u(t)^\top}{\| \bv_u(t) \|^2}\right) \left( \frac{-\nabla_{\w_u} \L(\w(t))}{\| \nabla_{\w_u} \L(\w(t)) \|} \right)    
\end{align*}

Now, as $\left( \frac{\widetilde \g_u}{\| \widetilde \g_u \|} \right)^\top \left( \frac{-\nabla_{\w_u} \L(\w(t))}{\| \nabla_{\w_u} \L(\w(t)) \|} \right) \geq \cos(\tau)$ and $\left( \frac{\widetilde \g_u}{\| \widetilde \g_u \|} \right)^\top \left( \frac{\bv_u(t)}{\| \bv_u(t) \|} \right) \leq \cos(\kappa)$, we can say

\begin{equation}
\label{eq:ewn:gd:w:conv}
 \frac{\bv_u(t+1)^\top\widetilde \g_u}{\| \widetilde \g_u \|} \geq \frac{\bv_u(t)^\top\widetilde \g_u}{\| \widetilde \g_u \|} + (\cos(\tau) - \cos(\kappa))\left(\eta(t)\frac{e^{\alpha_u(t)}}{\| \bv_u(t) \|} \| \nabla_{\w_u} \L(\w(t)) \|\right)
\end{equation}

However, in this case, $\| \bv_u(t) \|$ doesn't stay constant and thus increase in dot product doesn't directly correspond to an increase in angle. Now, using Equation (\ref{eq:v:EWN:gd:2}), we can say
\begin{equation}
\label{eq:v:bound:EWN}
    \| \bv_u(t+1) \|^2 \leq \| \bv_u(t) \|^2 + \left(\eta(t)\frac{e^{\alpha_u(t)}}{\| \bv_u(t) \|} \| \nabla_{\w_u} \L(\w(t)) \|\right)^2
\end{equation}
Using the above two equations, we can say, for time $t>t_2$,
\begin{equation*}
    \frac{\bv_u(t+1)^\top\widetilde \g_u}{\| \bv_u(t+1) \|\| \widetilde \g_u \|} \geq \frac{\frac{\bv_u(t)^\top\widetilde \g_u}{\|\widetilde \g_u\|} + (\cos(\tau) - \cos(\kappa))\left(\eta(t)\frac{e^{\alpha_u(t)}}{\| \bv_u(t) \|} \| \nabla_{\w_u} \L(\w(t)) \|\right)}{\sqrt{\| \bv_u(t) \|^2 + \left(\eta(t)\frac{e^{\alpha_u(t)}}{\| \bv_u(t) \|} \| \nabla_{\w_u} \L(\w(t)) \|\right)^2}}
\end{equation*}
Unrolling the equation above, we get
\begin{equation}
    \label{eq:unitv:unitg:dot:EWN}
    \frac{\bv_u(t+1)^\top\widetilde \g_u}{\| \bv_u(t+1) \| \| \widetilde \g_u \|} \geq \frac{\frac{\bv_u(t_2)^\top\widetilde \g_u}{\|\widetilde \g_u\|} + \sum_{k=t_2}^{k=t} (\cos(\tau) - \cos(\kappa))\left(\eta(k)\frac{e^{\alpha_u(k)}}{\| \bv_u(k) \|} \| \nabla_{\w_u} \L(\w(k)) \|\right)}{\sqrt{\| \bv_u(t_2) \|^2 + \sum_{k=t_2}^{k=t}\left(\eta(k)\frac{e^{\alpha_u(k)}}{\| \bv_u(k) \|} \| \nabla_{\w_u} \L(\w(k)) \|\right)^2}}
\end{equation}
Now, as $\alpha_u(t) \to \infty$, therefore, using Equation (\ref{eq:alpha:EWN:gd:2}), we can say
\[ \sum_{k=t_2}^{k=\infty} \eta(k) e^{\alpha_u(k)} \| \nabla_{\w_u} \L(\w(k)) \| = \infty \]
Now, using this identity, along with the Assumption (A5), Equation (\ref{eq:v:bound:EWN}) and Lemma \ref{lemma:5}, we can say
\[ \sum_{k=t_2}^\infty \eta(k)\frac{e^{\alpha_u(k)}}{\| \bv_u(k) \|} \| \nabla_{\w_u} \L(\w(k)) \| = \infty \]
Using this along with Equation (\ref{eq:unitv:unitg:dot:EWN}) and Lemma \ref{lemma:4}, we can say
\[ \lim_{t \to \infty}  \frac{\bv_u(t+1)^\top\widetilde \g_u}{\| \bv_u(t+1) \| \| \widetilde \g_u \|} \geq \infty \]
However, this is not possible as the vectors on LHS have bounded norm. This contradicts. 
Thus there must exist a $t_3$ such that $\left( \frac{\bv_u(t_3)}{\| \bv_u(t_3) \|} \right)^\top \left(\frac{\widetilde \g_u}{\| \widetilde \g_u \|}\right) > \cos(\kappa)$.

Now, we are going to show that there exists a $t_4 > t_3$, such that for all $t > t_4$, 
$\left( \frac{\bv_u(t)}{\| \bv_u(t) \|} \right)^\top \left(\frac{\widetilde \g_u}{\| \widetilde \g_u \|}\right) > \cos(\kappa)$. Consider a $\beta$ such that $\tau < \beta < \kappa$. Now, if at any time t, $\left( \frac{\bv_u(t)}{\| \bv_u(t) \|} \right)^\top \left(\frac{\widetilde \g_u}{\| \widetilde \g_u \|}\right) < \cos(\beta)$, then, similar to Equation (\ref{eq:ewn:gd:w:conv}), we can say
\begin{equation*}
 \frac{\bv_u(t+1)^\top\widetilde \g_u}{\| \widetilde \g_u \|} \geq \frac{\bv_u(t)^\top\widetilde \g_u}{\| \widetilde \g_u \|} + (\cos(\tau) - \cos(\beta))\left(\eta(t)\frac{e^{\alpha_u(t)}}{\| \bv_u(t) \|} \| \nabla_{\w_u} \L(\w(t)) \|\right)
\end{equation*}
Using the upper bound on $\| \bv_u(t+1) \|$ from Equation (\ref{eq:v:bound:EWN}), we can say
\begin{equation}
\label{eq:ewn:gd:w:conv:2}
    \frac{\bv_u(t+1)^\top\widetilde \g_u}{\| \bv_u(t+1) \|\| \widetilde \g_u \|} \geq \frac{\frac{\bv_u(t)^\top\widetilde \g_u}{\|\widetilde \g_u\|} + (\cos(\tau) - \cos(\beta))\left(\eta(t)\frac{e^{\alpha_u(t)}}{\| \bv_u(t) \|} \| \nabla_{\w_u} \L(\w(t)) \|\right)}{\sqrt{\| \bv_u(t) \|^2 + \left(\eta(t)\frac{e^{\alpha_u(t)}}{\| \bv_u(t) \|} \| \nabla_{\w_u} \L(\w(t)) \|\right)^2}}
\end{equation}
Let $\eta(t)\frac{e^{\alpha_u(t)}}{\| \bv_u(t) \|} \| \nabla_{\w_u} \L(\w(t)) \|$ be denoted by $\chi(t)$. Then, the above equation can be rewritten as
\[  \frac{\bv_u(t+1)^\top\widetilde \g_u}{\| \bv_u(t+1) \|\| \widetilde \g_u \|} \geq \frac{\bv_u(t)^\top\widetilde \g_u}{\| \bv_u(t) \| \|\widetilde \g_u\|} \frac{\| \bv_u(t) \|}{\sqrt{\| \bv_u(t) \|^2 + \chi(t)^2}} + (\cos(\tau) - \cos(\beta))\frac{\chi(t)}{\sqrt{\| \bv_u(t) \|^2 + \chi(t)^2}} \]
Now, we are going to show that for a small enough $\chi(t)$, RHS is greater than $\frac{\bv_u(t)^\top\widetilde \g_u}{\| \bv_u(t) \| \|\widetilde \g_u\|}$.
\begin{flalign*}
    &\frac{\bv_u(t)^\top\widetilde \g_u}{\| \bv_u(t) \| \|\widetilde \g_u\|} \frac{\| \bv_u(t) \|}{\sqrt{\| \bv_u(t) \|^2 + \chi(t)^2}} + (\cos(\tau) - \cos(\beta))\frac{\chi(t)}{\sqrt{\| \bv_u(t) \|^2 + \chi(t)^2}} > \frac{\bv_u(t)^\top\widetilde \g_u}{\| \bv_u(t) \| \|\widetilde \g_u\|} \\
    &\implies (\cos(\tau) - \cos(\beta))\frac{\chi(t)}{\sqrt{\| \bv_u(t) \|^2 + \chi(t)^2}} > \frac{\bv_u(t)^\top\widetilde \g_u}{\| \bv_u(t) \| \|\widetilde \g_u\|} \left(1 - \frac{\| \bv_u(t) \|}{\sqrt{\| \bv_u(t) \|^2 + \chi(t)^2}}\right) \\
    &\implies (\cos(\tau) - \cos(\beta)) > \frac{\bv_u(t)^\top\widetilde \g_u}{\| \bv_u(t) \| \|\widetilde \g_u\|}\left(\frac{\sqrt{\| \bv_u(t) \|^2 + \chi(t)^2} - \| \bv_u(t) \|}{\chi(t)}\right)\\
\end{flalign*}

Clearly as $\chi(t) \to 0$, the RHS tends to 0, therefore the equation is satisfied. Thus for a small enough $\chi(t)$, RHS of Equation (\ref{eq:ewn:gd:w:conv:2}) is greater than $\frac{\bv_u(t)^\top\widetilde \g_u}{\| \bv_u(t) \| \|\widetilde \g_u\|}$.
As $\| \bv_u(t) \|$ keeps on increasing and by Assumption (B3), $\lim_{t \to \infty} \eta(t) \| \w_u(t) \| \| \nabla_{\w_u} \L(\w(t)) \| = 0$, we can say there exists a time $t_5$, such that for any $t > t_5$, $\frac{\bv_u(t)^\top\widetilde \g_u}{\| \bv_u(t) \| \|\widetilde \g_u\|}$ goes up whenever $\left( \frac{\bv_u(t)}{\| \bv_u(t) \|} \right)^\top \left(\frac{\widetilde \g_u}{\| \widetilde \g_u \|}\right) < \cos(\beta)$. 

Also, by using Equation (\ref{eq:v:EWN:gd:2}) and Assumption (B3), we can say, that there exists a time $t_6$, such that for $t > t_6$, $\left( \frac{\bv_u(t)}{\| \bv_u(t) \|} \right)^\top \left(\frac{\widetilde \g_u}{\| \widetilde \g_u \|}\right) > \cos(\beta) \implies \left( \frac{\bv_u(t+1)}{\| \bv_u(t+1) \|} \right)^\top \left(\frac{\widetilde \g_u}{\| \widetilde \g_u \|}\right) > \cos(\kappa)$, as the RHS of Equation (\ref{eq:v:EWN:gd:2}) goes to 0 norm in limit. Now, define $t_4 > max(t_5, t_6)$ such that $\left( \frac{\bv_u(t_4)}{\| \bv_u(t_4) \|} \right)^\top \left(\frac{\widetilde \g_u}{\| \widetilde \g_u \|}\right) > \cos(\kappa)$ (must exist from previous arguments). Then, as the dot product always goes up when between $\cos(\beta)$ and $\cos(\kappa)$, and can't go in a single step from being greater than $\cos(\beta)$ to less than $\cos(\kappa)$, therefore, for every $t > t_4$, $\left( \frac{\bv_u(t)}{\| \bv_u(t) \|} \right)^\top \left(\frac{\widetilde \g_u}{\| \widetilde \g_u \|}\right) > \cos(\kappa)$.

Now as the above argument holds for any $\kappa$ between $\tau$ and $\Delta$, and for any $\tau > 0$, we can say that $\w_u(t)$ converges in direction of $\widetilde \g_u$.
\end{proof}

\subsection{Standard Weight Normalization} \label{proof:swn:thm:weightconv}
\subsubsection{Gradient flow} \label{proof:swn:gf:thm:weightconv}
\begin{theorem*}
    Consider a node $u$ in the network with $\| \widetilde{\g}_u \| > 0$ and $\lim_{t \to \infty} \| \w_u(t) \| = \infty$. Under assumptions (A1), (A2) for gradient flow, for SWN
    \begin{enumerate}[label=(\roman*)]
        \begin{minipage}{0.5\textwidth}
        \item $\lim_{t \to \infty} \frac{\w_u(t)}{\| \w_u(t) \|} := \widetilde{\w}_u$ exists.
        \end{minipage}
        \begin{minipage}{0.5\textwidth}
        \item $\widetilde{\w}_u = \lambda \widetilde \g_u$ for some $\lambda > 0$.
        \end{minipage}
    \end{enumerate}
\end{theorem*}
\textbf{Update Equations:}
\begin{equation}
    \label{eq:gamma:SWN:gf}
    \frac{d\gamma_u(t)}{dt} = -\eta(t) \frac{\bv_u(t)^\top \nabla_{\w_u} \L(\w(t))}{\| \bv_u(t) \|}
\end{equation}
\begin{equation}
    \label{eq:v:SWN:gf}
    \frac{d \bv_u(t)}{dt} = -\eta(t)\frac{\gamma_u(t)}{\| \bv_u(t) \|}\left(I - \frac{\bv_u(t)\bv_u(t)^\top}{\| \bv_u(t) \|^2}\right)\nabla_{\w_u} \L(\w(t))
\end{equation}

\begin{proof}
The proof will be given for $\gamma_u \to \infty$. The one for $\gamma_u \to -\infty$ can be handled similarly.

As $\| \widetilde \g_u \| > 0$, therefore $\nabla_{\w_u} \L(\w(t))$ converges in direction. Therefore, for every $\tau$ satisfying $0 < \tau < 2\pi$, there exists a time $t_1(\tau)$, such that for $t > t_1(\tau)$,  $\left( \frac{-\nabla_{\w_u} \L(\w(t))}{\| \nabla_{\w_u} \L(\w(t))\|} \right)^\top \left(\frac{\widetilde \g_u}{\| \widetilde \g_u \|}\right) \geq \cos(\tau)$. Now, Let's assume that $\w_u(t)$ does not converge in the direction of $\widetilde \g_u$. Then, there must exist a $\tau$ satisfying $0 < \tau < 2\pi$, such that for this $\tau$, there exists a time $t_2 > t_1(\tau)$ satisfying $\bv_u(t_2)^\top \left(\frac{\widetilde \g_u}{\| \widetilde \g_u \|}\right) = \cos(\Delta)$, where $\Delta > \tau$.

Now, we are going to show that for any $\kappa$ satisfying $\tau < \kappa < \Delta$, there exists a time $t_3 > t_2$ such that $\bv_u(t_3)^\top \left(\frac{\widetilde \g_u}{\| \widetilde \g_u \|}\right) > \cos(\kappa)$. Let's say for a given $\kappa$, no such $t_3$ exists. Then, taking dot product with $\frac{\widetilde \g_u}{\| \widetilde \g_u \|}$ on both sides of Equation (\ref{eq:v:SWN:gf}), we can say
\[ \left( \frac{\widetilde \g_u}{\| \widetilde \g_u \|} \right)^\top \frac{d \bv_u(t)}{dt} = \eta(t) \gamma_u(t) \| \nabla_{\w_u} \L(\w(t)) \| \left( \frac{\widetilde \g_u}{\| \widetilde \g_u \|} \right)^\top (I - \bv_u(t)\bv_u(t)^\top) \left( \frac{-\nabla_{\w_u} \L(\w(t))}{\| \nabla_{\w_u} \L(\w(t)) \|} \right) \]
Now, as $\left( \frac{\widetilde \g_u}{\| \widetilde \g_u \|} \right)^\top \left( \frac{-\nabla_{\w_u} \L(\w(t))}{\| \nabla_{\w_u} \L(\w(t)) \|} \right) \geq \cos(\tau)$ and $\left( \frac{\widetilde \g_u}{\| \widetilde \g_u \|} \right)^\top \bv_u \leq \cos(\kappa)$, we can say
\begin{equation}
\label{eq:swn:gf:w:conv}
 \left( \frac{\widetilde \g_u}{\| \widetilde \g_u \|} \right)^\top \frac{d \bv_u(t)}{dt} \geq \eta(t) \gamma_u(t) \| \nabla_{\w_u} \L(\w(t)) \| (\cos(\tau) - \cos(\kappa))   
\end{equation}
Now, using the fact that $\gamma_u \to \infty$ and using Equation (\ref{eq:gamma:SWN:gf}), we can say
\[ \int_{t = t_2}^\infty \eta(t) \| \nabla_{\w_u} \L(\w(t)) \| dt = \infty \]
Using this fact and integrating the Equation (\ref{eq:swn:gf:w:conv}) on both the sides from $t_2$ to $\infty$, we get a contradiction as vectors on LHS have a finite norm while RHS tends to $\infty$. Thus, for every $\kappa$ between $\tau$ and $\Delta$, there must exist a $t_3$, such that $\bv_u(t_3)^\top \left(\frac{\widetilde \g_u}{\| \widetilde \g_u \|}\right) > \cos(\kappa)$.

Now, we are going to show for all $t \geq t_3$, $\bv_u(t)^\top \left(\frac{\widetilde \g_u}{\| \widetilde \g_u \|}\right) > \cos(\kappa)$. Now, consider any $\beta$ such that $\tau < \beta < \kappa$. Using similar argument as in Equation (\ref{eq:swn:gf:w:conv}), we can say, if for any $t_4 > t_3$, $\bv_u(t_4)^\top \left(\frac{\widetilde \g_u}{\| \widetilde \g_u \|}\right) < \cos(\beta)$, then 
\begin{equation}
\label{eq:swn:gf:w:conv:2}
 \left( \frac{\widetilde \g_u}{\| \widetilde \g_u \|} \right)^\top \frac{d \bv_u(t_4)}{dt} \geq \eta(t_4) \gamma_u(t_4) \| \nabla_{\w_u} \L(\w(t_4)) \| (\cos(\tau) - \cos(\beta))   
\end{equation}
This means that the dot product between $\left( \frac{\widetilde \g_u}{\| \widetilde \g_u \|} \right)$ and $\bv_u(t)$ goes up, whenever $\left( \frac{\widetilde \g_u}{\| \widetilde \g_u \|} \right)^\top \bv_u(t) < \cos(\tau)$. Therefore, its not possible that $\bv_u(t)^\top \left(\frac{\widetilde \g_u}{\| \widetilde \g_u \|}\right) \leq \cos(\kappa)$ for any $t > t_3$. As $\kappa$ can be arbitrarily chosen between $\tau$ and $\Delta$, $\w_u(t)$ converges in the direction of $\widetilde \g_u$
\end{proof}

\subsubsection{Gradient Descent} \label{proof:swn:gd:thm:weightconv}
\begin{theorem*}
    Consider a node $u$ in the network with $\| \widetilde{\g}_u \| > 0$ and $\lim_{t \to \infty} \| \w_u(t) \| = \infty$. Under assumptions (B1)-(B3) for gradient descent, for SWN
    \begin{enumerate}[label=(\roman*)]
        \begin{minipage}{0.5\textwidth}
        \item $\lim_{t \to \infty} \frac{\w_u(t)}{\| \w_u(t) \|} := \widetilde{\w}_u$ exists.
        \end{minipage}
        \begin{minipage}{0.5\textwidth}
        \item $\widetilde{\w}_u = \lambda \widetilde \g_u$ for some $\lambda > 0$.
        \end{minipage}
    \end{enumerate}
\end{theorem*}
\textbf{Update Equations:}
\begin{equation}
\label{eq:gamma:SWN:gd:2}
    \gamma_u(t+1) = \gamma_u(t) - \eta(t)\frac{\bv_u(t)^\top \nabla_{\w_u} \L(\w(t))}{\| \bv_u(t) \|}
\end{equation}
\begin{equation}
\label{eq:v:SWN:gd:2}
    \bv_u(t+1) = \bv_u(t) - \eta(t)\frac{\gamma_u(t)}{\| \bv_u(t) \|}\left(I - \frac{\bv_u(t)\bv_u(t)^\top}{\| \bv_u(t) \|^2}\right)\nabla_{\w_u} \L(\w(t))
\end{equation}

\begin{proof}
The proof will be given for $\gamma_u \to \infty$. The one for $\gamma_u \to -\infty$ can be handled similarly.

As $\| \widetilde \g_u \| > 0$, therefore $\nabla_{\w_u} \L(\w(t))$ converges in direction. Therefore, for every $\tau$ satisfying $0 < \tau < 2\pi$, there exists a time $t_1(\tau)$, such that for $t > t_1(\tau)$,  $\left( \frac{-\nabla_{\w_u} \L(\w(t))}{\| \nabla_{\w_u} \L(\w(t))\|} \right)^\top \left(\frac{\widetilde \g_u}{\| \widetilde \g_u \|}\right) \geq \cos(\tau)$. Now, Let's assume that $\w_u(t)$ does not converge in the direction of $\widetilde \g_u$. Then, there must exist a $\tau$ satisfying $0 < \tau < 2\pi$, such that for this $\tau$, there exists a time $t_2 > t_1(\tau)$ satisfying $\bv_u(t_2)^\top \left(\frac{\widetilde \g_u}{\| \widetilde \g_u \|}\right) = \cos(\Delta)$, where $\Delta > \tau$.

Now, we are going to show that for any $\kappa$ satisfying $\tau < \kappa < \Delta$, there exists a time $t_3 > t_2$ such that $\left( \frac{\bv_u(t_3)}{\| \bv_u(t_3) \|} \right)^\top \left(\frac{\widetilde \g_u}{\| \widetilde \g_u \|}\right) > \cos(\kappa)$. Let's say for a given $\kappa$, no such $t_3$ exists. Then, taking dot product with $\frac{\widetilde \g_u}{\| \widetilde \g_u \|}$ on both sides of Equation (\ref{eq:v:SWN:gd:2}), we can say
\begin{align*}
\frac{\bv_u(t+1)^\top\widetilde \g_u}{\| \widetilde \g_u \|} = &\frac{\bv_u(t)^\top\widetilde \g_u}{\| \widetilde \g_u \|} + \\
&\eta(t)\frac{\gamma_u(t)}{\| \bv_u(t) \|} \| \nabla_{\w_u} \L(\w(t)) \|\left(\frac{\widetilde \g_u}{\| \widetilde \g_u \|}\right)^\top \left(I - \frac{\bv_u(t)\bv_u(t)^\top}{\| \bv_u(t) \|^2}\right) \left( \frac{-\nabla_{\w_u} \L(\w(t))}{\| \nabla_{\w_u} \L(\w(t)) \|} \right)    
\end{align*}

Now, as $\left( \frac{\widetilde \g_u}{\| \widetilde \g_u \|} \right)^\top \left( \frac{-\nabla_{\w_u} \L(\w(t))}{\| \nabla_{\w_u} \L(\w(t)) \|} \right) \geq \cos(\tau)$ and $\left( \frac{\widetilde \g_u}{\| \widetilde \g_u \|} \right)^\top \left( \frac{\bv_u(t)}{\| \bv_u(t) \|} \right) \leq \cos(\kappa)$, we can say

\begin{equation}
\label{eq:swn:gd:w:conv}
 \frac{\bv_u(t+1)^\top\widetilde \g_u}{\| \widetilde \g_u \|} \geq \frac{\bv_u(t)^\top\widetilde \g_u}{\| \widetilde \g_u \|} + (\cos(\tau) - \cos(\kappa))\left(\eta(t)\frac{\gamma_u(t)}{\| \bv_u(t) \|} \| \nabla_{\w_u} \L(\w(t)) \|\right)
\end{equation}
However, in this case, $\| \bv_u(t) \|$ doesn't stay constant and thus increase in dot product doesn't directly correspond to an increase in angle. Now, using Equation (\ref{eq:v:SWN:gd:2}), we can say
\begin{equation}
\label{eq:v:bound:SWN:gd}
    \| \bv_u(t+1) \|^2 \leq \| \bv_u(t) \|^2 + \left(\eta(t)\frac{\gamma_u(t)}{\| \bv_u(t) \|} \| \nabla_{\w_u} \L(\w(t)) \|\right)^2
\end{equation}
Using the above two equations, we can say, for time $t>t_2$,
\begin{equation*}
    \frac{\bv_u(t+1)^\top\widetilde \g_u}{\| \bv_u(t+1) \|\| \widetilde \g_u \|} \geq \frac{\frac{\bv_u(t)^\top\widetilde \g_u}{\|\widetilde \g_u\|} + \epsilon\left(\eta(t)\frac{\gamma_u(t)}{\| \bv_u(t) \|} \| \nabla_{\w_u} \L(\w(t)) \|\right)}{\sqrt{\| \bv_u(t) \|^2 + \left(\eta(t)\frac{\gamma_u(t)}{\| \bv_u(t) \|} \| \nabla_{\w_u} \L(\w(t)) \|\right)^2}}
\end{equation*}
Unrolling the equation above, we get
\begin{equation}
    \label{eq:unitv:unitg:dot:SWN:gd}
    \frac{\bv_u(t+1)^\top\widetilde \g_u}{\| \bv_u(t+1) \| \| \widetilde \g_u \|} \geq \frac{\frac{\bv_u(t_2)^\top\widetilde \g_u}{\|\widetilde \g_u\|} + \sum_{k=t_2}^{k=t} \epsilon\left(\eta(k)\frac{\gamma_u(k)}{\| \bv_u(k) \|} \| \nabla_{\w_u} \L(\w(k)) \|\right)}{\sqrt{\| \bv_u(t_2) \|^2 + \sum_{k=t_2}^{k=t}\left(\eta(k)\frac{\gamma_u(k)}{\| \bv_u(k) \|} \| \nabla_{\w_u} \L(\w(k)) \|\right)^2}}
\end{equation}
Now, as $\gamma_u(t) \to \infty$, therefore, using Equation (\ref{eq:v:SWN:gd:2}), we can say
\[ \sum_{k=t_2}^{k=\infty} \eta(k) \| \nabla_{\w_u} \L(\w(k)) \| = \infty \]
Now, using this identity, along with the assumption (A5), Equation (\ref{eq:v:bound:SWN:gd}) and Lemma \ref{lemma:5}, we can say
\[ \sum_{k=t_2}^\infty \eta(k)\frac{\gamma_u(k)}{\| \bv_u(k) \|} \| \nabla_{\w_u} \L(\w(k)) \| = \infty \]
Using this along with Equation (\ref{eq:unitv:unitg:dot:SWN:gd}) and Lemma \ref{lemma:4}, we can say
\[ \lim_{t \to \infty}  \frac{\bv_u(t+1)^\top\widetilde \g_u}{\| \bv_u(t+1) \| \| \widetilde \g_u \|} \geq \infty \]
However, this is not possible as the vectors on LHS have bounded norm. This contradicts. Thus there must exist a $t_3$ such that $\left( \frac{\bv_u(t_3)}{\| \bv_u(t_3) \|} \right)^\top \left(\frac{\widetilde \g_u}{\| \widetilde \g_u \|}\right) > \cos(\kappa)$.

Now, we are going to show that there exists a $t_4 > t_3$, such that for all $t > t_4$, 
$\left( \frac{\bv_u(t)}{\| \bv_u(t) \|} \right)^\top \left(\frac{\widetilde \g_u}{\| \widetilde \g_u \|}\right) > \cos(\kappa)$. Consider a $\beta$ such that $\tau < \beta < \kappa$. Now, if at any time t, $\left( \frac{\bv_u(t)}{\| \bv_u(t) \|} \right)^\top \left(\frac{\widetilde \g_u}{\| \widetilde \g_u \|}\right) < \cos(\beta)$, then, similar to Equation (\ref{eq:swn:gd:w:conv}), we can say
\begin{equation*}
 \frac{\bv_u(t+1)^\top\widetilde \g_u}{\| \widetilde \g_u \|} \geq \frac{\bv_u(t)^\top\widetilde \g_u}{\| \widetilde \g_u \|} + (\cos(\tau) - \cos(\beta))\left(\eta(t)\frac{\gamma_u(t)}{\| \bv_u(t) \|} \| \nabla_{\w_u} \L(\w(t)) \|\right)
\end{equation*}
Using the upper bound on $\| \bv_u(t+1) \|$ from Equation (\ref{eq:v:bound:SWN:gd}), we can say
\begin{equation}
\label{eq:swn:gd:w:conv:2}
    \frac{\bv_u(t+1)^\top\widetilde \g_u}{\| \bv_u(t+1) \|\| \widetilde \g_u \|} \geq \frac{\frac{\bv_u(t)^\top\widetilde \g_u}{\|\widetilde \g_u\|} + (\cos(\tau) - \cos(\beta))\left(\eta(t)\frac{\gamma_u(t)}{\| \bv_u(t) \|} \| \nabla_{\w_u} \L(\w(t)) \|\right)}{\sqrt{\| \bv_u(t) \|^2 + \left(\eta(t)\frac{\gamma_u(t)}{\| \bv_u(t) \|} \| \nabla_{\w_u} \L(\w(t)) \|\right)^2}}
\end{equation}
Let $\eta(t)\frac{\gamma_u(t)}{\| \bv_u(t) \|} \| \nabla_{\w_u} \L(\w(t)) \|$ be denoted by $\chi(t)$. Then, the above equation can be rewritten as
\[  \frac{\bv_u(t+1)^\top\widetilde \g_u}{\| \bv_u(t+1) \|\| \widetilde \g_u \|} \geq \frac{\bv_u(t)^\top\widetilde \g_u}{\| \bv_u(t) \| \|\widetilde \g_u\|} \frac{\| \bv_u(t) \|}{\sqrt{\| \bv_u(t) \|^2 + \chi(t)^2}} + (\cos(\tau) - \cos(\beta))\frac{\chi(t)}{\sqrt{\| \bv_u(t) \|^2 + \chi(t)^2}} \]
Now, we are going to show that for a small enough $\chi(t)$, RHS is greater than $\frac{\bv_u(t)^\top\widetilde \g_u}{\| \bv_u(t) \| \|\widetilde \g_u\|}$.
\begin{flalign*}
    &\frac{\bv_u(t)^\top\widetilde \g_u}{\| \bv_u(t) \| \|\widetilde \g_u\|} \frac{\| \bv_u(t) \|}{\sqrt{\| \bv_u(t) \|^2 + \chi(t)^2}} + (\cos(\tau) - \cos(\beta))\frac{\chi(t)}{\sqrt{\| \bv_u(t) \|^2 + \chi(t)^2}} > \frac{\bv_u(t)^\top\widetilde \g_u}{\| \bv_u(t) \| \|\widetilde \g_u\|} \\
    &\implies (\cos(\tau) - \cos(\beta))\frac{\chi(t)}{\sqrt{\| \bv_u(t) \|^2 + \chi(t)^2}} > \frac{\bv_u(t)^\top\widetilde \g_u}{\| \bv_u(t) \| \|\widetilde \g_u\|} \left(1 - \frac{\| \bv_u(t) \|}{\sqrt{\| \bv_u(t) \|^2 + \chi(t)^2}}\right) \\
    &\implies (\cos(\tau) - \cos(\beta)) > \frac{\bv_u(t)^\top\widetilde \g_u}{\| \bv_u(t) \| \|\widetilde \g_u\|}\left(\frac{\sqrt{\| \bv_u(t) \|^2 + \chi(t)^2} - \| \bv_u(t) \|}{\chi(t)}\right)\\
\end{flalign*}

Clearly as $\chi(t) \to 0$, the RHS tends to 0, therefore the equation is satisfied. Thus for a small enough $\chi(t)$, RHS of Equation (\ref{eq:swn:gd:w:conv:2}) is greater than $\frac{\bv_u(t)^\top\widetilde \g_u}{\| \bv_u(t) \| \|\widetilde \g_u\|}$.
As $\| \bv_u(t) \|$ keeps on increasing and by Assumption (A5), $\lim_{t \to \infty} \eta(t) \gamma_u(t) \| \nabla_{\w_u} \L(\w(t)) \| = 0$, we can say there exists a time $t_5$, such that for any $t > t_5$, $\frac{\bv_u(t)^\top\widetilde \g_u}{\| \bv_u(t) \| \|\widetilde \g_u\|}$ goes up whenever $\left( \frac{\bv_u(t)}{\| \bv_u(t) \|} \right)^\top \left(\frac{\widetilde \g_u}{\| \widetilde \g_u \|}\right) < \cos(\beta)$. 

Also, by using Equation (\ref{eq:v:SWN:gd:2}) and Assumption (A5), we can say, that there exists a time $t_6$, such that for $t > t_6$, $\left( \frac{\bv_u(t)}{\| \bv_u(t) \|} \right)^\top \left(\frac{\widetilde \g_u}{\| \widetilde \g_u \|}\right) > \cos(\beta) \implies \left( \frac{\bv_u(t+1)}{\| \bv_u(t+1) \|} \right)^\top \left(\frac{\widetilde \g_u}{\| \widetilde \g_u \|}\right) > \cos(\kappa)$, as the RHS of Equation (\ref{eq:v:SWN:gd:2}) goes to 0 norm in limit. Now, define $t_4 > max(t_5, t_6)$ such that $\left( \frac{\bv_u(t_4)}{\| \bv_u(t_4) \|} \right)^\top \left(\frac{\widetilde \g_u}{\| \widetilde \g_u \|}\right) > \cos(\kappa)$ (must exist from previous arguments). Then, as the dot product always goes up when between $\cos(\beta)$ and $\cos(\kappa)$, and can't go in a single step from being greater than $\cos(\beta)$ to less than $\cos(\kappa)$, therefore, for every $t > t_4$, $\left( \frac{\bv_u(t)}{\| \bv_u(t) \|} \right)^\top \left(\frac{\widetilde \g_u}{\| \widetilde \g_u \|}\right) > \cos(\kappa)$.

Now as the above argument holds for any $\kappa$ between $\tau$ and $\Delta$, and for any $\tau > 0$, we can say that $\w_u(t)$ converges in direction of $\widetilde \g_u$.
\end{proof}

\section{Proof of Theorem \ref{thm:weightnorm}} \label{proof:thm:weightnorm}
\begin{theorem*}
Consider two nodes $u$ and $v$ in the network with $\| \widetilde \g_u \| \geq \|\widetilde \g_v \|>0, \lim_{t \to \infty} \| \w_u(t) \| = \infty$ and $\lim_{t \to \infty} \| \w_v(t) \| = \infty$. Let $\frac{\| \widetilde \g_u \|}{\| \widetilde \g_v \|}$ be denoted by $c$. Under assumptions (A1), (A2) for gradient flow and (B1)-(B3) for gradient descent,
    % Under assumptions (A1)-(A3) for gradient flow and (A1)-(A4) for gradient descent, for all nodes $u$ and $v$ in the network, the following hold
    \begin{enumerate}[label=(\roman*)]
        \item for SWN, $\lim_{t\to\infty} \frac{\|\w_u(t)\|}{\|\w_v(t)\|}=c $
        \item for EWN, $\lim_{t\to\infty} \frac{\|\w_u(t)\|}{\|\w_v(t)\|}$ is either $0, \infty$ or $\frac{1}{c}$
    \end{enumerate}
\end{theorem*}
Proof for different cases will be split into different subsections and the corresponding case will be restated there for ease of the reader.
\subsection{Exponential Weight Normalization} \label{proof:ewn:thm:weightnorm}
\subsubsection{Gradient Flow} \label{proof:ewn:gf:thm:weightnorm}
\begin{theorem*}
Consider two nodes $u$ and $v$ in the network with $\| \widetilde \g_u \| \geq \|\widetilde \g_v \|>0, \lim_{t \to \infty} \| \w_u(t) \| = \infty$ and $\lim_{t \to \infty} \| \w_v(t) \| = \infty$. Let $\frac{\| \widetilde \g_u \|}{\| \widetilde \g_v \|}$ be denoted by $c$. Under assumptions (A1), (A2) for gradient flow, for EWN, $\lim_{t\to\infty} \frac{\|\w_u(t)\|}{\|\w_v(t)\|}$ is either $0, \infty$ or $\frac{1}{c}$
\end{theorem*}
    
\begin{proof}
Using Theorem \ref{thm:weightconv} and the fact that $\lim_{t \to \infty} \frac{\| \nabla_{\w_u} \L(\w(t)) \|}{\| \nabla_{\w_v} \L(\w(t)) \|} = c$, for any $0 < \epsilon < c$ and $0 < \delta < 2\pi$, there exists a time $t_1$, such that for $t > t_1$, the following hold
\begin{enumerate}[label=(\roman*)]
    \begin{minipage}{0.4\textwidth}
    \item $\frac{\| \nabla_{\w_u} \L(\w(t))\|}{\| \nabla_{\w_v} \L(\w(t))\|} \in [c - \epsilon, c + \epsilon]$
    \end{minipage}
    \begin{minipage}{0.5\textwidth}
    \item $\left( \frac{\w_u(t)}{\| \w_u(t) \|}\right)^\top \left(\frac{-\nabla_{\w_u} \L(\w(t))}{\| \nabla_{\w_u} \L(\w(t)) \|}\right) \geq \cos(\delta)$
    \end{minipage}
    \item $\left( \frac{\w_v(t)}{\| \w_v(t) \|}\right)^\top \left(\frac{-\nabla_{\w_v} \L(\w(t))}{\| \nabla_{\w_v} \L(\w(t)) \|}\right) \geq \cos(\delta)$.
\end{enumerate}

Now, we will provide the proof of part (iii) of Proposition \ref{prop:3}, i.e, 
for EWN, if at some time $t_2 > t_1$,
\begin{enumerate}[label=(\alph*)]
    \item $\frac{\| \w_u(t_2) \|}{\| \w_v(t_2) \|} > \frac{1}{(c-\epsilon)\cos(\delta)} \implies \lim_{t \to \infty} \frac{\| \w_u(t) \|}{\| \w_v(t) \|} = \infty$
    \item $\frac{\| \w_u(t_2) \|}{\| \w_v(t_2) \|} < \frac{\cos(\delta)}{c + \epsilon} \implies \lim_{t \to \infty} \frac{\| \w_u(t) \|}{\| \w_v(t) \|} = 0$
\end{enumerate}
        
(a) $\frac{\| \w_u(t_2) \|}{\| \w_v(t_2) \|} > \frac{1}{(c - \epsilon) \cos(\delta)} \implies \lim_{t \to \infty} \frac{\| \w_u(t) \|}{\| \w_v(t) \|} = \infty$

Using Equation (\ref{eq:alpha:EWN:gf}), 
\begin{equation}
\label{eq:normw:ewn:gf}
\frac{d \| \w_u(t) \| }{dt} = \frac{d e^{\alpha_u(t)}}{dt} = -\eta(t) \| \w_u(t) \|^2 (\bv_u(t)^\top \nabla_{\w_u} \L(\w(t))    
\end{equation}

Using the equation above, we can say for $t > t_1$,
\begin{align}
    \label{eq:wratio:lowbound:EWN:gf}
    \frac{d \frac{\| \w_u(t) \|}{\| \w_v(t) \|}}{dt} &= \frac{\| \w_v(t) \|\frac{d \| \w_u(t) \|}{dt} -   \| \w_u(t) \|\frac{d \| \w_v(t) \|}{dt}}{\| \w_v(t) \|^2} \nonumber \\
    &\geq \eta(t)\frac{\| \w_u(t) \|}{\| \w_v(t) \|}(\|\w_u(t)\|\|\nabla_{\w_u} \L(\w(t)) \| \cos(\delta) - \|\w_v(t)\|\|\nabla_{\w_v} \L(\w(t)) \|) \nonumber \\    
    &\geq \eta(t)\| \w_u(t) \|\|\nabla_{\w_u} \L(\w(t)) \|\left(\frac{\| \w_u(t) \|}{\| \w_v(t) \|} \cos(\delta) - \frac{1}{c - \epsilon}\right)
\end{align}

In this case, using Equation (\ref{eq:wratio:lowbound:EWN:gf}), we can see $\frac{d \frac{\| \w_u(t) \|}{\| \w_v(t) \|}}{dt} > 0$ at $t_2$. Thus, $\frac{\| \w_u(t) \|}{\| \w_v(t) \|}$ always remains greater than $\frac{1}{(c - \epsilon) \cos(\delta)}$ and keeps on increasing. Let's denote $\frac{\|\w_u(t_2)\|}{\|\w_v(t_2)\|}$ by $\Delta$. Then, for $t > t_2$,
\[ \frac{d \frac{\| \w_u(t) \|}{\| \w_v(t) \|}}{dt} \geq \left(\Delta\cos(\delta) - \frac{1}{c - \epsilon}\right) \eta(t) \|\w_u(t)\|\|\nabla_{\w_u} \L(\w(t)) \| \]
As $\alpha_u \to \infty$, therefore using Equation (\ref{eq:alpha:EWN:gf}), we can say $\int_{t_2}^\infty \eta(t) \|\w_u(t)\|\|\nabla_{\w_u} \L(\w(t)) \| dt \to \infty$. Thus, integrating both the sides of the equation above from $t_2$ to $\infty$, we get
\[ \int_{t_2}^\infty \frac{d \frac{\| \w_u(t) \|}{\| \w_v(t) \|}}{dt} dt \geq \infty \]
Thus $\lim_{t \to \infty} \frac{\| \w_u(t) \|}{\| \w_v(t) \|} = \infty$.

(b) $\frac{\| \w_u(t_2) \|}{\| \w_v(t_2) \|} < \frac{\cos(\delta)}{c + \epsilon} \implies \lim_{t \to \infty} \frac{\| \w_u(t) \|}{\| \w_v(t) \|} = 0$.

Using Equation (\ref{eq:normw:ewn:gf}), we can say for $t > t_1$,
\begin{align}
    \label{eq:wratio:uppbound:EWN:gf}
    \frac{d \frac{\| \w_u(t) \|}{\| \w_v(t) \|}}{dt} &= \frac{\| \w_v(t) \|\frac{d \| \w_u(t) \|}{dt} -   \| \w_u(t) \|\frac{d \| \w_v(t) \|}{dt}}{\| \w_v(t) \|^2} \nonumber \\
    &\leq \eta(t)\frac{\| \w_u(t) \|}{\| \w_v(t) \|}(\|\w_u(t)\|\|\nabla_{\w_u} \L(\w(t)) \| - \|\w_v(t)\|\|\nabla_{\w_v} \L(\w(t)) \| \cos(\delta))\nonumber\\
    &= (\eta(t) \| \w_v(t) \| \| \nabla_{\w_v} \L(\w(t)) \|) \frac{\| \w_u(t) \|}{\| \w_v(t) \|}\left(\frac{\| \w_u \| \| \nabla_{\w_u} \L(\w(t)) \|}{\| \w_v \| \| \nabla_{\w_v} \L(\w(t)) \|} - \cos(\delta)\right)\nonumber\\
    &\leq (\eta(t) \| \w_v(t) \| \| \nabla_{\w_v} \L(\w(t)) \|) \frac{\| \w_u(t) \|}{\| \w_v(t) \|}\left(\frac{\| \w_u \|}{\| \w_v \|} (c + \epsilon) - \cos(\delta)\right)
\end{align}

In this case, using Equation (\ref{eq:wratio:uppbound:EWN:gf}), we can see $\frac{d \frac{\| \w_u(t) \|}{\| \w_v(t) \|}}{dt} < 0$ at $t_2$. Thus, $\frac{\| \w_u(t) \|}{\| \w_v(t) \|}$ always remains smaller than $\frac{\cos(\delta)}{c + \epsilon}$ and keeps on decreasing. Now, lets say $\lim_{t \to \infty} \frac{\| \w_u(t) \|}{\| \w_v(t) \|} > 0$. This means that $\frac{\| \w_u(t) \|}{\| \w_v(t) \|} > \Delta$, for some $\Delta > 0$. Also, let's denote $\frac{\|\w_u(t_2)\|}{\|\w_v(t_2)\|}$ by $\beta$. Then we can say
\[ \frac{d \frac{\| \w_u(t) \|}{\| \w_v(t) \|}}{dt} \leq - \Delta(\cos(\delta) - \beta(c + \epsilon)) \eta(t) \|\w_v(t)\|\|\nabla_{\w_v} \L(\w(t)) \| \]
As $\alpha_v \to \infty$, therefore using Equation (\ref{eq:alpha:EWN:gf}), we can say $\int_{t_2}^\infty \eta(t) \|\w_v(t)\|\|\nabla_{\w_v} \L(\w(t)) \| dt = \infty$. Thus, integrating both the sides of the equation above from $t_2$ to $\infty$, we get
\[ \int_{t_2}^\infty \frac{d \frac{\| \w_u(t) \|}{\| \w_v(t) \|}}{dt} dt \leq -\infty \]
This is not possible as $\frac{\| \w_u(t) \|}{\| \w_v(t) \|}$ is lower bounded by 0. Thus $\lim_{t \to \infty} \frac{\| \w_u(t) \|}{\| \w_v(t) \|} = 0$.

Now, as $\epsilon$ and $\delta$ tend to 0, the length of the interval of stability $[\frac{\cos(\delta)}{c + \epsilon}, \frac{1}{(c - \epsilon)\cos(\delta)}]$ shrinks to zero, around the point $\frac{1}{c}$. Thus, either $\frac{\| \w_u(t) \|}{\| \w_v(t) \|}$ moves out of the interval of stability and converges to either $0$ or $\infty$, or it always remains within the interval of stability and converges to $\frac{1}{c}$.
\end{proof}

\subsubsection{Gradient Descent} \label{proof:ewn:gd:thm:weightnorm}
\begin{theorem*}
Consider two nodes $u$ and $v$ in the network with $\| \widetilde \g_u \| \geq \|\widetilde \g_v \|>0, \lim_{t \to \infty} \| \w_u(t) \| = \infty$ and $\lim_{t \to \infty} \| \w_v(t) \| = \infty$. Let $\frac{\| \widetilde \g_u \|}{\| \widetilde \g_v \|}$ be denoted by $c$. Under assumptions (B1)-(B3) for gradient descent, for EWN, $\lim_{t\to\infty} \frac{\|\w_u(t)\|}{\|\w_v(t)\|}$ is either $0, \infty$ or $\frac{1}{c}$
\end{theorem*}

\begin{proof}
Using Theorem \ref{thm:weightconv} and the fact that $\lim_{t \to \infty} \frac{\| \nabla_{\w_u} \L(\w(t)) \|}{\| \nabla_{\w_v} \L(\w(t)) \|} = c$, for any $0 < \epsilon < c$ and $0 < \delta < 2\pi$, there exists a time $t_1$, such that for $t > t_1$, the following hold
\begin{enumerate}[label=(\roman*)]
    \begin{minipage}{0.4\textwidth}
    \item $\frac{\| \nabla_{\w_u} \L(\w(t))\|}{\| \nabla_{\w_v} \L(\w(t))\|} \in [c - \epsilon, c + \epsilon]$
    \end{minipage}
    \begin{minipage}{0.5\textwidth}
    \item $\left( \frac{\w_u(t)}{\| \w_u(t) \|}\right)^\top \left(\frac{-\nabla_{\w_u} \L(\w(t))}{\| \nabla_{\w_u} \L(\w(t)) \|}\right) \geq \cos(\delta)$
    \end{minipage}
    \item $\left( \frac{\w_v(t)}{\| \w_v(t) \|}\right)^\top \left(\frac{-\nabla_{\w_v} \L(\w(t))}{\| \nabla_{\w_v} \L(\w(t)) \|}\right) \geq \cos(\delta)$.
\end{enumerate}

Now, we will provide the proof of part (iii) of Proposition \ref{prop:3}, i.e, 
for EWN, if at some time $t_2 > t_1$,
\begin{enumerate}[label=(\alph*)]
    \item $\frac{\| \w_u(t_2) \|}{\| \w_v(t_2) \|} > \frac{1}{(c-\epsilon)\cos(\delta)} \implies \lim_{t \to \infty} \frac{\| \w_u(t) \|}{\| \w_v(t) \|} = \infty$
    \item $\frac{\| \w_u(t_2) \|}{\| \w_v(t_2) \|} < \frac{\cos(\delta)}{c + \epsilon} \implies \lim_{t \to \infty} \frac{\| \w_u(t) \|}{\| \w_v(t) \|} = 0$
\end{enumerate}

(a) $\frac{\| \w_u(t_2) \|}{\| \w_v(t_2) \|} > \frac{1}{(c-\epsilon)\cos(\delta)} \implies \lim_{t \to \infty} \frac{\| \w_u(t) \|}{\| \w_v(t) \|} = \infty$

Using Equation (\ref{eq:alpha:EWN:gd:2}) and part 1 of the Proposition, we can say
\begin{align*}
    \frac{\| \w_u(t_2 + 1) \|}{\| \w_v(t_2 + 1) \|} &\geq \frac{ \| \w_u(t_2) \| + \eta(t_2)\cos(\delta)\| \w_u(t_2) \|^2\| \nabla_{\w_u} \L(\w(t_2)) \|}{ \| \w_v(t_2) \| + \eta(t_2)\| \w_v(t_2) \|^2\| \nabla_{\w_v} \L(\w(t_2)) \|} \\
    &= \frac{\| \w_u(t_2) \|}{\| \w_v(t_2) \|}\left(\frac{1 + \cos(\delta)\eta(t_2)\| \w_u(t_2) \|\| \nabla_{\w_u} \L(\w(t_2)) \|}{1 + \eta(t_2)\| \w_v(t_2) \|\| \nabla_{\w_v} \L(\w(t_2)) \|}\right)\\
    &\geq \frac{\| \w_u(t_2) \|}{\| \w_v(t_2) \|}\\
\end{align*}
Thus, $\frac{\| \w_u(t) \|}{\| \w_v(t) \|}$ keeps on increasing for $t>t_2$. It can either diverge to infinity or converge to a finite value. If it converges to a finite value, then by Stolz Cesaro theorem,
\[ \lim_{t \to \infty} \frac{\| \w_u(t) \|}{\| \w_v(t) \|} = \lim_{t \to \infty} \frac{\| \w_u(t) \|^2 \| \nabla_{\w_u} \L(\w(t))\|}{\| \w_v(t) \|^2 \| \nabla_{\w_v} \L(\w(t))\|} \]
However, this is not possible as $\frac{\| \w_u(t) \|}{\| \w_v(t) \|} > \frac{1}{c}$ for every $t > t_2$. Thus, $\frac{\| \w_u(t) \|}{\| \w_v(t) \|}$ diverges to infinity.

(b) $\frac{\| \w_u(t_2) \|}{\| \w_v(t_2) \|} < \frac{\cos(\delta)}{c + \epsilon} \implies \lim_{t \to \infty} \frac{\| \w_u(t) \|}{\| \w_v(t) \|} = 0$

Using Equation (\ref{eq:alpha:EWN:gd:2}) and part 1 of the Proposition, we can say
\begin{align*}
    \frac{\| \w_u(t_2 + 1) \|}{\| \w_v(t_2 + 1) \|} &\leq \frac{ \| \w_u(t_2) \| + \eta(t_2)\| \w_u(t_2) \|^2\| \nabla_{\w_u} \L(\w(t_2)) \|}{ \| \w_v(t_2) \| + \eta(t_2)\cos(\delta)\| \w_v(t_2) \|^2\| \nabla_{\w_v} \L(\w(t_2)) \|} \\
    &= \frac{\| \w_u(t_2) \|}{\| \w_v(t_2) \|}\left(\frac{1 + \eta(t_2)\| \w_u(t_2) \|\| \nabla_{\w_u} \L(\w(t_2)) \|}{1 + \eta(t_2)\cos(\delta)\| \w_v(t_2) \|\| \nabla_{\w_v} \L(\w(t_2)) \|}\right)\\
    &\leq \frac{\| \w_u(t_2) \|}{\| \w_v(t_2) \|}\\
\end{align*}
Thus, $\frac{\| \w_u(t) \|}{\| \w_v(t) \|}$ keeps on decreasing for $t>t_2$. As it is always greater than zero, it must converge. Therefore, by Stolz Cesaro Theorem,
\[ \lim_{t \to \infty} \frac{\| \w_u(t) \|}{\| \w_v(t) \|} = \lim_{t \to \infty} \frac{\| \w_u(t) \|^2 \| \nabla_{\w_u} \L(\w(t))\|}{\| \w_v(t) \|^2 \| \nabla_{\w_v} \L(\w(t))\|} \]
For $\frac{\| \w_u(t) \|}{\| \w_v(t) \|} < \frac{1}{c}$, this can only be satisfied when $\lim_{t \to \infty} \frac{\| \w_u(t) \|}{\| \w_v(t) \|} = 0$.

Now, as $\epsilon$ and $\delta$ tend to 0, the length of the interval of stability $[\frac{\cos(\delta)}{c + \epsilon}, \frac{1}{(c - \epsilon)\cos(\delta)}]$ shrinks to zero, around the point $\frac{1}{c}$. Thus, either $\frac{\| \w_u(t) \|}{\| \w_v(t) \|}$ moves out of the interval of stability and converges to either $0$ or $\infty$, or it always remains within the interval of stability and converges to $\frac{1}{c}$.
\end{proof}

\subsection{Standard Weight Normalization} \label{proof:swn:thm:weightnorm}
\subsubsection{Gradient Flow} \label{proof:swn:gf:thm:weightnorm}
\begin{theorem*}
Consider two nodes $u$ and $v$ in the network with $\| \widetilde \g_u \| \geq \|\widetilde \g_v \|>0, \lim_{t \to \infty} \| \w_u(t) \| = \infty$ and $\lim_{t \to \infty} \| \w_v(t) \| = \infty$. Let $\frac{\| \widetilde \g_u \|}{\| \widetilde \g_v \|}$ be denoted by $c$. Under assumptions (A1), (A2) for gradient flow, for SWN, $\lim_{t\to\infty} \frac{\|\w_u(t)\|}{\|\w_v(t)\|}=c $
\end{theorem*}
    
\begin{proof}
From Theorem \ref{thm:weightconv}, we can say, for both $u$ and $v$, weights and gradients converge in opposite directions.

Consider a time $t_1$, such that for any $t > t_1$,
\begin{itemize}
    \item $- \nabla_{\w_u} \L(\w(t))$ and $\w_u(t)$ atmost make an angle $\epsilon$ with each other
    \item $- \nabla_{\w_v} \L(\w(t))$ and $\w_v(t)$ atmost make an angle $\epsilon$ with each other
\end{itemize} 
Then, using Equation (\ref{eq:gamma:SWN:gf}), we can say for any time $t>t_1$,
\[ \| \w_u(t) \| \geq \| \w_u(t_2) \| + \text{cos}(\epsilon)\int_{k=t_2}^t \eta(k) \| \nabla_{\w_u} \L(\w(k)) \| dk \]
\[ \| \w_u(t) \| \leq \| \w_u(t_2) \| + \int_{k=t_2}^t \eta(k) \| \nabla_{\w_u} \L(\w(k)) \| dk \]
\[ \| \w_v(t) \| \geq \| \w_v(t_2) \| + \text{cos}(\epsilon)\int_{k=t_2}^t \eta(k) \| \nabla_{\w_v} \L(\w(k)) \| dk\]
\[ \| \w_v(t) \| \leq \| \w_v(t_2) \| + \int_{k=t_2}^t \eta(k) \| \nabla_{\w_v} \L(\w(k)) \| dk \]
Using the above equations, we can say, for time $t > t_2$,
\[ \frac{\| \w_u(t) \|}{\| \w_v(t) \|} \geq \frac{\| \w_u(t_2) \| + \text{cos}(\epsilon)\int_{k=t_2}^t \eta(k) \| \nabla_{\w_u} \L(\w(k)) \| dk}{\| \w_v(t_2) \| + \int_{k=t_2}^t \eta(k) \| \nabla_{\w_v} \L(\w(k)) \| dk} \]
\[ \frac{\| \w_u(t) \|}{\| \w_v(t) \|} \leq \frac{\| \w_u(t_2) \| + \int_{k=t_2}^t \eta(k) \| \nabla_{\w_u} \L(\w(k)) \| dk}{\| \w_v(t_2) \| + \text{cos}(\epsilon)\int_{k=t_2}^t \eta(k) \| \nabla_{\w_v} \L(\w(k)) \| dk} \]
We know that both integrals diverge as $\| \w_u(t) \|$ and $\| \w_v(t) \| \to \infty$, $\lim_{t \to \infty} \frac{\| \w_u(t) \|}{\| \w_v(t) \|}$ and $\lim_{t \to \infty} \frac{\| \nabla_{\w_u} \L(\w(t)) \|}{\| \nabla_{\w_v} \L(\w(t)) \|}$ exist. Taking limit $t \to \infty$ on both the equations and using the Integral form of Stolz-Cesaro theorem, we get
\[ \liminf_{t \to \infty} \frac{\| \w_u(t) \|}{\| \w_v(t) \|} \geq \frac{\text{cos}(\epsilon) \| \widetilde \g_u \|}{\| \widetilde \g_v \|} \]
\[ \limsup_{t \to \infty} \frac{\| \w_u(t) \|}{\| \w_v(t) \|} \leq \frac{\| \widetilde \g_u \|}{\text{cos}(\epsilon)\| \widetilde \g_v \|} \]
As this holds for any $\epsilon > 0$, therefore
\[ \lim_{t \to \infty} \frac{\| \w_u(t) \|}{\| \w_v(t) \|} = c \]
\end{proof}

\subsubsection{Gradient Descent} \label{proof:swn:gd:thm:weightnorm}
\begin{theorem*}
Consider two nodes $u$ and $v$ in the network with $\| \widetilde \g_u \| \geq \|\widetilde \g_v \|>0, \lim_{t \to \infty} \| \w_u(t) \| = \infty$ and $\lim_{t \to \infty} \| \w_v(t) \| = \infty$. Let $\frac{\| \widetilde \g_u \|}{\| \widetilde \g_v \|}$ be denoted by $c$. Under assumptions (B1) - (B3) for gradient descent, for SWN, $\lim_{t\to\infty} \frac{\|\w_u(t)\|}{\|\w_v(t)\|}=c $
\end{theorem*}

\begin{proof}
For ease of notation, we will denote the two nodes by $u$ and $s$. From Theorem \ref{thm:weightconv}, we can say, for both $u$ and $s$, weights and gradients converge in opposite directions.
 Now from Equation (\ref{eq:gamma:SWN:gd:2}), we can say
\[ \gamma_u(t) = \gamma_u(0) - \sum_{k=0}^{k=t-1} \eta(k) \frac{\bv_u(k)^\top\nabla_{\w_u} \L(\w(k))}{\| \bv_u(k) \|} \]
\[ \gamma_s(t) = \gamma_s(0) - \sum_{k=0}^{k=t-1} \eta(k) \frac{\bv_s(k)^\top\nabla_{\w_s} \L(\w(k))}{\| \bv_s(k) \|} \]

Now, $\gamma_s(t)$ either diverges to $\infty$ or $-\infty$. In both the cases, it is a strictly monotonic sequence for large enough $t$. Also $\lim_{t \to \infty} \frac{\gamma_u(t+1) - \gamma_u(t)}{\gamma_s(t+1) - \gamma_s(t)}$ exists. Therefore, using Stolz-Cesaro Theorem, we can say
\[ \lim_{t \to \infty} \frac{\| \w_u(t) \|}{\| \w_s(t) \|} = c \]
\end{proof}

\section{Proof of Proposition \ref{prop:3}} \label{proof:prop3}
\begin{proposition*}
    Consider two nodes $u$ and $v$ in the network such that $\| \widetilde \g_v \| \geq \| \widetilde \g_u \| > 0$ and $\| \w_u(t) \|, \| \w_v(t) \| \to \infty$. Let $\frac{\| \widetilde \g_u\|}{\| \widetilde \g_v \|}$ be denoted by $c$. Consider any $\epsilon, \delta$ such that $0< \epsilon < c$ and $0<\delta<\frac{\pi}{2}$.  Then, the following holds:
    \begin{enumerate}[label=(\roman*)]
        \item There exists a time $t_1$, such that for all $t > t_1$ both SWN and EWN trajectories have the following properties: 
        \begin{enumerate}
        \begin{minipage}{0.4\textwidth}
        \item $\frac{\| \nabla_{\w_u} \L(\w(t))\|}{\| \nabla_{\w_v} \L(\w(t))\|} \in [c - \epsilon, c + \epsilon]$
        \end{minipage}
        \begin{minipage}{0.5\textwidth}
        \item $\left( \frac{\w_u(t)}{\| \w_u(t) \|}\right)^\top \left(\frac{-\nabla_{\w_u} \L(\w(t))}{\| \nabla_{\w_u} \L(\w(t)) \|}\right) \geq \cos(\delta)$
        \end{minipage}
        \item $\left( \frac{\w_v(t)}{\| \w_v(t) \|}\right)^\top \left(\frac{-\nabla_{\w_v} \L(\w(t))}{\| \nabla_{\w_v} \L(\w(t)) \|}\right) \geq \cos(\delta)$.
        \end{enumerate}
        \item for SWN, $\lim_{t \to \infty} \frac{\| \w_u(t) \|}{\| \w_v(t) \|} = c$
        \item for EWN, if at some time $t_2 > t_1$,
        \begin{enumerate}
            \item $\frac{\| \w_u(t_2) \|}{\| \w_v(t_2) \|} > \frac{1}{(c-\epsilon)\cos(\delta)} \implies \lim_{t \to \infty} \frac{\| \w_u(t) \|}{\| \w_v(t) \|} = \infty$
            \item $\frac{\| \w_u(t_2) \|}{\| \w_v(t_2) \|} < \frac{\cos(\delta)}{c + \epsilon} \implies \lim_{t \to \infty} \frac{\| \w_u(t) \|}{\| \w_v(t) \|} = 0$
        \end{enumerate}
    \end{enumerate}
\end{proposition*}

\begin{proof}
(i) It follows from Theorem \ref{thm:weightconv} and the fact that $\lim_{t \to \infty} \frac{\| \nabla_{\w_u} \L(\w(t))\|}{\| \nabla_{\w_v} \L(\w(t))\|} = c$.

(ii) Proof provided in Appendix \ref{proof:swn:thm:weightnorm}

(iii) Proof provided in Appendix \ref{proof:ewn:thm:weightnorm}
\end{proof}

\section{Proof of Proposition \ref{prop:quant_sparsity}} \label{proof:prop:quant_sparsity}
\begin{proposition*}
%\label{prop:quant_sparsity}
    Consider a linear model over $\R^d$ given by $f(\x) = \w^\top \x$, where each $w_i$ is further reparameterized as $e^{\alpha_i}$. Consider a dataset consisting of a single data point $\z \succ 0$, that is labelled as +1. According to the initialization of $\balpha$, define a relation $R$ on $\{1,\ldots,d\}$, given by $i \sim j$ if $w_i(0)z_i = w_j(0)z_j$. Then, R is an equivalence relation on $\{1,\ldots,d\}$.
    Let these equivalent sets be denoted by $I_1, I_2, ..., I_k$. Define a total order on these sets given by $I_a > I_b$ if $\exists i \in I_a, j \in I_b$ such that $w_i(0)z_i > w_j(0)z_j$. Let the maximum set according to this order be denoted by $I^*$. Then, for gradient flow on exponential loss, the following holds
    \begin{enumerate}[label=(\roman*)]
        \item For any $i \in I^*$, $\lim_{t \to \infty} w_i(t) = \infty$
        \item For $i, j \in I^*$, $\frac{w_i(t)}{w_j(t)} = \frac{z_j}{z_i}$.
        \item For any $i \notin I^*$, $\lim_{t \to \infty} w_i(t) = \left(\frac{1}{w_i(0)} - \frac{z_i}{w_j(0)z_j}\right)^{-1}$, where $j$ is any element in $I^*$.
    \end{enumerate}
\end{proposition*}
\begin{proof}
In this case, the loss is given by
\[ \L(\w) = e^{-\w^\top \z} \]
As each $w_i$ is further exponentially reparameterized, therefore, using Theorem \ref{theorem:1}, 
\begin{equation}
\label{eq:lin:gf:upd}
\frac{dw_i}{dt} = w_i^2 z_i e^{-\w^\top \z}    
\end{equation}

Using this equation, we can say, for any pair of indices $i, j$,
\[ \frac{d(\frac{z_i}{w_j} - \frac{z_j}{w_i})}{dt} = 0 \]
Thus, for any pair of indices $i, j$
\begin{equation}
\label{eq:lin:invariant}
\frac{z_i}{w_j(t)} - \frac{z_j}{w_i(t)} = \frac{z_i}{w_j(0)} - \frac{z_j}{w_i(0)}    
\end{equation}

Thus, we can say for every equivalent set $I_a$, if $i, j \in I_a$, then as $\frac{z_j}{w_i(0)} - \frac{z_i}{w_j(0)} = 0$, therefore, at any time $t$,
\[ \frac{w_i(t)}{w_j(t)} = \frac{z_j}{z_i} \]
Now, clearly the gradient flow stops only when $\w^\top \z \to \infty$. This means atleast one of the $w_i$ must tend to $\infty$ (Notice from Equation (\ref{eq:lin:gf:upd}) that $w_i$ always goes up). For two equivalents sets $I_a$ and $I_b$, such that $I_a > I_b$, we can say, for $i \in I_a$ and $j \in  I_b$, the RHS of Equation (\ref{eq:lin:invariant}) is positive. Thus, it is not possible for $w_j$ to go to $\infty$, otherwise the quantity on the LHS will be negative when $t \to \infty$. Using this argument, we can say, only for $i \in I^*$, $w_i$ tend to $\infty$. Then, for $i \notin I^*$, considering $j \in I^*$ and using Equation (\ref{eq:lin:invariant}),
\[ \lim_{t \to \infty} w_i(t) = \left(\frac{1}{w_i(0)} - \frac{z_i}{w_j(0)z_j}\right)^{-1} \]
\end{proof}

\section{Proof of Theorem \ref{theorem:ewn:gf:rate}}
\begin{theorem*}
%\label{theorem:ewn:gf:rate}
For Exponential Weight Normalization, under assumption (A1), the following hold for $t > t_0$ in case of gradient flow
\begin{enumerate}[label=(\roman*)]
    \begin{minipage}{0.5\textwidth}
    \item $\| \w(t) \|$ grows with $t$ as $o((\log t)^{\frac{1}{L}})$
    \end{minipage}
    \begin{minipage}{0.5\textwidth}
    \item $\L(t)$ goes down with $t$ as $O\left(\frac{1}{t}\right)$
    \end{minipage}
\end{enumerate}
\end{theorem*}

\begin{proof}
Using Equation (\ref{eq:finiteconvrate:gf}), we can say
\[ \frac{1}{\L(t)} \geq \frac{1}{\L(t_0)} + \frac{L^2\epsilon^2}{k} (t - t_0) \]
Thus $\L(t)$ goes down with t as $O\left(\frac{1}{t}\right)$. Now, we know, $\L(\w) = \sum_i e^{-y_i\Phi(\w, \x_i)}$. Clearly, as $\Phi(\w, \x_i)$ is smooth, therefore $\L(\w)$ attains a minima over the compact set $\| w \| = 1$. Using homogeneity of $\Phi$ and the fact that $\L$ goes down at $O\left(\frac{1}{t}\right)$, we can say $\| \w(t) \|$ grows with t as $o((\log t)^{\frac{1}{L}})$.
\end{proof}

\section{Proof of Theorem \ref{theorem:3}}
\begin{theorem*}
%\label{theorem:3}
For Exponential Weight Normalization, under Assumptions (B1)-(B4), $\rho > 0$, $\eta(t) = O\left(\left(\log\frac{1}{\L}\right)^c\right)$ for $c < 1$ and $\lim_{t \to \infty} \frac{\| \br(t+1) - \br(t) \|}{g(t+1) - g(t)} = 0$, the following hold
\begin{enumerate}[label=(\roman*)]
    \item $\| \w(t) \|$ asymptotically grows with $t$ as $\Theta\left((\log d(t))^{\frac{1}{L}}\right)$
    \item $\L(\w(t))$ asymptotically goes down with $t$ as $\Theta\left(\frac{1}{d(t)(\log d(t))^2}\right)$.
\end{enumerate}
\end{theorem*}

First, we will establish rates for gradient flow to elucidate the proof technique and then go to the case of gradient descent.
\subsection{Gradient Flow}
Although the asymptotic convergence rates for smooth homogeneous neural nets have been established in \citet{LyuLi20}, the proof technique becomes easier to understand for smooth homogeneous nets, without weight normalization. In this case, we will use an assumption (A3) that $\lim_{t \to \infty} \frac{\w(t)}{\| \w(t) \|}$ exists.
\subsubsection{Unnormalized Network} \label{convrate:unnorm:gf}
\begin{theorem*}
For Unnorm, under Assumptions (A1)-(A3) for gradient flow, $\rho > 0$ and $\lim_{t \to \infty} \frac{\| \frac{d \br(t)}{dt} \|}{g'(t)} = 0$, the following hold
\begin{enumerate}
    \item $\| \w(t) \|$ asymptotically grows at $\Theta\left((\log t)^{\frac{1}{L}}\right)$
    \item $\L(\w(t))$ asymptotically goes down at the rate of $\Theta\left(\frac{1}{t(\log t)^{2 - \frac{2}{L}}}\right)$.
\end{enumerate}
\end{theorem*}

\begin{proof}
Consider $\w = g(t)\widetilde \w + \br(t)$, where $\lim_{t \to \infty} \frac{\| \br(t)\|}{g(t)} = 0$ and $\br(t)^\top\widetilde \w = 0$. Now, we make an additional assumption that
$\lim_{t \to \infty} \frac{\| \frac{d \br(t)}{d(t)} \|}{g'(t)} = 0$. This basically avoids any oscillations in $\br(t)$ for large $t$, where it can have a higher derivative, but the value may be bounded. Now, we know
\begin{equation}
\label{eq:unnorm:conv:upd}
    \frac{d \w(t)}{dt} = \sum_{i=1}^m e^{-y_i\Phi(\w(t), \x_i)} y_i\nabla_\w \Phi(\w(t), \x_i)   
\end{equation}

Now, we know $\| \frac{d \w(t)}{dt} \| \neq 0$ for any finite $t$, otherwise $\w$ won't change and $\L$ can't converge to 0. Thus, for all $t$, we can say
\[ \frac{\| \frac{d \w(t)}{dt} \|}{\| \sum_{i=1}^m e^{-y_i\Phi(\w(t), \x_i)} y_i\nabla_\w \Phi(\w, \x_i) \|} = 1 \]
Taking limit $t \to \infty$ on both the sides, we get
\[ \lim_{t \to \infty} \frac{\| \frac{d \w(t)}{dt} \|}{\| \sum_{i=1}^m e^{-y_i\Phi(\w(t), \x_i)} y_i\nabla_\w \Phi(\w, \x_i) \|} = 1 \]
Now, we know
\[ \| \frac{d \w(t)}{dt} \| = \| g'(t)\widetilde \w + \frac{d \br(t)}{dt} \| \]
\[ \| \sum_{i=1}^m e^{-y_i\Phi(\w(t), \x_i)} y_i\nabla_\w \Phi(\w, \x_i) \| = \| \sum_{i=1}^m e^{-y_ig(t)^L\Phi\left(\widetilde \w + \frac{\br(t)}{g(t)}, \x_i\right)} (y_ig(t)^{L-1}\nabla_\w \Phi\left(\widetilde \w + \frac{\br(t)}{g(t)}, \x_i\right) \| \]
Thus, we can say
\[ \| \sum_{i=1}^m e^{-y_i\Phi(\w(t), \x_i)} y_i\nabla_\w \Phi(\w, \x_i) \| = \L(\w(t))g(t)^{L-1} \| \sum_{i=1}^m \frac{e^{-y_i\Phi(\w(t), \x_i)}}{\L(\w(t))} y_i\nabla_\w \Phi\left(\widetilde \w + \frac{\br(t)}{g(t)}, \x_i\right) \| \]
Now, for large enough $t$, $\| \sum_{i=1}^m \frac{e^{-y_i\Phi(\w(t), \x_i)}}{\L(\w(t))} y_i\nabla_\w \Phi\left(\widetilde \w + \frac{\br(t)}{g(t)}, \x_i\right) \|$ is bounded as, using Euler's homogeneous theorem, we can say 
\[ \lim_{t \to \infty} \widetilde{\w}^\top \left(\sum_{i=1}^m \frac{e^{-y_i\Phi(\w(t), \x_i)}}{\L(\w(t))} y_i\nabla_\w \Phi\left(\widetilde \w + \frac{\br(t)}{g(t)}, \x_i\right)\right) = \lim_{t \to \infty} L\left(\sum_{i=1}^m \frac{e^{-y_i\Phi(\w(t), \x_i)}}{\L(\w(t))} \Phi(\widetilde{\w}, \x_i)\right) \]
Thus, its a convex combination of positive defined terms and hence bounded. Thus, we can say
\begin{equation}
\label{eq:unnorm:conv:gf:lossgradnorm}
k_1 \leq \lim_{t \to \infty} \frac{\| \frac{d \w(t)}{dt} \|}{\L(\w(t)) g(t)^{L-1}} \leq k_2    
\end{equation}
where $k_1$ and $k_2$ are some constants. Also, by the assumption
\[ \lim_{t \to \infty} \frac{\| \frac{d \w(t)}{dt} \|}{g'(t)} = 1 \]
Thus, we can say
\[ k_1 \leq \lim_{t \to \infty} \frac{g'(t)}{\L(\w(t)) g(t)^{L-1}} \leq k_2 \]
Now, as for large enough $t$, for all $i$ that satisfy $\Phi(\widetilde{\w}, \x_i) = \rho$ and any $\epsilon$ satisfying $0 < \epsilon < \rho$, we can say
\[ \rho - \epsilon \leq \Phi\left(\widetilde \w + \frac{\br(t)}{g(t)}, \x_i\right) \leq \rho + \epsilon \]
therefore, for large enough $t$, we can say
\[ c_1 e^{-g(t)^L(\rho + \epsilon)} \leq \L(\w(t)) \leq c_2 e^{-g(t)^L(\rho - \epsilon)} \]
where $c_1$ and $c_2$ are some constants. Using the above equations, we can say
\begin{align}
    \label{eq:swn:grad:lower:bound}
    \lim_{t \to \infty} \frac{g'(t)}{e^{-g(t)^L(\rho + \epsilon)} g(t)^{L-1}} &\geq k_1c_1 \\
    \label{eq:swn:grad:upper:bound}
    \lim_{t \to \infty} \frac{g'(t)}{e^{-g(t)^L(\rho - \epsilon)} g(t)^{L-1}} &\leq k_2c_2
\end{align}
Now, in Equation \ref{eq:swn:grad:lower:bound}, multiplying numerator and denominator by $L g(t)^{L-1}(\rho + \epsilon)$ and denoting  $g(t)^L(\rho + \epsilon)$ by $h_1(t)$, we get
\[ \lim_{t \to \infty} \frac{h_1'(t)}{e^{-h_1(t)}h_1(t)^{2 - \frac{2}{L}}} \geq \alpha \]
where $\alpha$ is some constant. This leads us to conclude
\[ \lim_{t \to \infty} \frac{h_1(t)}{\log(t)} \geq 1 \implies \lim_{t \to \infty} \frac{g(t)^L(\rho + \epsilon)}{\log(t)} \geq 1 \]
Similarly, in Equation \ref{eq:swn:grad:upper:bound}, multiplying numerator and denominator by $L g(t)^{L-1}(\rho - \epsilon)$ and denoting  $g(t)^L(\rho - \epsilon)$ by $h_2(t)$, we get
\[ \lim_{t \to \infty} \frac{h_2'(t)}{e^{-h_2(t)}h_2(t)^{2 - \frac{2}{L}}} \leq \beta \]
where $\beta$ is some constant. This leads us to conclude
\[ \lim_{t \to \infty} \frac{h_2(t)}{\log(t)} \leq 1 \implies \lim_{t \to \infty} \frac{g(t)^L(\rho - \epsilon)}{\log(t)} \leq 1 \]
As $\epsilon$ can be chosen to be arbitrarily small, we can conclude
\[ \lim_{t \to \infty} \frac{g(t)^L \rho}{\log(t)} = 1 \]
Substituting this in Equation \ref{eq:unnorm:conv:gf:lossgradnorm}, we get that loss asymptotically goes down at $\Theta\left(\frac{1}{t (\log t)^{2 - \frac{2}{L}}}\right)$.
\end{proof}

\subsubsection{Exponential Weight Normalization}
\begin{theorem*}
For EWN, under Assumptions (A1)-(A3) for gradient flow and $\lim_{t \to \infty} \frac{\| \frac{d \br(t)}{dt} \|}{g'(t)} = 0$, the following hold
\begin{enumerate}
    \item $\| \w(t) \|$ asymptotically grows at $\Theta\left(\log (t)^{\frac{1}{L}}\right)$
    \item $\L(\w(t))$ asymptotically goes down at the rate of $\Theta\left(\frac{1}{t(\log t)^2}\right)$.
\end{enumerate}
\end{theorem*}

\begin{proof}
Consider $\w = g(t)\widetilde \w + \br(t)$, where $\lim_{t \to \infty} \frac{\| \br(t)\|}{g(t)} = 0$ and $\br(t)^\top\widetilde \w = 0$. Now, we make an additional assumption that
$\lim_{t \to \infty} \frac{\| \frac{d \br(t)}{d(t)} \|}{g'(t)} = 0$.

In this case, 
\[ -\frac{d \L(\w(t))}{d \w} = \sum_{i=1}^m e^{-y_i\Phi(\w(t), \x_i)} y_i\nabla_\w \Phi(\w(t), \x_i) \]
However, in this case, for a node $u$,
\[ \frac{d \w_u(t)}{dt} = -\| \w_u(t) \|^2\frac{d \L(\w(t))}{d \w_u} \]
Consider a vector $\ba(t)$ of equal dimension as $\w$, and its components corresponding to a node $u$ is given by $\ba_u(t) = -\|\widetilde \w_u\|^2\frac{d \L(\w(t))}{d \w_u}$. Now as we know $\w$ converges in direction to $\widetilde \w$, therefore, using the update equation above, we can say
\[ \lim_{t \to \infty} \frac{\|\frac{d \w(t)}{dt}\|}{g(t)^2\|\ba(t)\|} = 1 \]
Using the update equation for $-\frac{d \L(\w(t))}{d \w}$, which is the same as Equation \ref{eq:unnorm:conv:upd}, and using the same arguments as for Unnorm in Appendix \ref{convrate:unnorm:gf},  we can say
\begin{equation}
    \label{eq:dLdwnorm:EWN:gf}
    k_1 \leq \lim_{t \to \infty} \frac{\| \frac{d \L(\w(t))}{dw} \|}{\L(\w(t))g(t)^{L-1}} \leq k_2    
\end{equation}
where $k_1$ and $k_2$ are some constants. Now, using the expression for $\ba(t)$, we can say
\[ k_3 \leq \lim_{t \to \infty} \frac{\| \ba(t) \|}{\L(\w(t))g(t)^{L-1}} \leq k_4 \]
where $k_3$ and $k_4$ are some constants. Now, from the assumption, we know
\[ \lim_{t \to \infty} \frac{ \| \frac{d\w(t)}{dt} \|}{g'(t)} = 1 \]
Using the equations above, we can say
\begin{equation}
\label{eq:convrate:ewn:gf:gradloss}
    k_3 \leq \frac{g'(t)}{\L(\w(t))g(t)^{L+1}} \leq k_4
\end{equation}
Using similar reasoning as for Unnorm in Appendix \ref{convrate:unnorm:gf}, we can say, for a large enough $t$,
\[ c_1 e^{-g(t)^L(\rho + \epsilon)} \leq \L(\w(t)) \leq c_2 e^{-g(t)^L(\rho - \epsilon)} \]
where $c_1$ and $c_2$ are some constants. Substituting this in equation above, we get
\begin{align}
    \label{eq:ewn:gf:grad:lower:bound}
    \lim_{t \to \infty} \frac{g'(t)}{e^{-g(t)^L(\rho + \epsilon)} g(t)^{L+1}} &\geq k_3c_1 \\
    \label{eq:ewn:gf:grad:upper:bound}
    \lim_{t \to \infty} \frac{g'(t)}{e^{-g(t)^L(\rho - \epsilon)} g(t)^{L+1}} &\leq k_4c_2
\end{align}
Now, using similar arguments as for Unnorm in Appendix \ref{convrate:unnorm:gf}, we can say
\[ \lim_{t \to \infty} \frac{g(t)^L \rho}{\log(t)} = 1 \]
Substituting this in Equation \ref{eq:convrate:ewn:gf:gradloss}, we get that loss asymptotically goes down at $\Theta\left(\frac{1}{t (\log t)^2}\right)$.
\end{proof}

\subsection{Gradient Descent}\label{conv:gd}
\begin{theorem*}
For Exponential Weight Normalization, under Assumptions (B1)-(B4), $\rho > 0$, $\eta(t) = O\left(\left(\log\frac{1}{\L}\right)^c\right)$ for $c < 1$ and $\lim_{t \to \infty} \frac{\| \br(t+1) - \br(t) \|}{g(t+1) - g(t)} = 0$, the following hold
\begin{enumerate}[label=(\roman*)]
    \item $\| \w(t) \|$ asymptotically grows with $t$ as $\Theta\left((\log d(t))^{\frac{1}{L}}\right)$
    \item $\L(\w(t))$ asymptotically goes down with $t$ as $\Theta\left(\frac{1}{d(t)(\log d(t))^2}\right)$.
\end{enumerate}
\end{theorem*}

\begin{proof}
Consider $\w = g(t)\widetilde \w + \br(t)$, where $\lim_{t \to \infty} \frac{\| \br(t)\|}{g(t)} = 0$ and $\br(t)^\top\widetilde \w = 0$. Now, we make additional assumptions that
$\lim_{t \to \infty} \frac{\| \br(t+1) - \br(t) \|}{g(t+1) - g(t)} = 0$.

Consider a node $u$ in the network that has $\| \widetilde \w_u\| > 0$. The update equations for $\bv_u(t)$ and $\alpha_u(t)$ are given by
\[ \alpha_u(t+1) = \alpha_u(t) - \eta(t)e^{\alpha_u(t)}\frac{\bv_u(t)^\top\nabla_{\w_u} \L(\w(t))}{\| \bv_u(t) \|} \]
\[ \bv_u(t+1) = \bv_u(t) - \eta(t)\frac{e^{\alpha_u(t)}}{\| \bv_u(t) \|}\left(I - \frac{\bv_u(t)\bv_u(t)^\top}{\| \bv_u(t) \|^2}\right)\nabla_{\w_u} \L(\w(t)) \]
Now, we will first estimate $\|e^{\alpha_u(t+1)}\frac{\bv_u(t+1)}{\|\bv_u(t+1)\|} - e^{\alpha_u(t)}\frac{\bv_u(t)}{\|\bv_u(t)\|} \|$. Let $\delta_u(t)$ denote $\eta(t)e^{\alpha_u(t)}\| \nabla_{\w_u} \L(\w(t)) \|$ and $\epsilon_u(t)$ denote the angle between $\bv_u(t)$ and $-\nabla_{\w_u} \L(\w(t))$. We know $\lim_{t \to \infty} \delta_u(t) = 0$ and $\lim_{t \to \infty} \epsilon_u(t) = 0$. Now, rewriting update equations in terms of these symbols, we get
\[ e^{\alpha_u(t+1)} = e^{\alpha_u(t)}e^{\delta_u(t) \cos(\epsilon_u(t))} \]
\[ \bv_u(t+1) = \bv_u(t) + \delta_u(t)\sin(\epsilon_u(t))\frac{-\nabla_{\w_u} \L(\w(t))_{\bv_u(t)_\perp}}{\| \nabla_{\w_u} \L(\w(t))_{\bv_u(t)_\perp} \|} \]
where $-\nabla_{\w_u} \L(\w(t))_{\bv_u(t)_\perp}$ denotes the component of $-\nabla_{\w_u} \L(\w(t))$ perpendicular to $\bv_u(t)$. Now, using these equations we can say
\begin{equation}
\label{eq:diffwnorm:EWN:gd:convrate}
    \begin{split}
    e^{\alpha_u(t+1)}\frac{\bv_u(t+1)}{\| \bv_u(t+1) \|} - e^{\alpha_u(t)}\frac{\bv_u(t)}{\| \bv_u(t) \|} &= e^{\alpha_u(t)}\left(e^{\delta_u(t)\cos(\epsilon_u(t))}\frac{\| \bv_u(t) \|}{\| \bv_u(t+1) \|} - 1\right)\frac{\bv_u(t)}{\| \bv_u(t) \|} \\
    &+\left(\frac{e^{\alpha_u(t+1)}\delta_u(t)\sin(\epsilon_u(t))}{\| \bv_u(t) \|\| \bv_u(t+1) \|}\right)\frac{-\nabla_{\w_u} \L(\w(t))_{\bv_u(t)_\perp}}{\| \nabla_{\w_u} \L(\w(t))_{\bv_u(t)_\perp} \|}
    \end{split}
\end{equation}

Now as $\lim_{t \to \infty} \frac{\| \bv_u(t) \|}{\| \bv_u(t+1) \|} = 1$, therefore we can say
\[ \lim_{t \to \infty} \frac{e^{\delta_u(t)cos(\epsilon_u(t))}\frac{\| \bv_u(t) \|}{\| \bv_u(t+1) \|} - 1}{\delta_u(t)cos(\epsilon_u(t))} = 1 \]
Now, as $\| \bv_u(t) \|$ keeps on increasing during the gradient descent trajectory, therefore we can say $\frac{1}{\| \bv_u(t) \|\| \bv_u(t+1) \|} \leq k$, where $k > 0$ is some constant. Now dividing both sides of Equation (\ref{eq:diffwnorm:EWN:gd:convrate}) by $e^{\alpha_u(t)}\delta_u(t)\cos(\epsilon_u(t))$ and analyzing the coefficient of the second term on RHS, we get
\[ \lim_{t \to \infty} \frac{e^{\delta_u(t)\cos(\epsilon_u(t))}\sin(\epsilon_u(t))}{\|\bv_u(t)\|\|\bv_u(t+1)\|\cos(\epsilon_u(t))} \leq 0 \]
Taking norm on both sides of Equation (\ref{eq:diffwnorm:EWN:gd:convrate}), using Pythagoras theorem and the limits established above, we can say
\[ \lim_{t \to \infty} \frac{\| e^{\alpha_u(t+1)}\frac{\bv_u(t+1)}{\| \bv_u(t+1) \|} - e^{\alpha_u(t)}\frac{\bv_u(t)}{\| \bv_u(t) \|} \| }{e^{\alpha_u(t)}\delta_u(t)} = 1 \]
Now, we also know 
\[ \lim_{t \to \infty} \frac{\| e^{\alpha_u(t+1)}\frac{\bv_u(t+1)}{\| \bv_u(t+1) \|} - e^{\alpha_u(t)}\frac{\bv_u(t)}{\| \bv_u(t) \|} \| }{g(t+1) - g(t)} = \| \widetilde \w_u \| \]

Now, using equations above and Equation (\ref{eq:dLdwnorm:EWN:gf}), we can say
\[  k_1 \leq \lim_{t \to \infty} \frac{g(t+1) - g(t)}{\eta(t)\L(\w(t))g(t)^{L+1}} \leq k_2 \] 
where $k_1$ and $k_2$ are some constants.
Using similar reasoning as for Unnorm in Appendix \ref{convrate:unnorm:gf}, we can say
\[ c_1 e^{-g(t)^L(\rho + \epsilon)} \leq \L(\w(t)) \leq c_2 e^{-g(t)^L(\rho - \epsilon)} \]
where $c_1$ and $c_2$ are some constants.
Substituting in the equation above, we get
\begin{align*}
    \lim_{t \to \infty} \frac{g(t+1) - g(t)}{\eta(t)e^{-g(t)^L(\rho + \epsilon)} g(t)^{L+1}} &\geq k_1c_1 \\
    \lim_{t \to \infty} \frac{g(t+1) - g(t)}{\eta(t)e^{-g(t)^L(\rho - \epsilon)} g(t)^{L+1}} &\leq k_2c_2
\end{align*}
These equations govern the rate of $g(t)$ for any $\eta(t)$ that satisfies assumption (A4). Now, to obtain a better closed form, we will need the new assumption on $\eta(t)$, i.e, $\eta(t) = O\left(\left(\log\frac{1}{\L}\right)^c\right)$, where $c < 1$.

Now, define a map $d:\mathbb{N} \to \mathbb{R}$, given by $d(t) = \sum_{\tau = 0}^{t-1} \eta(\tau)$ and a real analytic function $f(t)$ satisfying $f(d(t)) = g(t)$ for all $t \in \mathbb{N}$ and $\lim_{t \to \infty} \frac{f(d(t+1)) - f(d(t))}{\eta(t) f'(d(t))} = 1$. We will later verify that the $f(t)$ obtained indeed satisfies this for the given $\eta(t)$. Substituting in the equations above,
\begin{equation}
\label{eq:ewn:gd:convrate:lossgrad}
k_1 \leq \lim_{t \to \infty} \frac{f'(d(t))}{\L(\w(t))f(d(t))^{L+1}} \leq k_2     
\end{equation}
\begin{align*}
    \lim_{t \to \infty} \frac{f'(d(t))}{e^{-f(d(t))^L(\rho + \epsilon)} f(d(t))^{L+1}} &\geq k_1c_1 \\
    \lim_{t \to \infty} \frac{f'(d(t))}{e^{-f(d(t))^L(\rho - \epsilon)} f(d(t))^{L+1}} &\leq k_2c_2
\end{align*}
As $t \to \infty$, $d(t) \to \infty$, therefore
\begin{align*}
    \lim_{t \to \infty} \frac{f'(t)}{e^{-f(t)^L(\rho + \epsilon)} f(t)^{L+1}} &\geq k_1c_1 \\
    \lim_{t \to \infty} \frac{f'(t)}{e^{-f(t)^L(\rho - \epsilon)} f(t)^{L+1}} &\leq k_2c_2
\end{align*}
Now, using similar arguments as in Appendix \ref{convrate:unnorm:gf}, 
\[ \lim_{t \to \infty} \frac{f(t)^L \rho}{\log(t)} = 1 \]
Substituting in the Equation \ref{eq:ewn:gd:convrate:lossgrad}, we get that the loss goes down at $\Theta\left(\frac{1}{d(t) (\log d(t))^2}\right)$.

We also verify that $\lim_{t \to \infty} \frac{f(d(t+1)) - f(d(t))}{\eta(t) f'(d(t))} = 1$ for $\eta(t) = O(\left(\log\frac{1}{\L}\right)^c)$, where $c < 1$. This can be easily verified by using mean value theorem, and simply verifying $\lim_{t \to \infty} \frac{\eta(t) f''(d(t))}{f'(d(t))} = 0$. Obtaining the expressions for $f'(d(t))$ and $f''(d(t))$, we get
\[ \lim_{t \to \infty} \frac{\eta(t) f''(d(t))}{f'(d(t))} = \lim_{t \to \infty} \frac{\eta(t)\left(d(t)(\frac{1}{L} - 1) - \log d(t)\right)}{d(t)\log(d(t))} \]
As loss goes down at $\Theta\left(\frac{1}{d(t) (\log d(t))^2}\right)$, therefore if $\eta(t) = O(\left(\log\frac{1}{\L}\right)^c)$ for $c < 1$, the above limit tends to 0 as $d(t) \to \infty$.
\end{proof}

\section{Cross-Entropy Loss}
In this section, we will provide the corresponding assumptions and theorems, along with their proofs, for cross-entropy loss. 

\subsection{Notations}
Let $k$ denote the total number of classes. As $\Phi(\w, \x_i)$ is a multidimensional function for multi-class classification, let's denote the $j^{th}$ component of the output by $\Phi_j(\w, \x_i)$. Also, denote the asymptotic normalized margin for $j^{th}$ class corresponding to $i^{th}$ data point($j \neq y_i$) by $\rho_{i, j}$, i.e, $\rho_{i, j} = \Phi_{y_i}(\widetilde \w, \x_i) - \Phi_j(\widetilde \w, \x_i)$. Margin for a data point $i$ is defined as $\rho_i = \min_{j \neq y_i} \rho_{i, j}$. The margin for the entire network is defined as $\rho = \min_i \rho_i$.

\subsection{Assumptions}
The assumptions can be broadly divided into loss function/architecture based assumptions and trajectory based assumptions. The loss functions/architecture based assumptions are shared across both gradient flow and gradient descent.

\textbf{Loss function/Architecture based assumptions}
\begin{enumerate}[label=\arabic*]
    \item $\ell(y_i, \Phi(\w, \x_i)) = \log\left(1 + \sum_{j \neq y_i} e^{-(\Phi_{y_i}(\w, \x_i) - \Phi_j(\w, \x_i))}\right)$
    \item $\Phi(., \x)$ is a $\C^1$ function, for a fixed $\x$
    \item $\Phi(\lambda \w, \x) = \lambda^L \Phi(\w, \x)$, for some $\lambda > 0$ and $L > 0$
\end{enumerate}

\textbf{Gradient flow}. For gradient flow, we make the following trajectory based assumptions
\begin{enumerate}[label=(A\arabic*)]
\item There exists a time $t_0$ such that $\L(\w(t_0)) < \log 2$.
\item $\lim_{t \to \infty} \frac{ -\nabla_\w \L(\w(t))}{\| \nabla_\w \L(\w(t)) \|} := \widetilde{\g}$.
\end{enumerate}

\textbf{Gradient Descent}. For gradient descent, we require the learning rate $\eta(t)$ to not grow too fast, and a slightly stronger assumption on loss.
\begin{enumerate}[label=(B\arabic*)]
    \begin{minipage}{0.5\textwidth}
    \item $\lim_{t \to \infty} \L(\w(t)) = 0$
    \end{minipage}
    \begin{minipage}{0.5\textwidth}
    \item $\lim_{t \to \infty} \frac{ -\nabla_\w \L(\w(t))}{\| \nabla_\w \L(\w(t)) \|} := \widetilde{\g}$
    % \item $\lim_{t \to \infty}\frac{\bell(\w(t))}{\| \bell(\w(t)) \|} := \widetilde{\bell}$
    % \item Let $\rho = \min_i y_i\Phi(\widetilde \w, \x_i)$. Then $\rho > 0$. 
    \end{minipage}
    \item $\lim_{t \to \infty} \eta(t) \| \w_u(t) \| \nabla_{\w_u} \L(\w(t)) \| = 0$ for all $u$ in the network.
\end{enumerate}

The assumption (B3) is mild, as the norm of the gradient of cross-entropy loss goes down exponentially fast as compared to norm of the weights.
 
\subsection{Asymptotic relations between weights and gradients}
This section contains the main theorems that establish asymptotic relations between weights and gradients for SWN and EWN. First, we will state a common proposition for both SWN and EWN.

\begin{proposition}
\label{prop:ce:lzero:gf}
Under assumption (A1) for gradient flow, for both SWN and EWN, $\lim_{t \to \infty} \L(\w(t)) = 0$.
\end{proposition}

\begin{proof}
First of all, for cross-entropy loss
\begin{equation}
\label{eq:ce:loss:grad}
\frac{d\L(t)}{d\w} = -\sum_i \left(\frac{\sum_{j \neq y_i} e^{-(\Phi_{y_i}(\w, \x_i) - \Phi_j(\w, \x_i))} (\nabla_\w \Phi_{y_i}(\w, \x_i) - \nabla_\w \Phi_j(\w, \x_i))}{1 + \sum_{j \neq y_i} e^{-(\Phi_{y_i}(\w, \x_i) - \Phi_j(\w, \x_i))}}\right)
\end{equation}
Now, using Theorem \ref{theorem:1},
\[ \frac{d\L(t)}{dt} = \left(\frac{d\L(t)}{d\w}\right)^\top \frac{d\w(t)}{dt} = -\sum_u \| \w_u(t) \|^2 \left\| \frac{d\L(t)}{d\w_u} \right\|^2 \]
Let $k$ be the total number of neurons in the network. Then using the elementary inequality, $(\sum_{i=1}^n a_i)^2 \leq n\sum_{i=1}^n a_i^2$, we get
\[ \frac{d\L(t)}{dt} \leq -\frac{1}{k} \left(\sum_u \| \w_u(t) \| \left\| \frac{d\L(t)}{d\w_u} \right\| \right)^2 \]
Again using the fact that $\left|\w(t)^\top \frac{d\L(t)}{d\w}\right| \leq \sum_u \| \w_u(t) \| \left\| \frac{d\L(t)}{d\w_u} \right\|$, we get
\begin{equation}
\label{eq:ce:dL:dt:upp:bound}
\frac{d\L(t)}{dt} \leq -\frac{1}{k} \left(\w(t)^\top \frac{d\L(t)}{d\w}\right)^2    
\end{equation}
Taking the dot product with $\w$ on both sides of Equation (\ref{eq:ce:loss:grad}) and using $\w^\top \nabla_\w \Phi(\w, \x_i) = L \Phi(\w, \x_i)$ (Euler's homogeneity theorem), we get
\[ \w(t)^\top \frac{d\L(t)}{d\w} = -L\sum_i \left(\frac{\sum_{j \neq y_i} e^{-(\Phi_{y_i}(\w, \x_i) - \Phi_j(\w, \x_i))} (\Phi_{y_i}(\w, \x_i) - \Phi_j(\w, \x_i))}{1 + \sum_{j \neq y_i} e^{-(\Phi_{y_i}(\w, \x_i) - \Phi_j(\w, \x_i))}}\right) \]
Now, using the fact, that at time $t_0$, $\L(t_0) < \log 2$, which means $\min_i \min_{j \neq y_i} (\Phi_{y_i}(\w, \x_i) - \Phi_j(\w, \x_i)) = \epsilon > 0$. Also, as we know, for gradient flow, the loss cannot go up, therefore, for any time $t > t_0$, $\min_i \min_{j \neq y_i} (\Phi_{y_i}(\w, \x_i) - \Phi_j(\w, \x_i)) > \epsilon > 0$. Using this, we can say, for any $t > t_0$,
\[ \w(t)^\top \frac{d\L(t)}{d\w} \leq -L\epsilon \sum_i \left(\frac{\sum_{j \neq y_i} e^{-(\Phi_{y_i}(\w, \x_i) - \Phi_j(\w, \x_i))}}{1 + \sum_{j \neq y_i} e^{-(\Phi_{y_i}(\w, \x_i) - \Phi_j(\w, \x_i))}}\right) \]
Using the fact that $\ln(1+t) > \frac{t}{1+t}$ for $t \in (0,1)$, therefore,
\[ \w(t)^\top \frac{d\L(t)}{d\w} \leq -L\epsilon\L(t) \]
Substituting this in Equation (\ref{eq:ce:dL:dt:upp:bound}), we get
\[ \frac{d\L(t)}{dt} \leq - \frac{L^2\epsilon^2}{k} \L(t)^2 \]
Integrating this equation from $t_0$ to $t$, we get
\begin{equation}
\label{eq:ce:finiteconvrate:gf}
\frac{1}{\L(t)} \geq \frac{1}{\L(t_0)} + \frac{L^2\epsilon^2}{k} (t - t_0)
\end{equation}
Clearly as t tends to $\infty$, RHS tends to $\infty$ and thus $\L$ tends to 0.
\end{proof}

Now, we provide one of our main theorem that establishes gradient convergence implies weight convergence.

\begin{theorem}
\label{thm:ce:weightconv}
    Consider a node $u$ in the network with $\| \widetilde{\g}_u \| > 0$ and $\lim_{t \to \infty} \| \w_u(t) \| = \infty$. Under assumptions (A1), (A2) for gradient flow and (B1)-(B3) for gradient descent, for both SWN and EWN
    \begin{enumerate}[label=(\roman*)]
        \begin{minipage}{0.5\textwidth}
        \item $\lim_{t \to \infty} \frac{\w_u(t)}{\| \w_u(t) \|} := \widetilde{\w}_u$ exists.
        \end{minipage}
        \begin{minipage}{0.5\textwidth}
        \item $\widetilde{\w}_u = \lambda \widetilde \g_u$ for some $\lambda > 0$.
        \end{minipage}
    \end{enumerate}
\end{theorem}

\begin{proof}
Same as in Appendix \ref{proof:thm:weightconv}.
\end{proof}

Now, we provide the main theorem that distinguishes SWN and EWN.

\begin{theorem}
\label{thm:ce:weightnorm}
    Consider two nodes $u$ and $v$ in the network with $\| \widetilde \g_u \| \geq \|\widetilde \g_v \|>0, \lim_{t \to \infty} \| \w_u(t) \| = \infty$ and $\lim_{t \to \infty} \| \w_v(t) \| = \infty$. Let $\frac{\| \widetilde \g_u \|}{\| \widetilde \g_v \|}$ be denoted by $c$. Under assumptions (A1), (A2) for gradient flow and (B1)-(B3) for gradient descent,
    % Under assumptions (A1)-(A3) for gradient flow and (A1)-(A4) for gradient descent, for all nodes $u$ and $v$ in the network, the following hold
    \begin{enumerate}[label=(\roman*)]
        \item for SWN, $\lim_{t\to\infty} \frac{\|\w_u(t)\|}{\|\w_v(t)\|}=c $
        \item for EWN, $\lim_{t\to\infty} \frac{\|\w_u(t)\|}{\|\w_v(t)\|}$ is either $0, \infty$ or $\frac{1}{c}$
    \end{enumerate}
\end{theorem}

\begin{proof}
Same as in Appendix \ref{proof:thm:weightnorm}
\end{proof}
\subsection{Sparsity Inductive Bias for Exponential Weight Normalisation}

The inverse relation between $\| \w_u(t) \|$ and $\| \nabla_{\w_u} \L(\w(t)) \|$ in the EWN trajectory results in an interesting inductive bias that favours movement along sparse directions. 
\begin{proposition}
\label{ce:prop:3}
    Consider two nodes $u$ and $v$ in the network such that $\| \widetilde \g_v \| \geq \| \widetilde \g_u \| > 0$ and $\| \w_u(t) \|, \| \w_v(t) \| \to \infty$. Let $\frac{\| \widetilde \g_u\|}{\| \widetilde \g_v \|}$ be denoted by $c$. Consider any $\epsilon, \delta$ such that $0< \epsilon < c$ and $0<\delta<2\pi$.  Then, the following holds:
    \begin{enumerate}[label=(\roman*)]
        \item There exists a time $t_1$, such that for all $t > t_1$ both SWN and EWN trajectories have the following properties: 
        \begin{enumerate}
        \begin{minipage}{0.4\textwidth}
        \item $\frac{\| \nabla_{\w_u} \L(\w(t))\|}{\| \nabla_{\w_v} \L(\w(t))\|} \in [c - \epsilon, c + \epsilon]$
        \end{minipage}
        \begin{minipage}{0.5\textwidth}
        \item $\left( \frac{\w_u(t)}{\| \w_u(t) \|}\right)^\top \left(\frac{-\nabla_{\w_u} \L(\w(t))}{\| \nabla_{\w_u} \L(\w(t)) \|}\right) \geq \cos(\delta)$
        \end{minipage}
        \item $\left( \frac{\w_v(t)}{\| \w_v(t) \|}\right)^\top \left(\frac{-\nabla_{\w_v} \L(\w(t))}{\| \nabla_{\w_v} \L(\w(t)) \|}\right) \geq \cos(\delta)$.
        \end{enumerate}
        \item for SWN, $\lim_{t \to \infty} \frac{\| \w_u(t) \|}{\| \w_v(t) \|} = c$
        \item for EWN, if at some time $t_2 > t_1$,
        \begin{enumerate}
            \item $\frac{\| \w_u(t_2) \|}{\| \w_v(t_2) \|} > \frac{1}{(c-\epsilon)\cos(\delta)} \implies \lim_{t \to \infty} \frac{\| \w_u(t) \|}{\| \w_v(t) \|} = \infty$
            \item $\frac{\| \w_u(t_2) \|}{\| \w_v(t_2) \|} < \frac{\cos(\delta)}{c + \epsilon} \implies \lim_{t \to \infty} \frac{\| \w_u(t) \|}{\| \w_v(t) \|} = 0$
        \end{enumerate}
    \end{enumerate}
\end{proposition}
\begin{proof}
The proof follows from Appendix \ref{proof:prop3}.
\end{proof}

\subsection{Convergence rates}
In this section, we provide convergence rate of loss for EWN.

\paragraph{Gradient Flow:} We provide a finite-time convergence rate of loss for gradient flow in case of EWN.
\begin{theorem}
\label{theorem:ce:ewn:gf:rate}
For Exponential Weight Normalization, under assumption (A1), the following hold for $t > t_0$ in case of gradient flow
\begin{enumerate}[label=(\roman*)]
    \begin{minipage}{0.5\textwidth}
    \item $\| \w(t) \|$ grows with $t$ as $o((\log t)^{\frac{1}{L}})$
    \end{minipage}
    \begin{minipage}{0.5\textwidth}
    \item $\L(t)$ goes down with $t$ as $O\left(\frac{1}{t}\right)$
    \end{minipage}
\end{enumerate}
\end{theorem}

\begin{proof}
Follow from Equation (\ref{eq:ce:finiteconvrate:gf})
\end{proof}

\paragraph{Gradient Descent:} 
\begin{theorem}
\label{ce:theorem:3}
For Exponential Weight Normalization, under Assumptions (B1)-(B4), $\rho > 0$, $\eta(t) = O\left(\left(\log\frac{1}{\L}\right)^c\right)$ for $c < 1$ and $\lim_{t \to \infty} \frac{\| \br(t+1) - \br(t) \|}{g(t+1) - g(t)} = 0$, the following hold
\begin{enumerate}[label=(\roman*)]
    \item $\| \w(t) \|$ asymptotically grows with $t$ as $\Theta\left((\log d(t))^{\frac{1}{L}}\right)$
    \item $\L(\w(t))$ asymptotically goes down with $t$ as $\Theta\left(\frac{1}{d(t)(\log d(t))^2}\right)$.
\end{enumerate}
\end{theorem}

\begin{proof}
The proof follows Appendix \ref{conv:gd}, the only difference is in the gradient update. Let $\w$ be represented as $\w = g(t)\widetilde \w + \br(t)$, where $\lim_{t \to \infty} \frac{\|\br(t)\|}{g(t)} = 0$. Using Equation (\ref{eq:ce:loss:grad}), we can say
\begin{equation}
    \label{eq:ce:dLdwnorm:EWN:gf}
    k_1 \leq \lim_{t \to \infty} \frac{\| \frac{d \L(\w(t))}{dw} \|}{\L(\w(t))g(t)^{L-1}} \leq k_2    
\end{equation}
where $k_1$ and $k_2$ are some constants.
As the order remains the same as in the proof for exponential loss, the proof follows from Appendix \ref{conv:gd}.
\end{proof}

\section{Lemma Proofs} \label{pflemmas}

\begin{lemma*}
Consider sequence a satisfying the following properties
\begin{enumerate}
    \item $a_k > 0$
    \item $\sum_{k=0}^\infty a_k = \infty$
    \item $\lim_{k \to \infty} a_k = 0$
\end{enumerate}
Then $\sum_{k=0}^{\infty} \frac{a_k}{\sqrt{\sum_{j=0}^{k} a_j^2}} = \infty$
\end{lemma*}

\begin{proof}
If $\sum_{k=0}^\infty a_k^2$ is bounded, then the statement is obvious. Let's consider the case when $\sum_{k=0}^\infty a_k^2$ diverges. As $\lim_{k \to \infty} a_k = 0$, therefore there must be an index $k_1$, such that for $k \geq k_1$, $a_k \leq \epsilon$. Now, as $a_k \leq \epsilon$, therefore $a_k^2 \leq \epsilon a_k$. Now, as $\sum_{k=0}^\infty a_k^2$ diverges, therefore, there must be an index $k_2 > k_1$, such that for any $k > k_2$, $\sum_{j = k_1}^k a_j^2 \geq \sum_{j=0}^{k_1-1} a_j^2$. Now, for $k > k_2$, we can say
\begin{align*}
\sum_{j=k_1}^k \frac{a_j}{\sqrt{\sum_{l=0}^j a_l^2}} &\geq \frac{1}{\sqrt{2}}\sum_{j=k_1}^k \frac{a_j}{\sqrt{\sum_{l=k_1}^j a_l^2}} \\
&\geq \frac{1}{\sqrt{2}}\sum_{j=k_1}^k \frac{a_j}{\sqrt{\sum_{l=k_1}^j \epsilon a_l}}\\
&\geq \frac{1}{\sqrt{2\epsilon}}\sum_{j=k_1}^k \frac{a_j}{\sqrt{\sum_{l=k_1}^k a_l}}\\
&= \frac{1}{\sqrt{2\epsilon}} \sqrt{\sum_{j=k_1}^k a_j}
\end{align*}
As $\sum_{k=0}^\infty a_k$ diverges, therefore $\sum_{k=0}^\infty \frac{a_k}{\sqrt{\sum_{j=0}^k a_j^2}}$ diverges as well.
\end{proof}

\begin{lemma*}
Consider two sequences a and b satisfying the following properties
\begin{enumerate}
    \item $a_k > 0, \sum_{k=0}^\infty a_k = \infty$ and $\lim_{k \to \infty} a_k = 0$
    \item $b_0 > 0$, $b$ is increasing and $b_{k+1}^2 \leq b_k^2 + \left(\frac{a_k}{b_k}\right)^2$
\end{enumerate}
Then $\sum_{k=0}^\infty \frac{a_k}{b_k} = \infty$.
\end{lemma*}

\begin{proof}
As we know $b$ is increasing and $b_{k+1}^2 \leq b_k^2 + (\frac{a_k}{b_k})^2$, we get
\[ b_k \leq \sqrt{b_0^2 + \sum_{j=0}^{k-1} \left(\frac{a_j}{b_j}\right)^2} \leq \sqrt{b_0^2 + \frac{1}{b_0^2}\sum_{j=0}^{k-1} a_j^2} \]
Using this, we can say
\[ \sum_{j=0}^k \frac{a_j}{b_j} \geq \sum_{j=0}^k \frac{a_j}{\sqrt{b_0^2 + \frac{1}{b_0^2}\sum_{l=0}^{j-1} a_l^2}} \geq \sum_{j=0}^k \frac{a_j}{\sqrt{b_0^2 + \frac{1}{b_0^2}\sum_{l=0}^{k-1} a_l^2}}\]
Now, if $\sum_{k=0}^\infty a_k^2$ does not diverge to infinity, then $b$ remains bounded using the bound above and then its trivial to establish that $\sum_{k=0}^\infty \frac{a_k}{b_k}$ diverges. In case, $\sum_{k=0}^\infty a_k^2$ diverges to infinity, then there must be an index $k_1$ such that for any $k>k_1$, we can say $\sum_{j=0}^{k-1} a_j^2 \geq b_0^4$. So, for $k>k_1$, we can say
\[ \sum_{j=0}^k \frac{a_j}{b_j} \geq \sum_{j=0}^k \frac{b_0}{\sqrt{2}}\frac{a_j}{\sqrt{\sum_{l=0}^{k-1} a_l^2}}\]
Now, as we have assumed $a$ tends to zero, so there must be an index $k_2$ such that for any $k>k_2$, $a_k \leq \epsilon$. Also, as we have assumed $\sum_{j=0}^\infty a_j^2$ diverges, therefore there must be an index $k_3>k_2$, such that for $k>k_3$, $\sum_{j=k_2}^k a_j^2 \geq \sum_{j=0}^{k_2} a_j^2$. Using these things and that if $a_j \leq \epsilon$, then $a_j^2 \leq \epsilon a_j$, we can say for $k>k_3$,
\[ \sum_{j=k_3}^k \frac{a_j}{b_j} \geq \sum_{j=k_3}^k \frac{b_0}{2}\frac{a_j}{\sqrt{\sum_{l=k_3}^{k-1} \epsilon a_l}} \geq \frac{b_0}{2\sqrt{\epsilon}} \sqrt{\sum_{j=k_3}^{k-1} a_j} \]
Now, as $\sum_{k=0}^\infty a_k$ diverges, thus $\sum_{k=0}^\infty \frac{a_k}{b_k}$ diverges as well.
\end{proof}

% \begin{lemma*}
% Consider two sequences a and b satisfying the following properties
% \begin{enumerate}
%     \item $a_k > 0$ and $\sum_{k=0}^\infty a_k = \infty$
%     \item $b_k > 0$ and $\sum_{k=0}^\infty b_k = \infty$
%     \item $\sum_{k=0}^\infty (a_k - b_k)$ converges to a finite value
%     \item $lim_{k \to \infty} \frac{a_k}{b_k}$ exists
% \end{enumerate}
% Then $\lim_{k \to \infty} \frac{a_k}{b_k} = 1$.
% \end{lemma*}

% \begin{proof}
% Let's say $\lim_{k \to \infty} \frac{a_k}{b_k} = c > 1$. The other case can be handled similarly. Choose an $\epsilon > 0$ such that $c - \epsilon > 1$. Then, there exists an index $k_1$, such that for $k > k_1$, we can say
% \[ c - \epsilon \leq \frac{a_k}{b_k} \leq c + \epsilon \]
% Using this, we can say, for $k > k_1$,
% \[ b_k(c - \epsilon - 1) \leq a_k - b_k \leq b_k(c + \epsilon - 1) \]
% Summing the equation above from $k_1$ to $\infty$ and recognizing that $\sum_{k=0}^\infty b$ diverges, we get $\sum_{k=k_1}^\infty (a_k-b_k) = \infty$. This contradicts. Therefore $\lim_{k \to \infty} \frac{a_k}{b_k} = 1$.
% \end{proof}

\section{Integral Form of Stolz-Cesaro Theorem} \label{intstolz}
We first state the Stolz-Cesaro Theorem.

\begin{theorem*} \citep{muresan2015concrete}
Assume that $\{a\}_{k=1}^\infty$ and $\{b\}_{k=1}^\infty$ are two sequences of real numbers such that $\{b\}_{k=1}^\infty$ is strictly monotonic and diverging. Additionally, if $\lim_{k \to \infty} \frac{a_{k+1} - a_k}{b_{k+1} - b_k} = L$ exists, then $\lim_{k \to \infty} \frac{a_k}{b_k}$ exists and is equal to $L$. 
\end{theorem*}
Now, we state and prove the Integral Form of Stolz-Cesaro Theorem.
\begin{theorem*}
Consider two functions $f(t)$ and $g(t)$ greater than zero satisfying $\int_{a}^b f(t) dt < \infty$ and $\int_{a}^b g(t) dt < \infty$ for every finite $a,b$. For any time t, its known that $\int_t^\infty f(t) dt = \infty$ and $\int_t^\infty g(t) dt = \infty$. If $\lim_{t \to \infty} \frac{f(t)}{g(t)}$ exist and is equal to $L$, then $\lim_{t \to \infty} \frac{\int_c^t f(t)dt}{\int_c^t g(t)dt}$ exists for any $c$ and is equal to $L$.
\end{theorem*}

\begin{proof}
\textbf{Case 1:} L = 0 or $\infty$:

We will prove for L = $\infty$. The case for 0 can be handled similarly. For any $M>0$, there must exist a time $t_1>c$, such that $\frac{f(t)}{g(t)} > M$, for $t>t_1$. Thus we can say for $t>t_1$,
\[ \int_c^t f(t) dt > \int_c^{t_1} f(t)dt + M\int_{t_1}^t g(t)dt \]
Adding $M\int_c^{t_1} g(t) dt$ on both the sides, we get
\[ \int_c^t f(t) dt + M\int_c^{t_1} g(t) dt > \int_c^{t_1} f(t)dt + M\int_{c}^t g(t)dt \]
Dividing both sides by $\int_c^{t} g(t)dt$ and taking limsup $t \to \infty$(using also the fact that $\int_{a}^b f(t) dt < \infty$ and $\int_{a}^b g(t) dt < \infty$ for every finite $a,b$), we get
\[ \limsup_{t \to \infty} \frac{\int_c^t f(t) dt}{\int_c^t g(t) dt} > M \]
Similarly the equation holds for liminf as well. Thus, both liminf and limsup are greater than $M$ for any $M$. Hence $\lim_{t \to \infty} \frac{\int_c^t f(t) dt}{\int_c^t g(t) dt} = \infty$.

\textbf{Case2:} L is finite

In this case, there must exist some time $t_1>c$, such that $L - \epsilon < \frac{f(t)}{g(t)} < L+\epsilon$.Thus, we can say for $t>t_1$,
\[ \int_c^{t_1} f(t)dt + (L - \epsilon)\int_{t_1}^{t} g(t)dt \leq \int_c^t f(t)dt \leq \int_c^{t_1} f(t)dt + (L + \epsilon)\int_{t_1}^{t} g(t)dt \]
Taking the left inequality, adding $(L - \epsilon)\int_{c}^{t_1} g(t)dt$ on both the sides, dividing both the sides by $\int_{c}^{t} g(t)dt$ and taking $\liminf_{t \to \infty}$, we get
\[ L - \epsilon \leq \liminf_{t \to \infty} \frac{\int_c^t f(t)dt}{\int_c^t g(t)dt} \]
Similarly, taking the right inequality, adding $(L + \epsilon)\int_{c}^{t_1} g(t)dt$ on both the sides, dividing both the sides by $\int_{c}^{t} g(t)dt$ and taking $\limsup_{t \to \infty}$, we get
\[ \limsup_{t \to \infty} \frac{\int_c^t f(t)dt}{\int_c^t g(t)dt} \leq L + \epsilon \]
Using the two inequalities, we get, for any $\epsilon > 0$,
\[ \limsup_{t \to \infty} \frac{\int_c^t f(t)dt}{\int_c^t g(t)dt} - \liminf_{t \to \infty} \frac{\int_c^t f(t)dt}{\int_c^t g(t)dt} \leq 2\epsilon \]
Thus, $\lim_{t \to \infty} \frac{\int_c^t f(t)dt}{\int_c^t g(t)dt}$ exists and is equal to $L$.
\end{proof}

\section{Standard Weight Normalization is not Locally Lipschitz in its parameters} \label{swn:ll}
In this section, we will denote $\w$ by $\btheta$ so as to be consistent with the notaion in \citet{LyuLi20}. SWN(in its parameters $\gamma$ and $\bv$) is also a homogeneous network. Therefore, results from \citet{LyuLi20} should directly apply to the case of SWN as well. However, a crucial point to be noted is that it is not even locally Lipschitz around $\| \bv_u \| = 0$. Therefore, the assumptions from \citet{LyuLi20} do not hold.

However, during gradient descent or gradient flow, if started from a finite $\| \bv_u \| > 0$, for all $u$, then during the entire trajectory, $\| \bv_u \|$ cannot go down. Therefore, the network is still locally Lipschitz along the trajectory it takes. Examining the proofs from \citet{LyuLi20}, its clear that the proof regarding monotonicity of margin and convergence rates are just dependent on the path that gradient descent/flow takes and thus the proofs hold.

However, the result regarding the limit points of $\frac{\btheta}{\| \btheta \|}$ do not hold. One of the crucial theorems the proof relies on is stated below

\begin{theorem*}
    Let $\{x_k \in \mathcal{R}^d: k \in \mathbb{N}\}$ be a sequence of feasible points of an optimization problem (P), $\{ \epsilon_k > 0: k \in \mathbb{N} \}$ and $\{ \delta_k > 0 : k \in \mathbb{N} \}$ be two sequences. $x_k$ is an $(\epsilon_k, \delta_k)$-KKT point for every k and $\epsilon_k \to 0, \delta_k \to 0$. If $x_k \to x$ as $k \to \infty$ and MFCQ holds at $x$, then $x$ is a KKT point of (P) 
\end{theorem*}

The above statement requires MFCQ to be satisfied at $x$, that was shown in \citet{LyuLi20} assuming local lipschitzness/smoothness at $x$. However, in this case, for gradient flow, as $\| \bv_u \|$ does not grow, while $| \gamma_u | \to \infty$, therefore the convergent point of $\frac{\btheta}{\| \btheta \|}$ will always have the component corresponding to $\bv_u$ as 0. Thus, the network is not locally lipschitz at $x$ and the proof that MFCQ holds is violated. Similarly, for gradient descent as well, it can't be said that $\bv_u$ has a non-zero component in $\frac{\btheta}{\| \btheta \|}$. Thus, the proof does not hold.

\section{Experiment Details}
In all the experiments, techniques for handling numerical underflow were used as described in \citet{LyuLi20}. However, the learning rate they used was of $O\left(\frac{1}{\L}\right)$, but in our case, we generally modify it to be $O\left(\frac{1}{\L^c}\right)$, where $c < 1$.

\subsection{\texttt{Lin-Sep}}
The learning rate used was $\frac{k(t)}{\L^{0.97}}$, so that it speeds up at the beginning of training, but slows down as loss approaches $e^{-300}$. The constant $k(t)$ was initialized at $0.01$, and was increased by a factor of $1.1$ every time loss went down and decreased  by a factor of $1.1$ every time loss went up after a gradient step. Its value was capped at $0.01$ for EWN and SWN.

\begin{figure}[t]
\floatconts{SWN:demons_lin_sep}{\caption{\textbf{Demonstration of Results for SWN in \texttt{Lin-Sep} experiment: } (a) Evolution of $\| \w_u \|$ (b) Cosine between weights and gradients for weights 5, 6 and 8. (c) Weight and gradient norms for weights 5, 6 and 8.}}  
% The three graphs are plotted at loss values of $e^{-200}, e^{-250}$ and $e^{-300}$ respectively. At each loss value, for the 3 weights, $ \log \| \nabla_{\w_u} \L\| - \log \| \w_u \|$ is approximately same.}}
{
% \subfigure[Dataset]{
%   \includegraphics[width=0.4\textwidth]{plots/lin_sep_dataset/SWN/Assumptions/dataset.pdf}
% }
\subfigure{
  \includegraphics[width=0.3\textwidth]{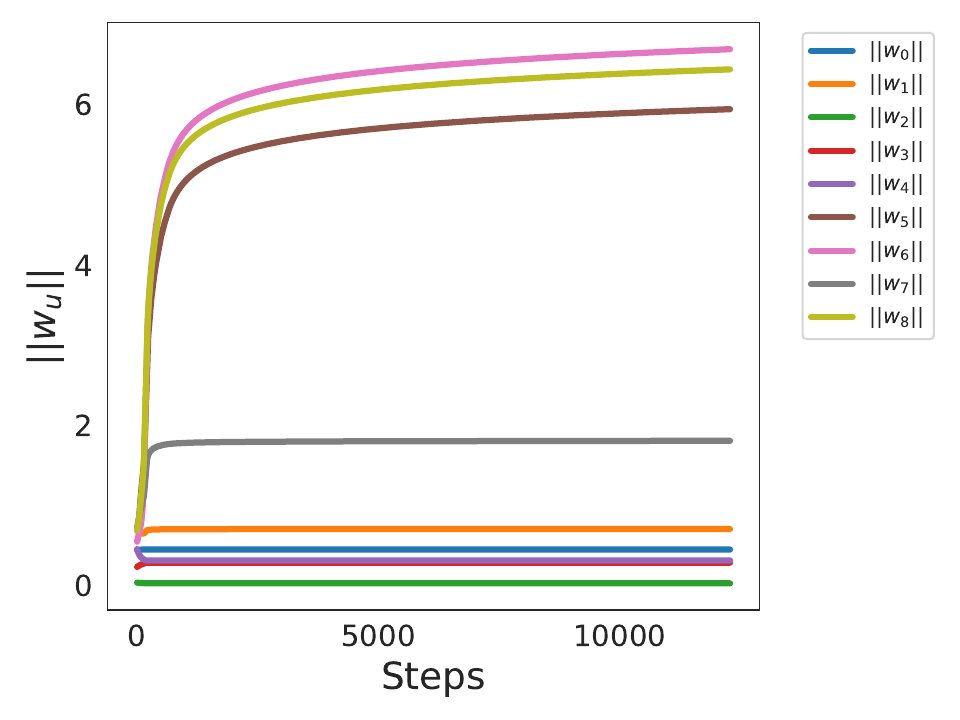}
}
% \subfigure{
%   \includegraphics[width=0.4\textwidth]{plots/lin_sep_dataset/SWN/Theorems/l_tilde_grad.pdf}
% }
\subfigure{
  \includegraphics[width=0.3\textwidth]{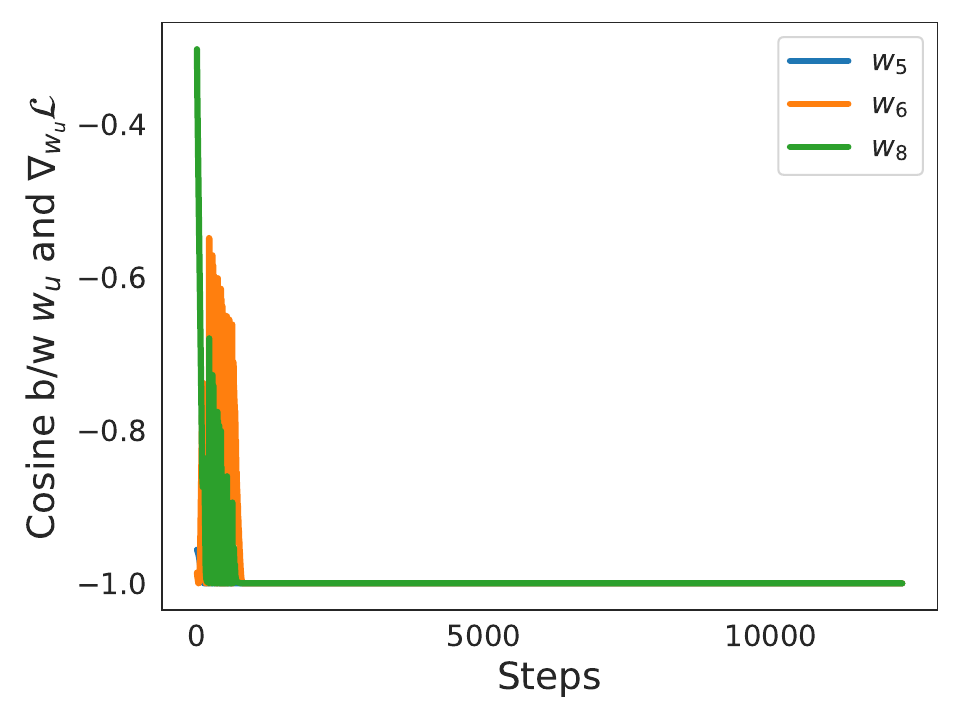}
}
% \subfigure{
%   \includegraphics[width=0.3\textwidth]{plots/lin_sep_dataset/SWN/Theorems/weight_grad_ratios_200.pdf}
% }
% \subfigure{
%   \includegraphics[width=0.3\textwidth]{plots/lin_sep_dataset/SWN/Theorems/weight_grad_ratios_250.pdf}
% }
\subfigure{
  \includegraphics[width=0.3\textwidth]{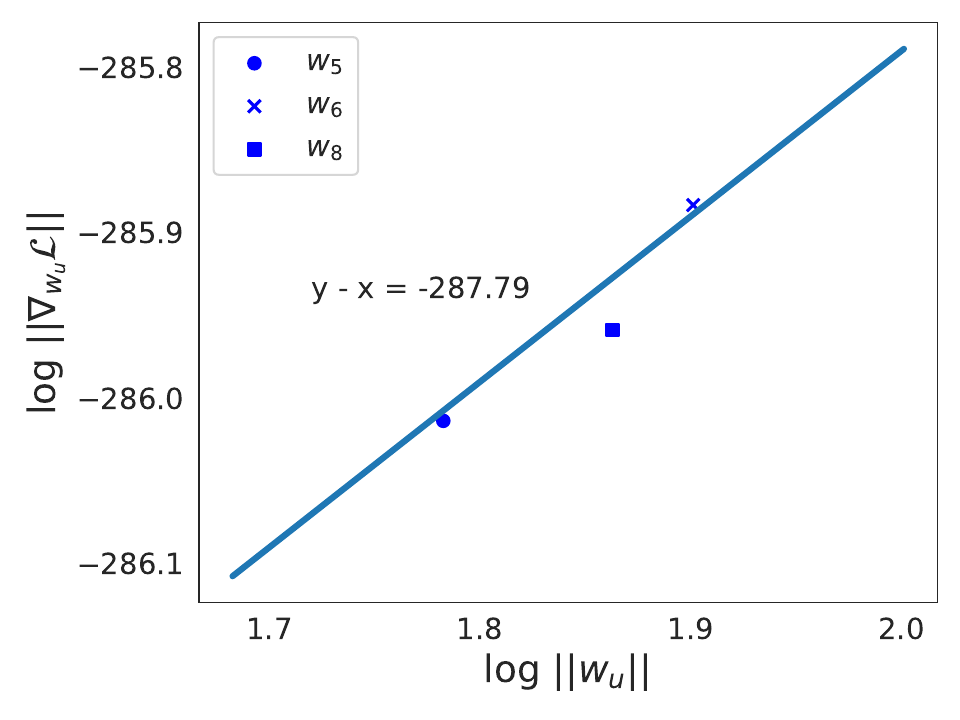}
}
}
\end{figure}

\begin{figure}[t]
\floatconts{EWN:demons_XOR}{\caption{\textbf{Demonstration of Results for EWN in \texttt{XOR} experiment with ReLU-square activation: } (a) Evolution of $\| \w_u \|$ (b) Cosine between weights and gradients for weights 0, 1, 13 and 17. (c) Weight and gradient norms for weights 0, 1, 13 and 17.}}  
% The three graphs are plotted at loss values of $e^{-200}, e^{-250}$ and $e^{-300}$ respectively. At each loss value, for the 3 weights, $ \log \| \nabla_{\w_u} \L\| - \log \| \w_u \|$ is approximately same.}}
{
% \subfigure[Dataset]{
%   \includegraphics[width=0.4\textwidth]{plots/lin_sep_dataset/SWN/Assumptions/dataset.pdf}
% }
\subfigure{
  \includegraphics[width=0.3\textwidth]{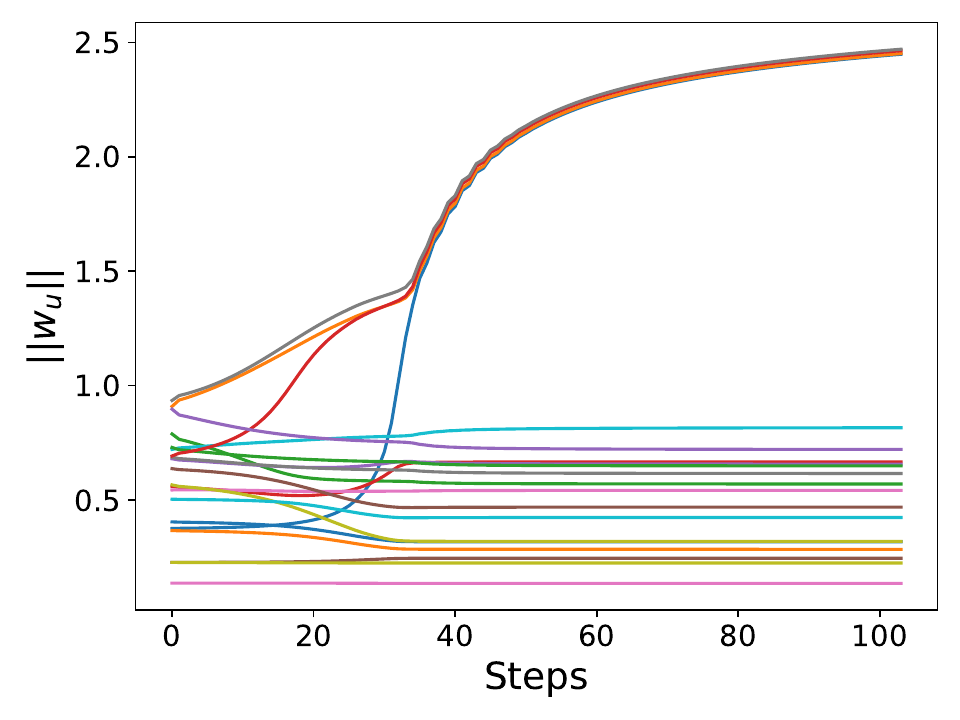}
}
% \subfigure{
%   \includegraphics[width=0.4\textwidth]{plots/lin_sep_dataset/SWN/Theorems/l_tilde_grad.pdf}
% }
\subfigure{
  \includegraphics[width=0.3\textwidth]{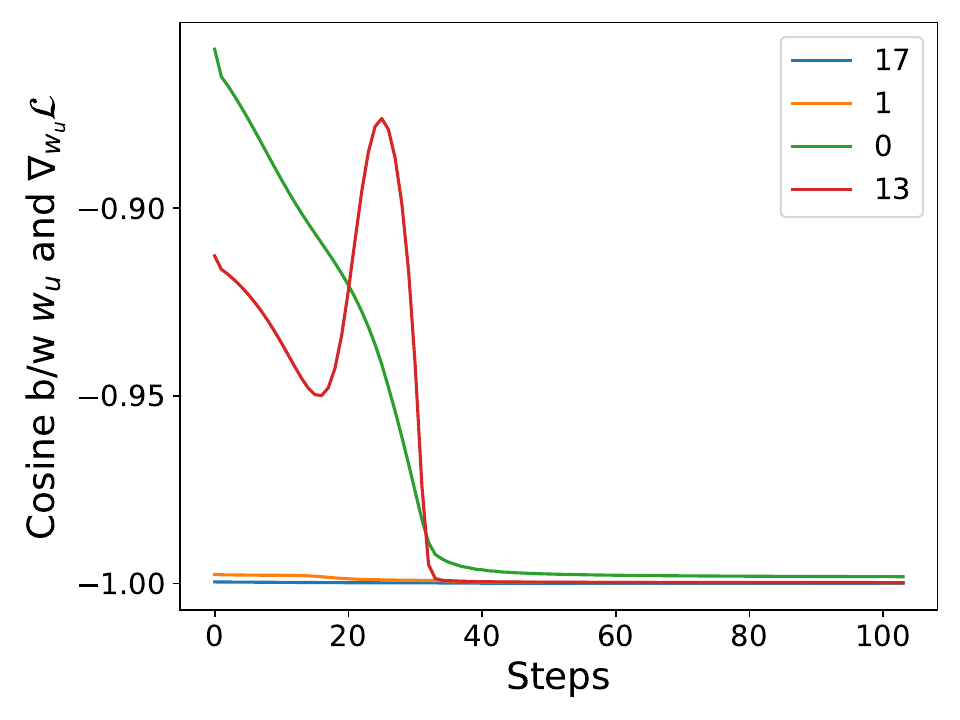}
}
% \subfigure{
%   \includegraphics[width=0.3\textwidth]{plots/lin_sep_dataset/SWN/Theorems/weight_grad_ratios_200.pdf}
% }
% \subfigure{
%   \includegraphics[width=0.3\textwidth]{plots/lin_sep_dataset/SWN/Theorems/weight_grad_ratios_250.pdf}
% }
\subfigure{
  \includegraphics[width=0.3\textwidth]{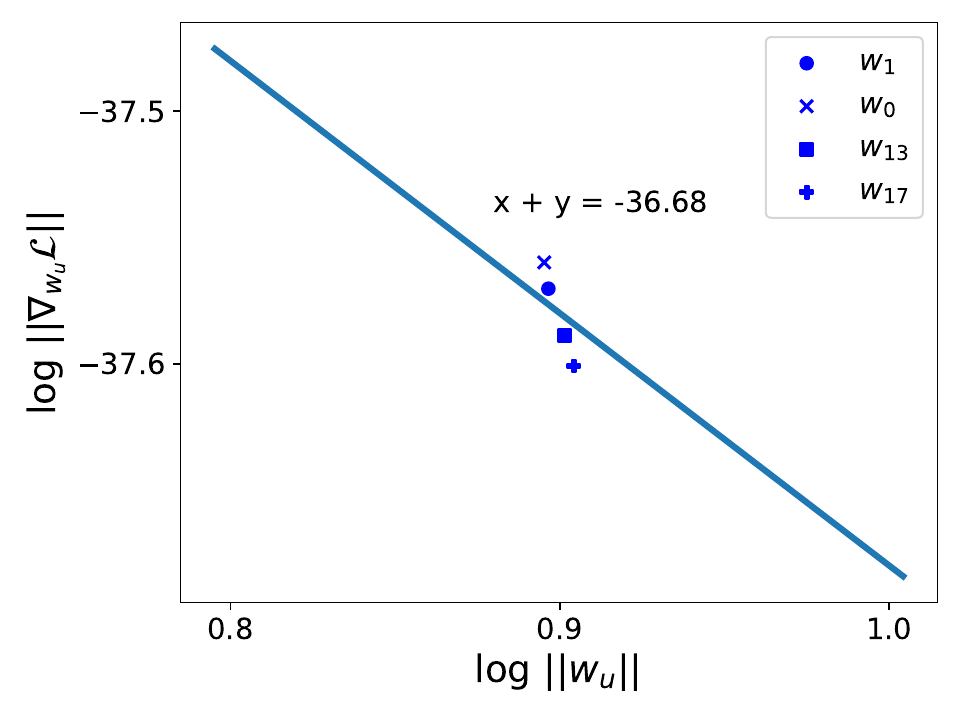}
}
}
\end{figure}

\begin{figure}[t]
\floatconts{SWN:demons_XOR}{\caption{\textbf{Demonstration of Results for SWN in \texttt{XOR} experiment with ReLU-square activation: } (a) Evolution of $\| \w_u \|$ (b) Cosine between weights and gradients for weights 0, 1, 13 and 17. (c) Weight and gradient norms for weights 0, 1, 13 and 17.}}  
% The three graphs are plotted at loss values of $e^{-200}, e^{-250}$ and $e^{-300}$ respectively. At each loss value, for the 3 weights, $ \log \| \nabla_{\w_u} \L\| - \log \| \w_u \|$ is approximately same.}}
{
% \subfigure[Dataset]{
%   \includegraphics[width=0.4\textwidth]{plots/lin_sep_dataset/SWN/Assumptions/dataset.pdf}
% }
\subfigure{
  \includegraphics[width=0.3\textwidth]{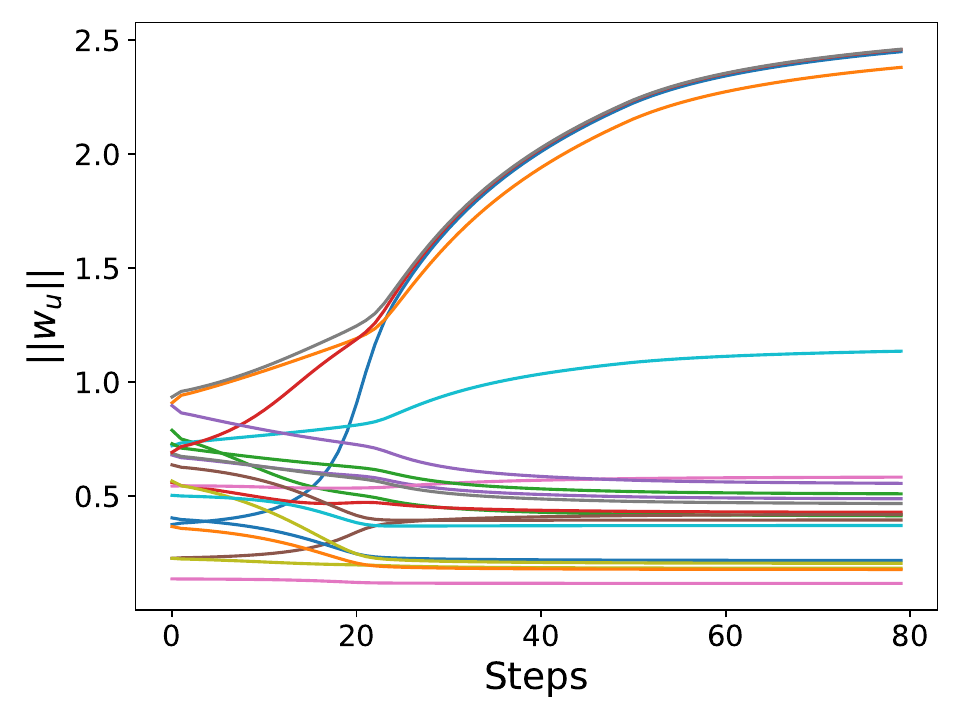}
}
% \subfigure{
%   \includegraphics[width=0.4\textwidth]{plots/lin_sep_dataset/SWN/Theorems/l_tilde_grad.pdf}
% }
\subfigure{
  \includegraphics[width=0.3\textwidth]{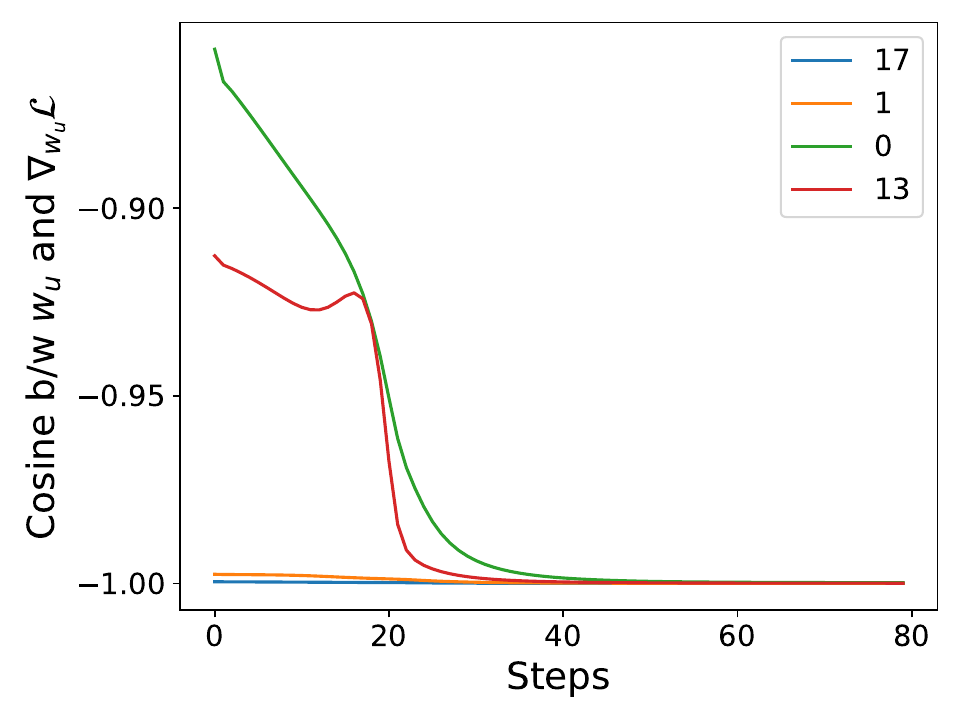}
}
% \subfigure{
%   \includegraphics[width=0.3\textwidth]{plots/lin_sep_dataset/SWN/Theorems/weight_grad_ratios_200.pdf}
% }
% \subfigure{
%   \includegraphics[width=0.3\textwidth]{plots/lin_sep_dataset/SWN/Theorems/weight_grad_ratios_250.pdf}
% }
\subfigure{
  \includegraphics[width=0.3\textwidth]{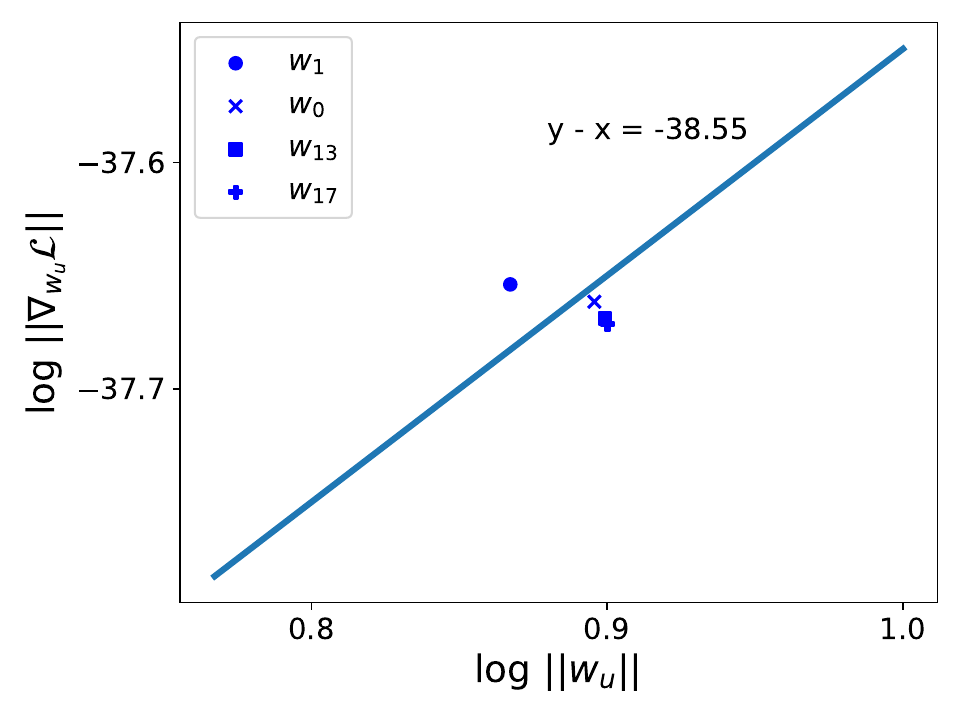}
}
}
\end{figure}

\begin{figure}[t]
\floatconts{EWN:demons_MNIST}{\caption{\textbf{Demonstration of Results for EWN on MNIST dataset with 2-class classification: } (a) Evolution of $\| \w_u \|$ (b) Cosine between weights and gradients for weights 96 and 105. (c) Weight and gradient norms for weights 96 and 105.}}  
% The three graphs are plotted at loss values of $e^{-200}, e^{-250}$ and $e^{-300}$ respectively. At each loss value, for the 3 weights, $ \log \| \nabla_{\w_u} \L\| - \log \| \w_u \|$ is approximately same.}}
{
% \subfigure[Dataset]{
%   \includegraphics[width=0.4\textwidth]{plots/lin_sep_dataset/SWN/Assumptions/dataset.pdf}
% }
\subfigure{
  \includegraphics[width=0.3\textwidth]{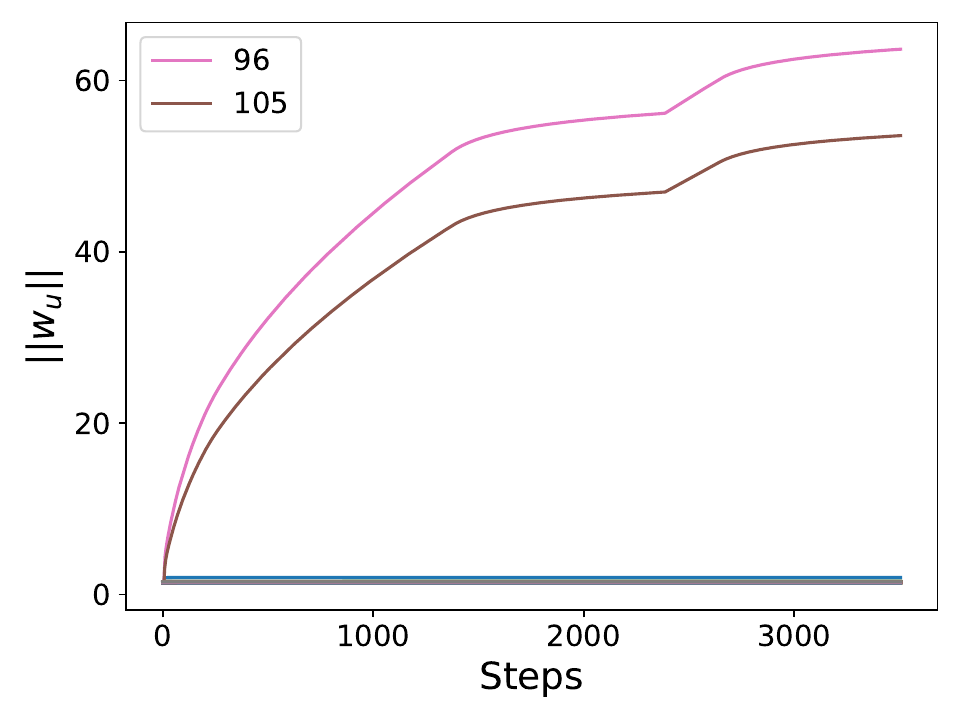}
}
% \subfigure{
%   \includegraphics[width=0.4\textwidth]{plots/lin_sep_dataset/SWN/Theorems/l_tilde_grad.pdf}
% }
\subfigure{
  \includegraphics[width=0.3\textwidth]{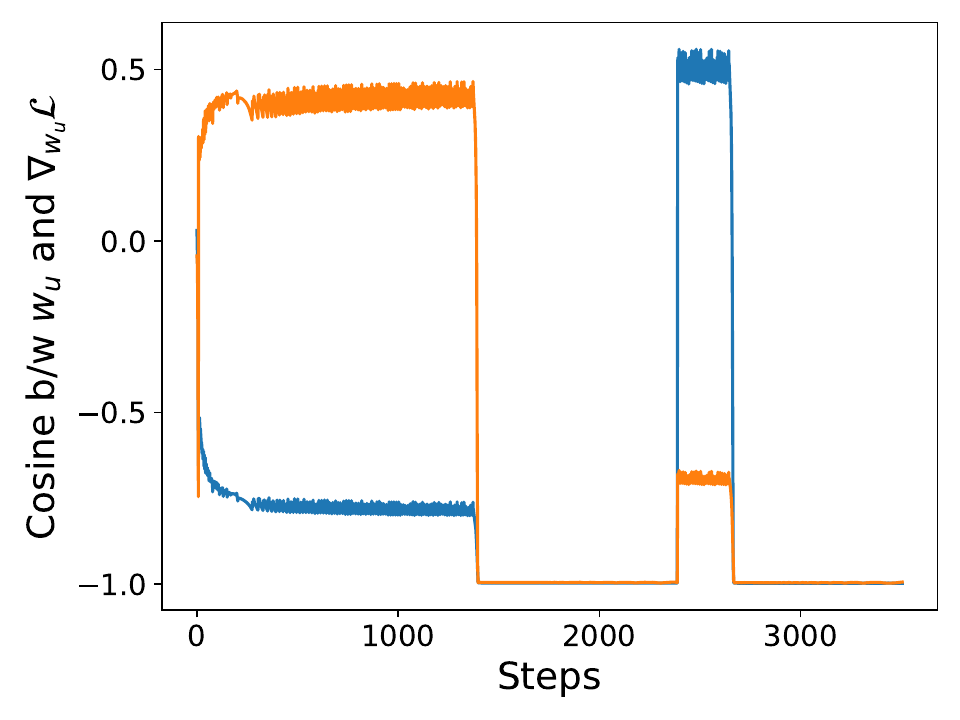}
}
% \subfigure{
%   \includegraphics[width=0.3\textwidth]{plots/lin_sep_dataset/SWN/Theorems/weight_grad_ratios_200.pdf}
% }
% \subfigure{
%   \includegraphics[width=0.3\textwidth]{plots/lin_sep_dataset/SWN/Theorems/weight_grad_ratios_250.pdf}
% }
\subfigure{
  \includegraphics[width=0.3\textwidth]{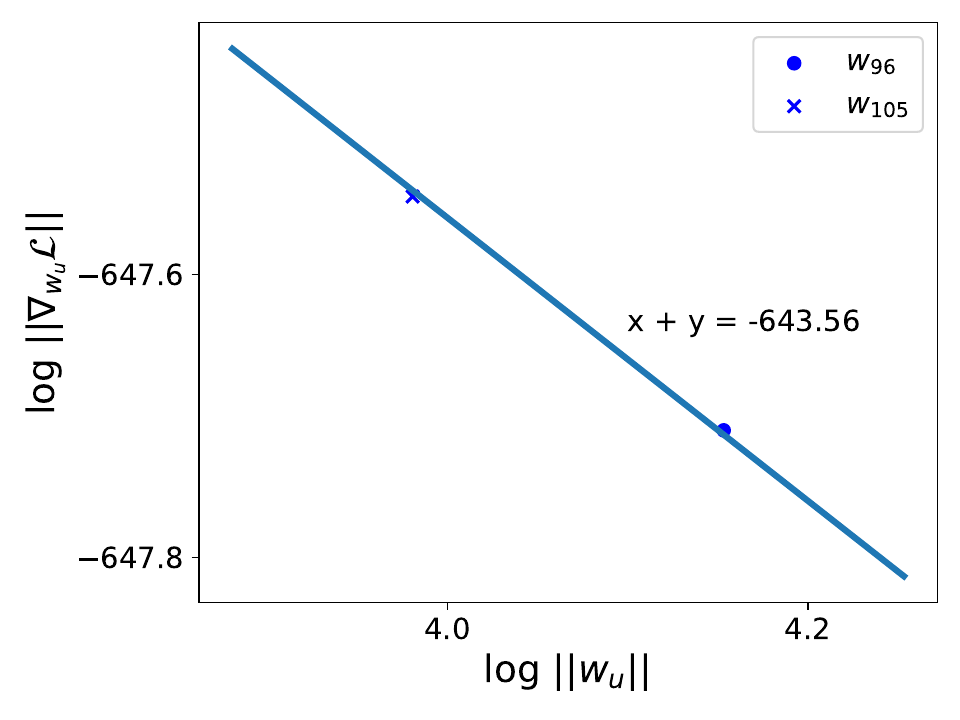}
}
}
\end{figure}

\subsection{\texttt{Simple-Traj}}
The learning rate used was $\frac{k(t)}{\L^{0.9}}$, so that it speeds up at the beginning of training, but slows down as loss approaches $e^{-50}$. The constant $k(t)$ was initialized at $0.01$, and was increased by a factor of $1.1$ every time loss went down and decreased  by a factor of $1.1$ every time loss went up after a gradient step. Its value was capped at $0.1$ for EWN and Unnorm.

\subsection{\texttt{XOR}}
The learning rate used was $\frac{k(t)}{\L^{0.93}}$ for SWN and Unnorm, while $\frac{k(t)}{\L^{0.8}}$ for EWN, so that it speeds up at the beginning of training, but slows down as loss approaches $e^{-50}$. The constant $k(t)$ was initialized at $0.01$, and was increased by a factor of $1.1$ every time loss went down and decreased  by a factor of $1.1$ every time loss went up after a gradient step. Its value was capped at $0.1$ for EWN and Unnorm and $0.01$ for SWN.

\subsection{Convergence rate experiment}
For all SWN, EWN and Unnorm, the learning rate was constant $\eta = 0.001$ and they were trained for 5000 steps. All the networks were explicitly initialized to the same point in function space.

\begin{figure}[t]
\floatconts{conv:MNIST:2356}{\caption{Convergence rate of EWN, SWN and Unnorm on the MNIST dataset for seed - 2356 at different loss values}}  
% The three graphs are plotted at loss values of $e^{-200}, e^{-250}$ and $e^{-300}$ respectively. At each loss value, for the 3 weights, $ \log \| \nabla_{\w_u} \L\| - \log \| \w_u \|$ is approximately same.}}
{
% \subfigure[Dataset]{
%   \includegraphics[width=0.4\textwidth]{plots/lin_sep_dataset/SWN/Assumptions/dataset.pdf}
% }
\subfigure[$\L = e^{-10}$]{
  \includegraphics[width=0.3\textwidth]{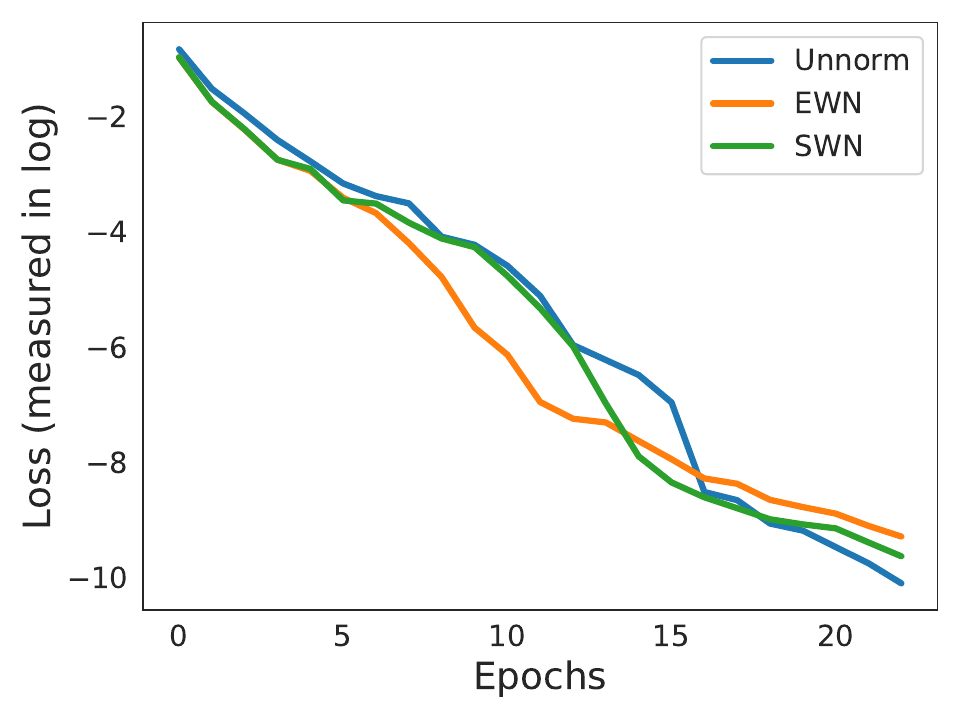}
}
% \subfigure{
%   \includegraphics[width=0.4\textwidth]{plots/lin_sep_dataset/SWN/Theorems/l_tilde_grad.pdf}
% }
\subfigure[$\L = e^{-100}$]{
  \includegraphics[width=0.3\textwidth]{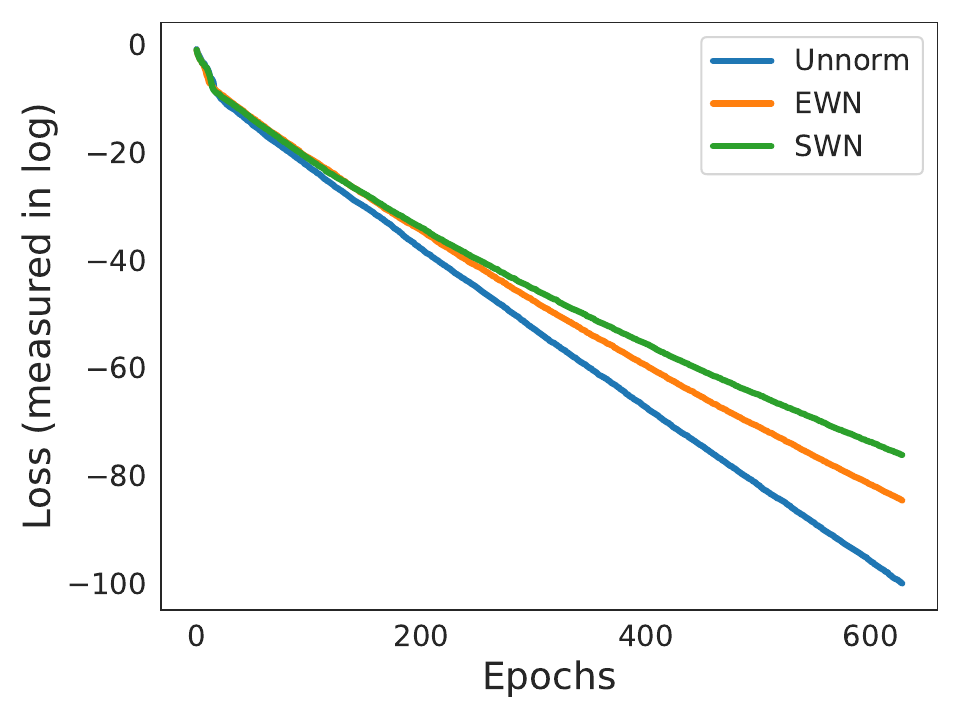}
}
% \subfigure{
%   \includegraphics[width=0.3\textwidth]{plots/lin_sep_dataset/SWN/Theorems/weight_grad_ratios_200.pdf}
% }
% \subfigure{
%   \includegraphics[width=0.3\textwidth]{plots/lin_sep_dataset/SWN/Theorems/weight_grad_ratios_250.pdf}
% }
\subfigure[$\L = e^{-300}$]{
  \includegraphics[width=0.3\textwidth]{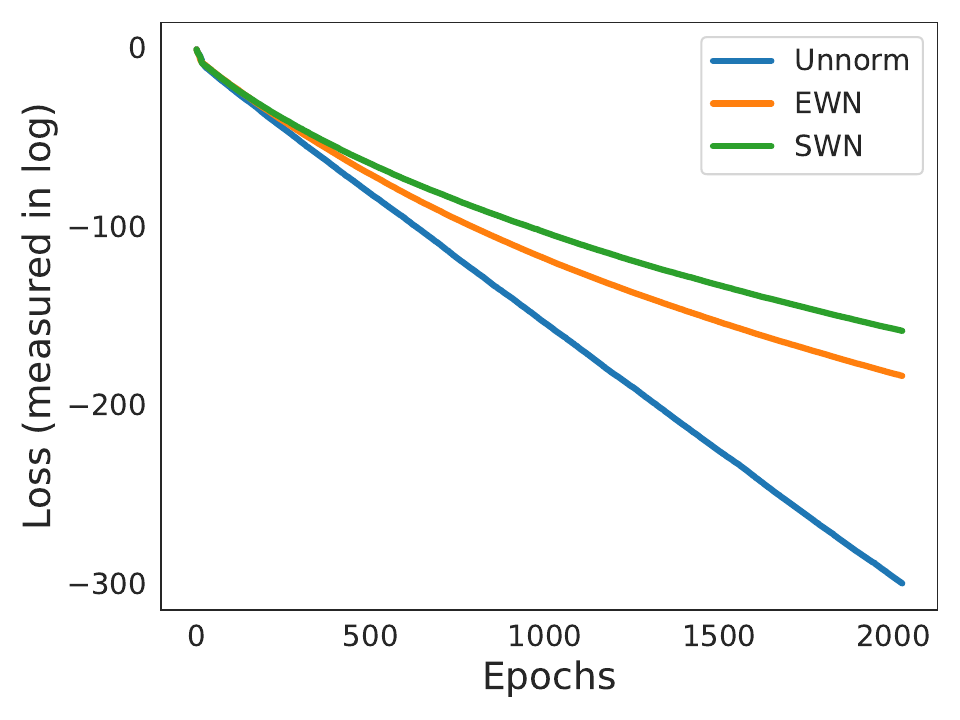}
}
}
\end{figure}

% \begin{figure}[t]
% \centering
%  \begin{subfigure}{0.3\textwidth}
% \centering
%   \includegraphics[width=\textwidth]{plots/MNIST_convergence/MNIST_conv_rate_-10_2356.pdf}
%   \caption{$\L = e^{-10}$}
% \end{subfigure}
% \begin{subfigure}{0.3\textwidth}
% \centering
%   \includegraphics[width=\textwidth]{plots/MNIST_convergence/MNIST_conv_rate_-100_2356.pdf}
%   \caption{$\L = e^{-100}$}
% \end{subfigure}
% \begin{subfigure}{0.3\textwidth}
% \centering
%   \includegraphics[width=\textwidth]{plots/MNIST_convergence/MNIST_conv_rate_-300_2356.pdf}
%   \caption{$\L = e^{-300}$}
% \end{subfigure}
% \caption{Convergence rate of EWN, SWN and Unnorm on the MNIST dataset for seed - 2356 at different loss values}
% \label{conv:MNIST:2356}
% \end{figure}

\begin{figure}[t]
\floatconts{conv:MNIST:3576}{\caption{Convergence rate of EWN, SWN and Unnorm on the MNIST dataset for seed - 3576 at different loss values}}  
% The three graphs are plotted at loss values of $e^{-200}, e^{-250}$ and $e^{-300}$ respectively. At each loss value, for the 3 weights, $ \log \| \nabla_{\w_u} \L\| - \log \| \w_u \|$ is approximately same.}}
{
% \subfigure[Dataset]{
%   \includegraphics[width=0.4\textwidth]{plots/lin_sep_dataset/SWN/Assumptions/dataset.pdf}
% }
\subfigure[$\L = e^{-10}$]{
  \includegraphics[width=0.3\textwidth]{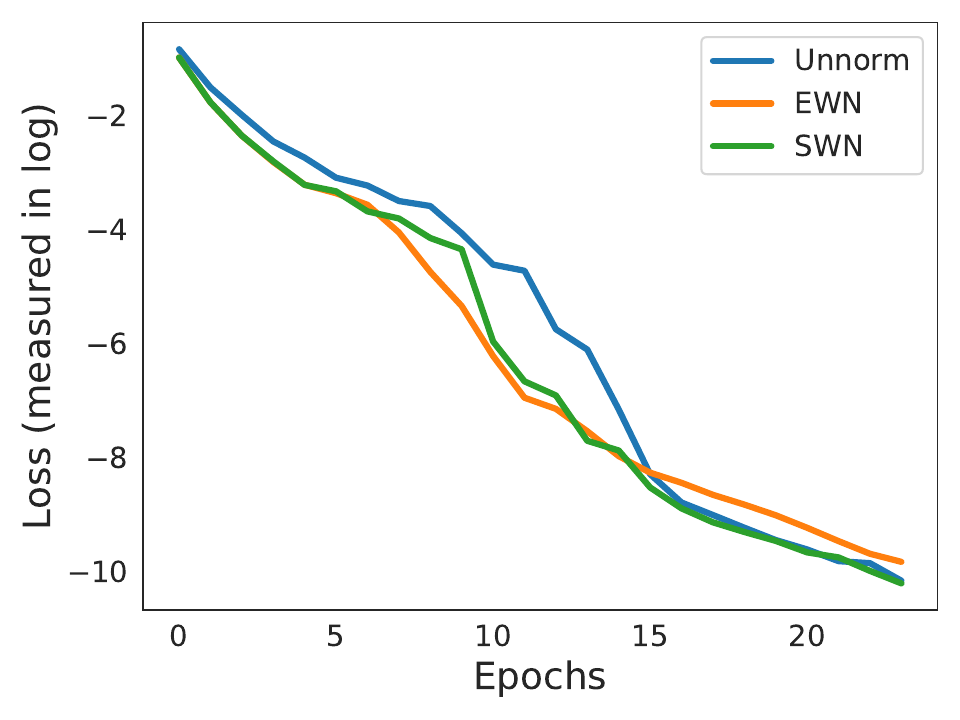}
}
% \subfigure{
%   \includegraphics[width=0.4\textwidth]{plots/lin_sep_dataset/SWN/Theorems/l_tilde_grad.pdf}
% }
\subfigure[$\L = e^{-100}$]{
  \includegraphics[width=0.3\textwidth]{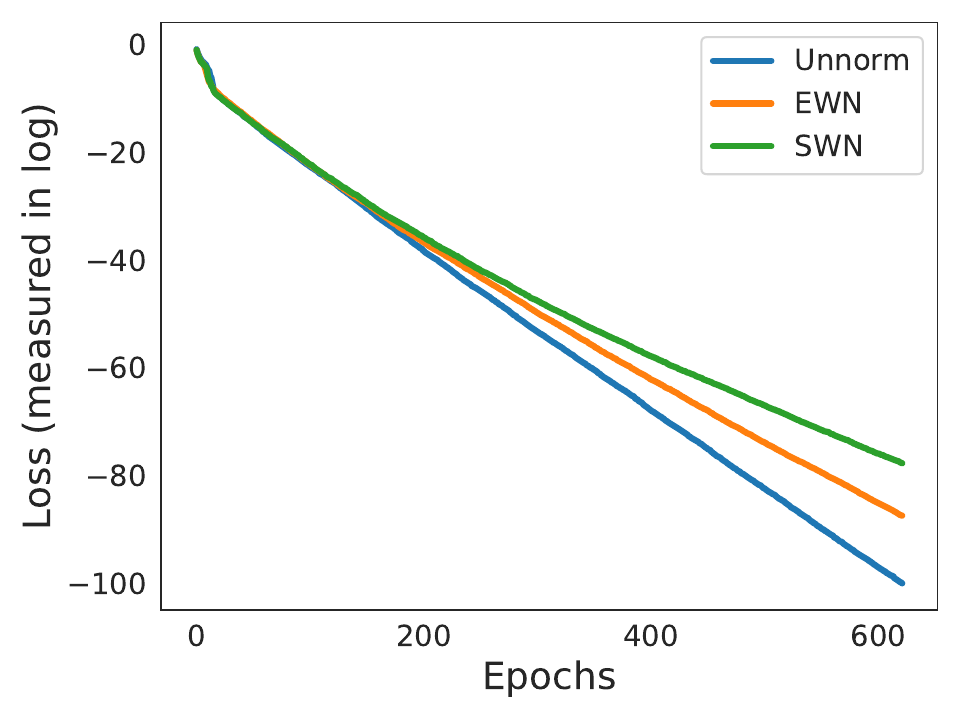}
}
% \subfigure{
%   \includegraphics[width=0.3\textwidth]{plots/lin_sep_dataset/SWN/Theorems/weight_grad_ratios_200.pdf}
% }
% \subfigure{
%   \includegraphics[width=0.3\textwidth]{plots/lin_sep_dataset/SWN/Theorems/weight_grad_ratios_250.pdf}
% }
\subfigure[$\L = e^{-300}$]{
  \includegraphics[width=0.3\textwidth]{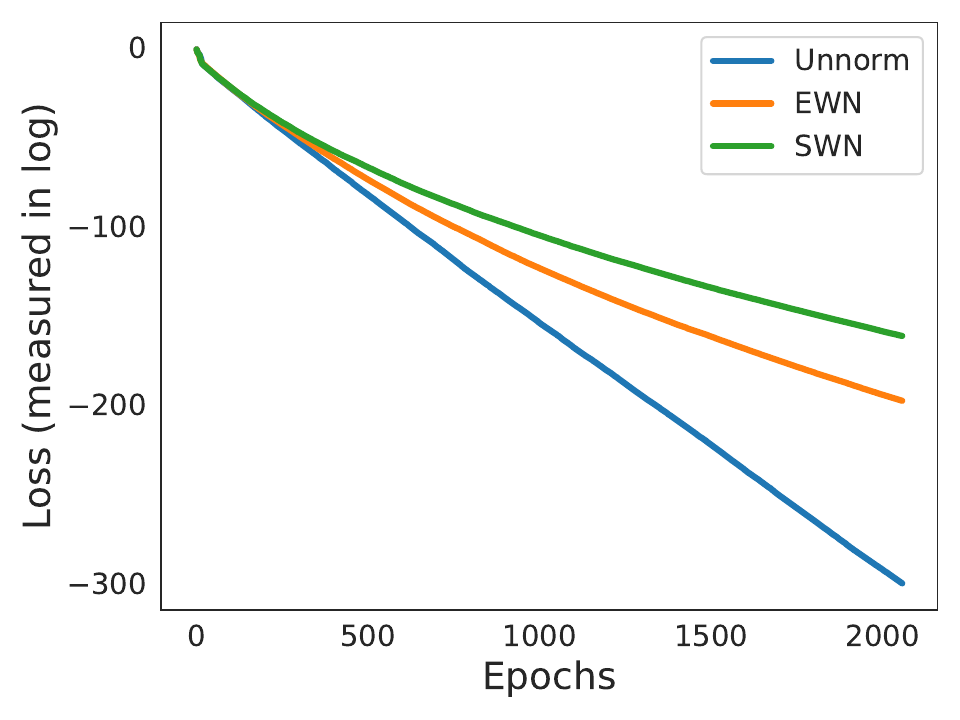}
}
}
\end{figure}

% \begin{figure}[t]
% \centering
%  \begin{subfigure}{0.3\textwidth}
% \centering
%   \includegraphics[width=\textwidth]{plots/MNIST_convergence/MNIST_conv_rate_-10_3576.pdf}
%   \caption{$\L = e^{-10}$}
% \end{subfigure}
% \begin{subfigure}{0.3\textwidth}
% \centering
%   \includegraphics[width=\textwidth]{plots/MNIST_convergence/MNIST_conv_rate_-100_3576.pdf}
%   \caption{$\L = e^{-100}$}
% \end{subfigure}
% \begin{subfigure}{0.3\textwidth}
% \centering
%   \includegraphics[width=\textwidth]{plots/MNIST_convergence/MNIST_conv_rate_-300_3576.pdf}
%   \caption{$\L = e^{-300}$}
% \end{subfigure}
% \caption{Convergence rate of EWN, SWN and Unnorm on the MNIST dataset for seed - 3576 at different loss values}
% \label{conv:MNIST:3576}
% \end{figure}

\subsection{Pruning Experiments}
The learning rate used was $\frac{k(t)}{\L}$. The constant $k(t)$ was initialized at $0.01$, and was increased by a factor of $1.1$ every time loss went down and decreased  by a factor of $1.1$ every time loss went up after an epoch.
% \section{Future work}
% There are multiple research directions that can be pursued following this work

% \begin{itemize}
%     \item As we demonstrated that exponential weight normalization is better suited for pruning, combining it with existing pruning methods can help in designing low-memory footprint networks.
%     \item We relied on the assumption that the network is smooth, that doesn't apply to ReLU non-linearity. Can this assumption be dropped?
%     \item Multiple trajectory based assumptions were used in the paper to derive the conclusions. Understanding when these assumptions hold can give a better idea of what kind of architectures the theory works for.
%     \item Understanding how the sparsity in weights is affected by the dataset and the network architecture is also an interesting research direction.
% \end{itemize}

\section{Demonstration of Theorems on various datasets} \label{app:swn}
In this section, we demonstrate Theorem \ref{thm:weightconv} and \ref{thm:weightnorm} on various datasets - \texttt{Lin-Sep}, \texttt{XOR} and MNIST (2-class classification).

\subsection{\texttt{Lin-Sep}}
We demonstrate Theorem \ref{thm:weightconv} and Theorem \ref{thm:weightnorm} for EWN and SWN on a linearly separable dataset (\texttt{Lin-Sep}) in Figure \ref{demons_lin_sep} and \ref{SWN:demons_lin_sep} respectively. As can be seen in Figure \ref{SWN:demons_lin_sep}, for weights 5, 6 and 8, whose norms keep on growing, weights and gradients eventually become oppositely aligned, and their norms are directly proportional to each other.

\subsection{\texttt{XOR}}
In this experiment, we train a 2-layer network with 20 hidden neurons and ReLU-square activation on XOR dataset, till a loss value of $e^{-40}$. We demonstrate Theorems \ref{thm:weightconv} and \ref{thm:weightnorm} for EWN and SWN in Figure \ref{EWN:demons_XOR} and Figure \ref{SWN:demons_XOR} respectively. 

As can be seen in Figure \ref{EWN:demons_XOR}, for weights 0, 1, 13 and 17, whose norms keep on growing, weights and gradients eventually become oppositely aligned, and their norms are inversely proportional to each other.

Similarly, in Figure \ref{SWN:demons_XOR}, for weights 0, 1, 13 and 17, whose norms keep on growing, weights and gradients eventually become oppositely aligned, and their norms are directly proportional to each other.

\subsection{MNIST}
In this experiment, we train a 2-layer network with 128 hidden neurons and ReLU-square activation on MNIST dataset with 2 classes - $\{0,1\}$. Since even after considering just 2 classes, this dataset is huge, therefore training takes longer. Therefore, we only consider exponential weight normalized network in this case. 

We demonstrate Theorems \ref{thm:weightconv} and \ref{thm:weightnorm} for MNIST dataset for EWN in Figure \ref{EWN:demons_MNIST}. As can be seen, for weights 96 and 105, whose norms keep on growing, weights and gradients eventually become oppositely aligned, and their norms are inversely proportional to each other.

% \section{MNIST Pruning Experiments}

% \begin{figure}[t]
% \floatconts{norm_dist_MNIST}{\caption{Norm of the weight vector vs gradient descent steps, for various nodes in the first layer, when trained on MNIST.}}
% {
% \subfigure[EWN]{
%   \includegraphics[width=0.4\textwidth]{plots/pruning/EWN_w_norm.png}
% }
% \subfigure[SWN]{
%   \includegraphics[width=0.4\textwidth]{plots/pruning/SWN_w_norm.png}
% }
% \subfigure[Unnorm]{
%   \includegraphics[width=0.4\textwidth]{plots/pruning/NWN_w_norm.png}
% }
% }
% \end{figure}

\section{Convergence rate plots for pruning experiment} \label{pruning:conv}
In this section, we provide convergence rate plots for the pruning experiments.

% \subsection{CIFAR pruning}
% In the case of CIFAR-10 dataset, training till a loss value of $e^{-300}$ takes around 4 hrs on GeForce GTX 1080Ti GPU, while it takes almost 1 hr to reach the loss value of $e^{-10}$. This shows that although the loss value seems very low, but it can be achieved with moderate computation power.

% The convergence rate for 2 different seeds at various loss levels are shown in Figure \ref{conv:CIFAR:3576} and \ref{conv:CIFAR:4512}. As can be seen, at different loss levels, the convergence rate of EWN, SWN and unnorm are comparable.

% \subsection{MNIST pruning}
%For MNIST dataset, training is slower as there is only one layer to optimize effectively. It takes around 1.5 days to reach a value of $e^{-300}$ with GeForce GTX 1080Ti GPU.

The convergence rate for 2 different seeds at various loss levels are shown in Figure \ref{conv:MNIST:2356} and \ref{conv:MNIST:3576}. As can be seen, initially the convergence rate of all the normalizations are comparable. But at extremely low loss values, Unnorm becomes slightly faster as compared to SWN or EWN. Note that the results regarding asymptotic convergence rate do not apply in this case, as we are training at extremely high learning rates of $O\left(\frac{1}{\L}\right)$.

\section{Verification of Assumptions for various datasets} \label{assum:verif}
In this section, we verify the assumptions on three datasets - \texttt{Lin-Sep}, \texttt{XOR} and MNIST (2-class classification). Note that in Figures \ref{assum_lin_sep}, \ref{SWN:assum_lin_sep}, \ref{EWN:assum_XOR}, \ref{SWN:assum_XOR} and \ref{EWN:assum_MNIST}, plots demonstrating the components of unit gradient vector, each line corresponds to a single parameter of the network, while in the plots demonstrating the evolution of $\eta(t) \| \w_u(t) \| \| \nabla_{\w_u} \L(\w(t)) \|$, each line corresponds to a neuron of the network.

\subsection{\texttt{Lin-sep}}
We verify assumptions (B1)-(B3) for the \texttt{Lin-Sep} experiment for EWN and SWN in Figure \ref{assum_lin_sep} and Figure \ref{SWN:assum_lin_sep} respectively. As can be seen, for both the cases, the components of unit gradient vector become constant as training proceeds. Another thing to note, is that even for an aggressive learning rate schedule of the form $\frac{1}{\L^{0.97}}$, $\eta(t) \| \w_u(t) \| \| \nabla_{\w_u} \L(\w(t)) \|$ still goes down to 0.

\subsection{\texttt{XOR}}
In this experiment, we train a 2-layer network with 20 hidden neurons and ReLU-square activation on XOR dataset, till a loss value of $e^{-40}$. We verify assumptions (B1)-(B3) for the \texttt{XOR} experiment for EWN and SWN in Figure \ref{EWN:assum_XOR} and Figure \ref{SWN:assum_XOR} respectively. As can be seen, for both the cases, the components of unit gradient vector become constant as training proceeds. Another thing to note, is that even for an aggressive learning rate schedule of the form $\frac{1}{\L^{0.93}}$, $\eta(t) \| \w_u(t) \| \| \nabla_{\w_u} \L(\w(t)) \|$ still goes down to 0.

\subsection{MNIST}
In this experiment, we train a 2-layer network with 128 hidden neurons and ReLU-square activation on MNIST dataset with 2 classes - $\{0,1\}$. Since even after considering just 2 classes, this dataset is huge, therefore training takes longer. Therefore, we only consider exponential weight normalized network in this case. 

We verify assumptions (B1)-(B3) for MNIST dataset for EWN in Figure \ref{EWN:assum_MNIST}. As can be seen, the components of unit gradient vector become constant as training proceeds. Another thing to note, is that even for an aggressive learning rate schedule of the form $\frac{1}{\L^{0.97}}$, $\eta(t) \| \w_u(t) \| \| \nabla_{\w_u} \L(\w(t)) \|$ still goes down to 0.

\begin{figure}[t]
\floatconts{assum_lin_sep}{\caption{\textbf{Verification of Assumptions for EWN in \texttt{Lin-Sep} experiment: } (a) Evolution of $\frac{\nabla_{\w} \L(t)}{\| \nabla_{\w} \L(t) \|}$ (b) Evolution of $\eta(t) \| \w_u(t) \| \| \nabla_{\w_u} \L(t) \|$}}  
% The three graphs are plotted at loss values of $e^{-200}, e^{-250}$ and $e^{-300}$ respectively. At each loss value, for the 3 weights, $ \log \| \nabla_{\w_u} \L\| - \log \| \w_u \|$ is approximately same.}}
{
% \subfigure[Dataset]{
%   \includegraphics[width=0.4\textwidth]{plots/lin_sep_dataset/SWN/Assumptions/dataset.pdf}
% }
% \subfigure[$\L = e^{-10}$]{
%   \includegraphics[width=0.3\textwidth]{plots/MNIST_convergence/MNIST_conv_rate_-10_2356.pdf}
% }
% \subfigure{
%   \includegraphics[width=0.4\textwidth]{plots/lin_sep_dataset/SWN/Theorems/l_tilde_grad.pdf}
% }
\subfigure{
  \includegraphics[width=0.4\textwidth]{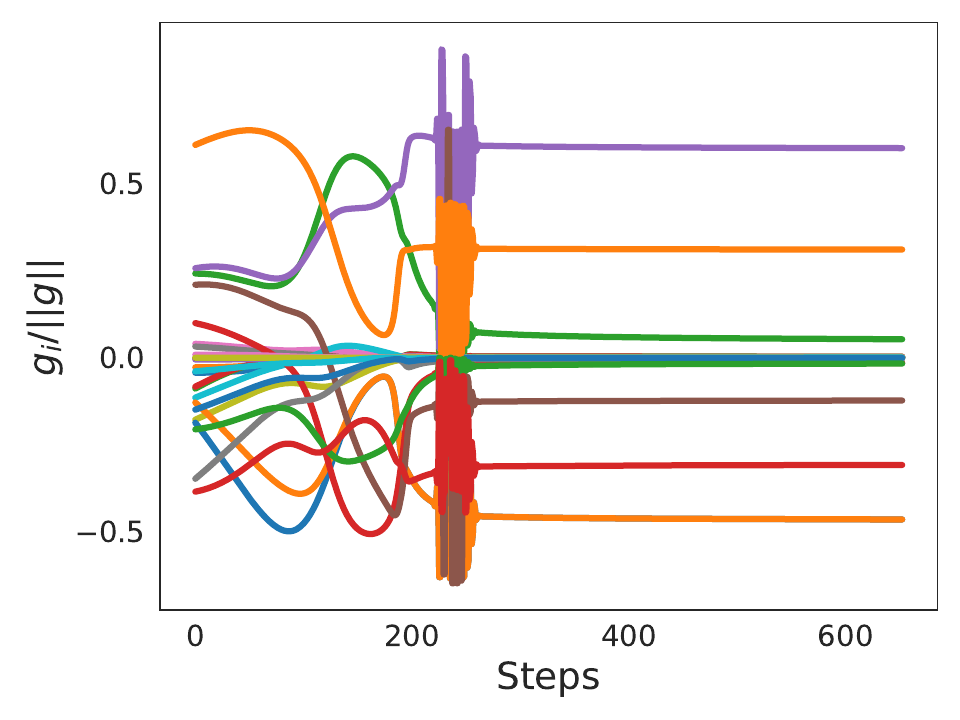}
}
% \subfigure{
%   \includegraphics[width=0.3\textwidth]{plots/lin_sep_dataset/SWN/Theorems/weight_grad_ratios_200.pdf}
% }
% \subfigure{
%   \includegraphics[width=0.3\textwidth]{plots/lin_sep_dataset/SWN/Theorems/weight_grad_ratios_250.pdf}
% }
\subfigure{
  \includegraphics[width=0.4\textwidth]{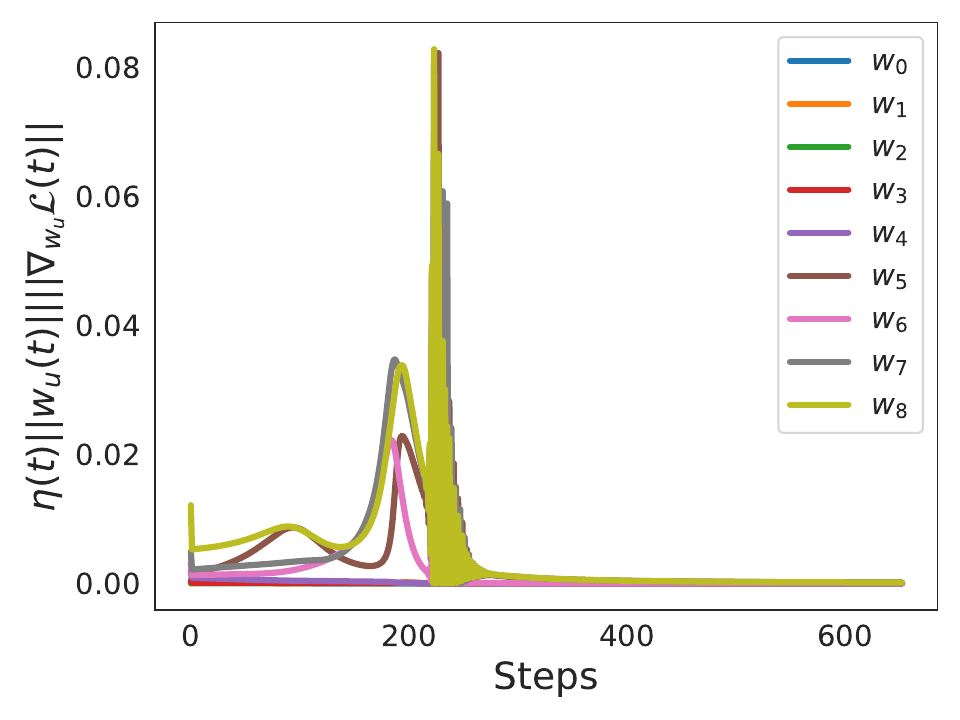}
}
}
\end{figure}

% \begin{figure}[t]
% % \begin{subfigure}{.4\textwidth}
% %   \centering
% %   \includegraphics[width=\textwidth]{plots/lin_sep_dataset/Assumptions/dataset.pdf}
% %   \caption{}
% % \end{subfigure}%
% \begin{subfigure}{.5\textwidth}
%   \centering
%   \includegraphics[width=\textwidth]{plots/lin_sep_dataset/Assumptions/grad_conv.pdf}
%   \caption{}
%   %\label{fig:test2}
% \end{subfigure}
% \begin{subfigure}{.5\textwidth}
%   \centering
%   \includegraphics[width=\textwidth]{plots/lin_sep_dataset/Assumptions/EWN_grad_lr_norm.pdf}
%   \caption{}
%   %\label{fig:test2}
% \end{subfigure}
% \caption{\textbf{Verification of Assumptions for EWN in \texttt{Lin-Sep} experiment: } (a) Evolution of $\frac{\nabla_{\w} \L(t)}{\| \nabla_{\w} \L(t) \|}$ (b) Evolution of $\eta(t) \| \w_u(t) \| \| \nabla_{\w_u} \L(t) \|$}
% \label{assum_lin_sep}
% \end{figure}

\begin{figure}[t]
\floatconts{SWN:assum_lin_sep}{\caption{\textbf{Verification of Assumptions for SWN in \texttt{Lin-Sep} experiment: } (a) Evolution of $\frac{\nabla_{\w} \L(t)}{\| \nabla_{\w} \L(t) \|}$ (b) Evolution of $\eta(t) \| \w_u(t) \| \| \nabla_{\w_u} \L(t) \|$}}  
% The three graphs are plotted at loss values of $e^{-200}, e^{-250}$ and $e^{-300}$ respectively. At each loss value, for the 3 weights, $ \log \| \nabla_{\w_u} \L\| - \log \| \w_u \|$ is approximately same.}}
{
% \subfigure[Dataset]{
%   \includegraphics[width=0.4\textwidth]{plots/lin_sep_dataset/SWN/Assumptions/dataset.pdf}
% }
% \subfigure[$\L = e^{-10}$]{
%   \includegraphics[width=0.3\textwidth]{plots/MNIST_convergence/MNIST_conv_rate_-10_2356.pdf}
% }
% \subfigure{
%   \includegraphics[width=0.4\textwidth]{plots/lin_sep_dataset/SWN/Theorems/l_tilde_grad.pdf}
% }
\subfigure{
  \includegraphics[width=0.4\textwidth]{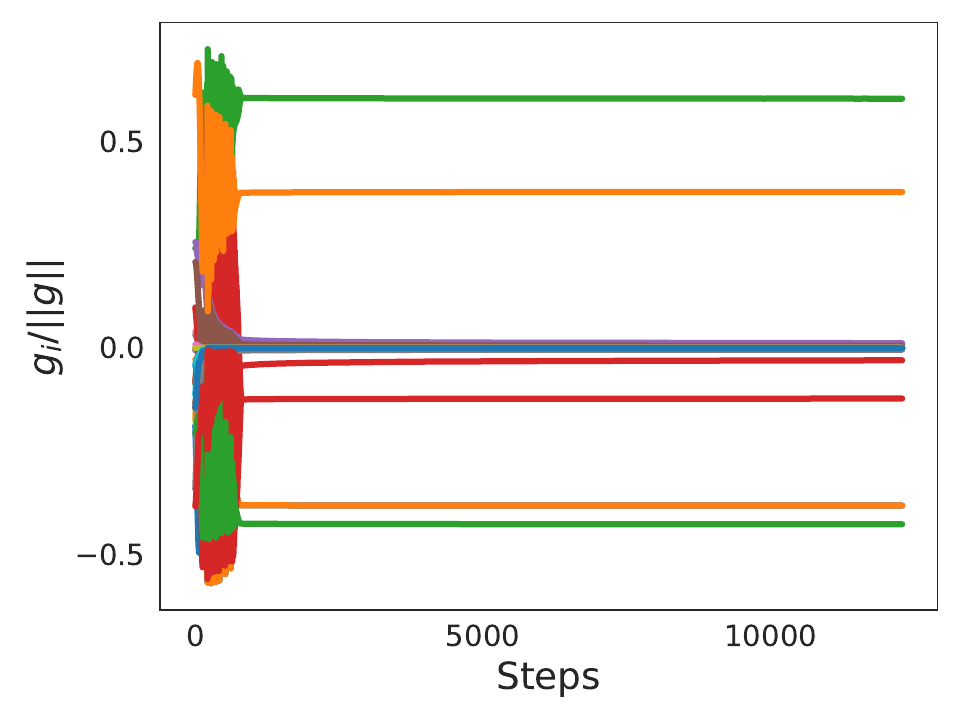}
}
% \subfigure{
%   \includegraphics[width=0.3\textwidth]{plots/lin_sep_dataset/SWN/Theorems/weight_grad_ratios_200.pdf}
% }
% \subfigure{
%   \includegraphics[width=0.3\textwidth]{plots/lin_sep_dataset/SWN/Theorems/weight_grad_ratios_250.pdf}
% }
\subfigure{
  \includegraphics[width=0.4\textwidth]{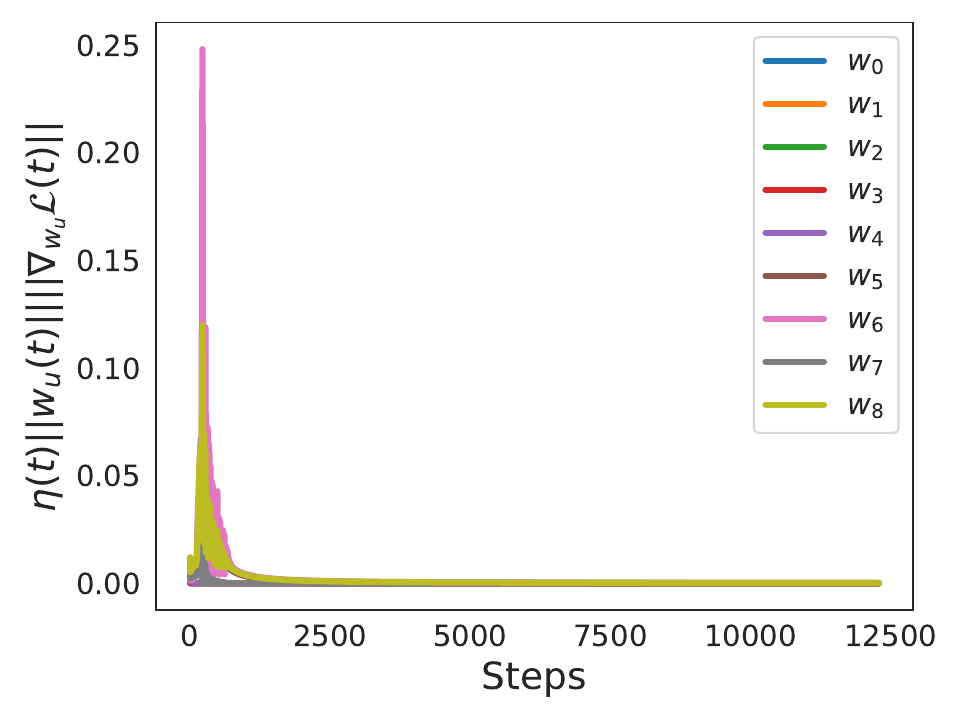}
}
}
\end{figure}

\begin{figure}[t]
\floatconts{EWN:assum_XOR}{\caption{\textbf{Verification of Assumptions for EWN in \texttt{XOR} experiment with ReLU-square activation: } (a) Evolution of $\frac{\nabla_{\w} \L(t)}{\| \nabla_{\w} \L(t) \|}$ (b) Evolution of $\eta(t) \| \w_u(t) \| \| \nabla_{\w_u} \L(t) \|$}}  
% The three graphs are plotted at loss values of $e^{-200}, e^{-250}$ and $e^{-300}$ respectively. At each loss value, for the 3 weights, $ \log \| \nabla_{\w_u} \L\| - \log \| \w_u \|$ is approximately same.}}
{
% \subfigure[Dataset]{
%   \includegraphics[width=0.4\textwidth]{plots/lin_sep_dataset/SWN/Assumptions/dataset.pdf}
% }
% \subfigure[$\L = e^{-10}$]{
%   \includegraphics[width=0.3\textwidth]{plots/MNIST_convergence/MNIST_conv_rate_-10_2356.pdf}
% }
% \subfigure{
%   \includegraphics[width=0.4\textwidth]{plots/lin_sep_dataset/SWN/Theorems/l_tilde_grad.pdf}
% }
\subfigure{
  \includegraphics[width=0.4\textwidth]{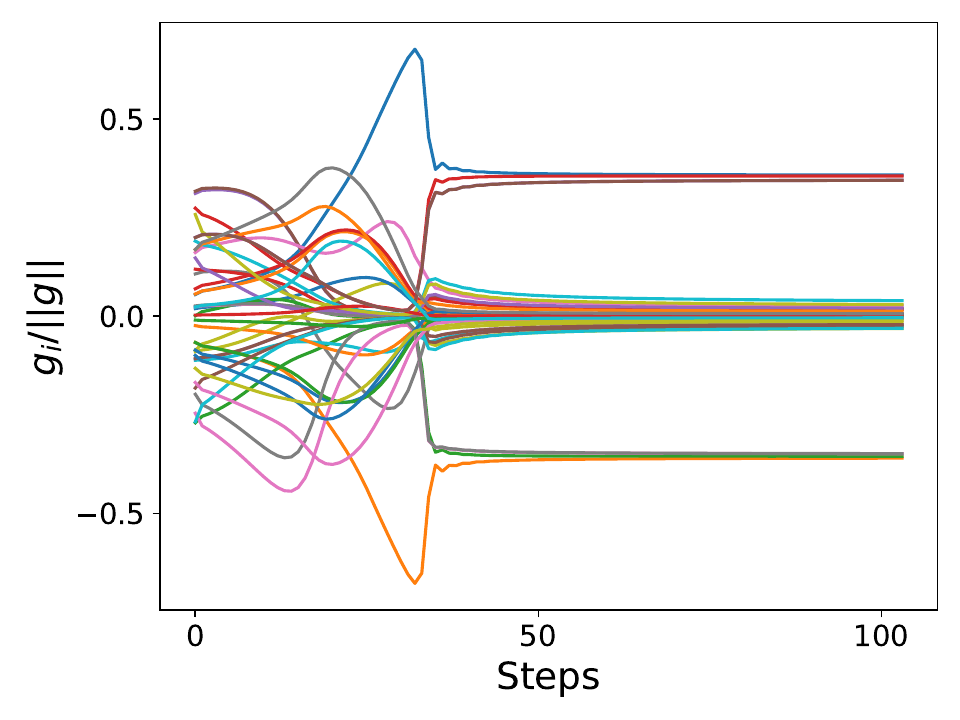}
}
% \subfigure{
%   \includegraphics[width=0.3\textwidth]{plots/lin_sep_dataset/SWN/Theorems/weight_grad_ratios_200.pdf}
% }
% \subfigure{
%   \includegraphics[width=0.3\textwidth]{plots/lin_sep_dataset/SWN/Theorems/weight_grad_ratios_250.pdf}
% }
\subfigure{
  \includegraphics[width=0.4\textwidth]{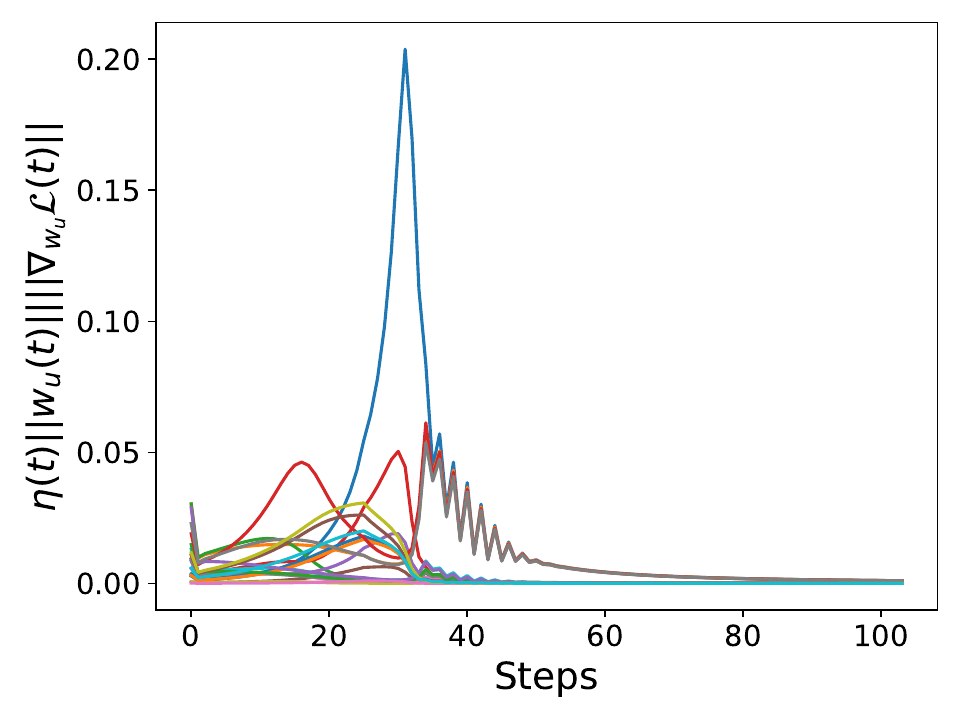}
}
}
\end{figure}

\begin{figure}[t]
\floatconts{SWN:assum_XOR}{\caption{\textbf{Verification of Assumptions for SWN in \texttt{XOR} experiment with ReLU-square activation: } (a) Evolution of $\frac{\nabla_{\w} \L(t)}{\| \nabla_{\w} \L(t) \|}$ (b) Evolution of $\eta(t) \| \w_u(t) \| \| \nabla_{\w_u} \L(t) \|$}}  
% The three graphs are plotted at loss values of $e^{-200}, e^{-250}$ and $e^{-300}$ respectively. At each loss value, for the 3 weights, $ \log \| \nabla_{\w_u} \L\| - \log \| \w_u \|$ is approximately same.}}
{
% \subfigure[Dataset]{
%   \includegraphics[width=0.4\textwidth]{plots/lin_sep_dataset/SWN/Assumptions/dataset.pdf}
% }
% \subfigure[$\L = e^{-10}$]{
%   \includegraphics[width=0.3\textwidth]{plots/MNIST_convergence/MNIST_conv_rate_-10_2356.pdf}
% }
% \subfigure{
%   \includegraphics[width=0.4\textwidth]{plots/lin_sep_dataset/SWN/Theorems/l_tilde_grad.pdf}
% }
\subfigure{
  \includegraphics[width=0.4\textwidth]{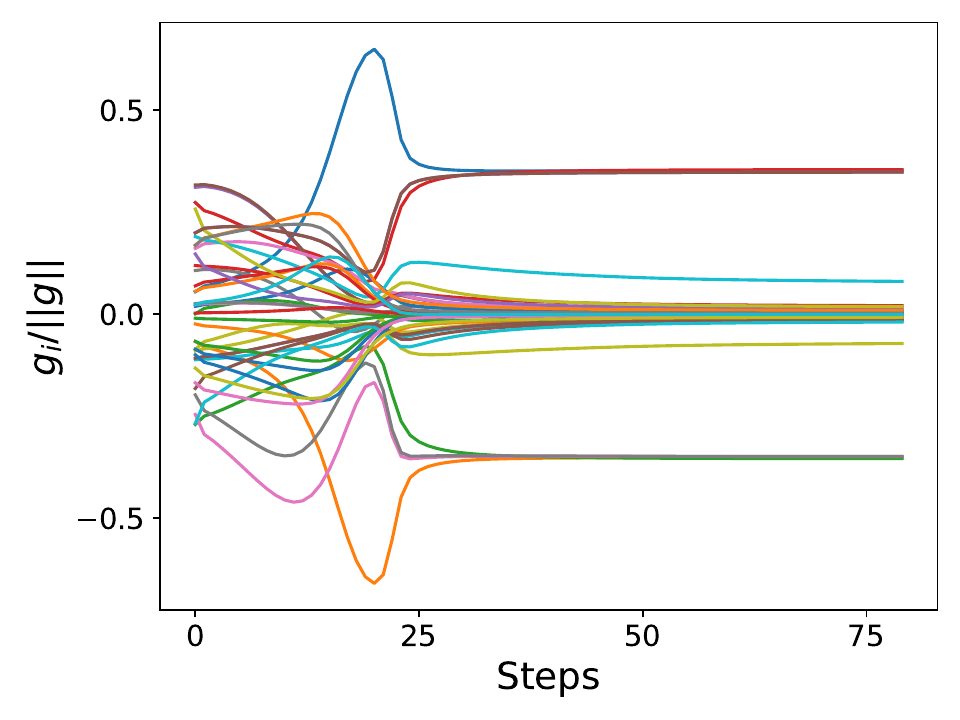}
}
% \subfigure{
%   \includegraphics[width=0.3\textwidth]{plots/lin_sep_dataset/SWN/Theorems/weight_grad_ratios_200.pdf}
% }
% \subfigure{
%   \includegraphics[width=0.3\textwidth]{plots/lin_sep_dataset/SWN/Theorems/weight_grad_ratios_250.pdf}
% }
\subfigure{
  \includegraphics[width=0.4\textwidth]{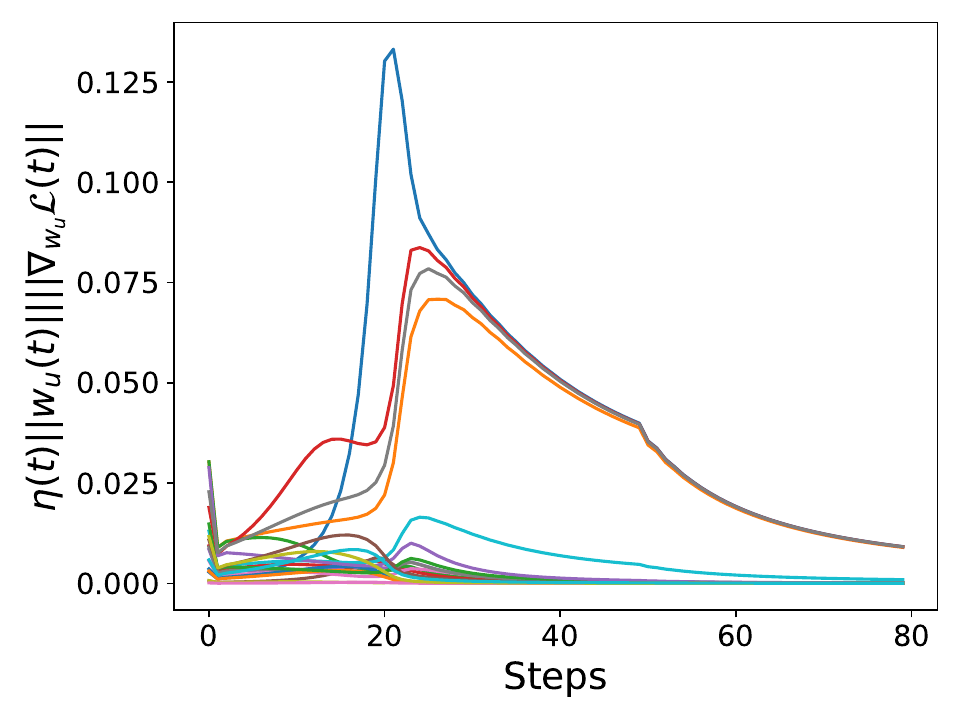}
}
}
\end{figure}

\begin{figure}[t]
\floatconts{EWN:assum_MNIST}{\caption{\textbf{Verification of Assumptions for EWN on MNIST dataset with ReLU-square activation: } (a) Evolution of $\frac{\nabla_{\w} \L(t)}{\| \nabla_{\w} \L(t) \|}$ (b) Evolution of $\eta(t) \| \w_u(t) \| \| \nabla_{\w_u} \L(t) \|$}}  
% The three graphs are plotted at loss values of $e^{-200}, e^{-250}$ and $e^{-300}$ respectively. At each loss value, for the 3 weights, $ \log \| \nabla_{\w_u} \L\| - \log \| \w_u \|$ is approximately same.}}
{
% \subfigure[Dataset]{
%   \includegraphics[width=0.4\textwidth]{plots/lin_sep_dataset/SWN/Assumptions/dataset.pdf}
% }
% \subfigure[$\L = e^{-10}$]{
%   \includegraphics[width=0.3\textwidth]{plots/MNIST_convergence/MNIST_conv_rate_-10_2356.pdf}
% }
% \subfigure{
%   \includegraphics[width=0.4\textwidth]{plots/lin_sep_dataset/SWN/Theorems/l_tilde_grad.pdf}
% }
\subfigure{
  \includegraphics[width=0.4\textwidth]{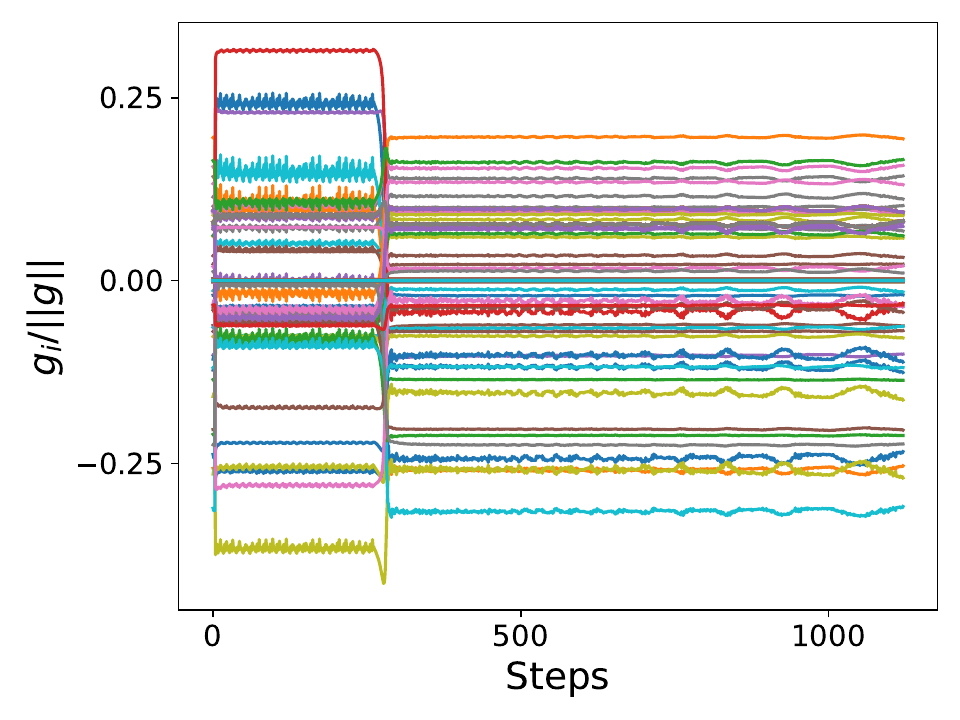}
}
% \subfigure{
%   \includegraphics[width=0.3\textwidth]{plots/lin_sep_dataset/SWN/Theorems/weight_grad_ratios_200.pdf}
% }
% \subfigure{
%   \includegraphics[width=0.3\textwidth]{plots/lin_sep_dataset/SWN/Theorems/weight_grad_ratios_250.pdf}
% }
\subfigure{
  \includegraphics[width=0.4\textwidth]{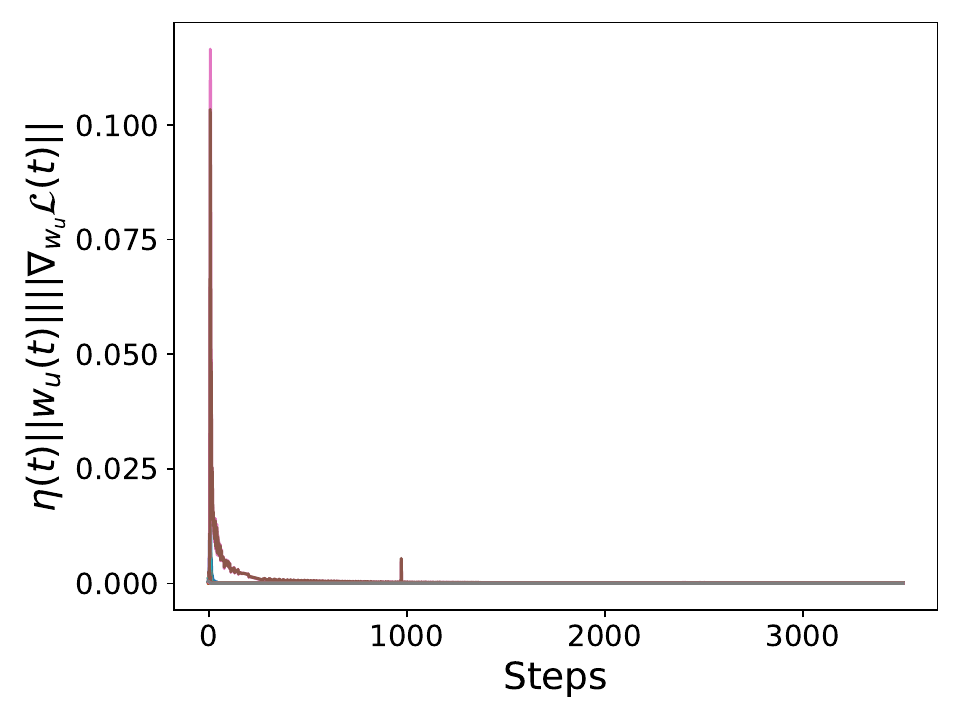}
}
}
\end{figure}

\section{Pruning algorithm} \label{pruning:algo}
We will explain the pruning algorithm used in Figure \ref{pruning}b. The same algorithm is used for SWN, EWN as well as the unnormalized net.

Let $t_1$ and $t_2$ denote the optimization iteration indices when log-loss has a value of $-10$ and $-100$ respectively. Consider three pruning strategies available at the time instant given by $t_2$:
\begin{itemize}
    \item Prune weights on the basis of $\| \w_u(t_2) \|$
    \item Prune weights on the basis of $\| \w_u(t_2) - \w_u(0)\|$
    \item Prune weights on the basis of $\| \w_u(t_2) - \w_u(t_1)\|$ 
\end{itemize}

For a given level of pruning, we try all the 3 strategies, and then pick the one with the best test performance.

Variants of option (a) are the most prevalent pruning algorithms. The non-standard options of (b) and (c) represent the intuition of pruning neurons whose weight has not moved in the recent past. This intuition is in turn motivated by the asymptotic behaviour of the inductive bias, i.e. the neuron weight norms become ‘relatively sparse’ in the ‘limit’.

In fact, if we prune the network based on the final L2 norm of the weights (strategy (a)), we get a similar pruning performance for EWN as the current pruning algorithm at a loss value of $e^{-100}$ and $e^{-300}$. However, the pruning performance of EWN at the loss value of $e^{-10}$ drops. We believe this is because when the loss values are less than $e^{-100}$ the network is already in the asymptotic regime and hence the difference between using the three pruning strategies is not significant. However, when the loss values are around $e^{-10}$, the network is not in the asymptotic regime, and a difference emerges between the three pruning strategies.

%% file: morwani22.bbl
\begin{thebibliography}{36}
\providecommand{\natexlab}[1]{#1}
\providecommand{\url}[1]{\texttt{#1}}
\expandafter\ifx\csname urlstyle\endcsname\relax
  \providecommand{\doi}[1]{doi: #1}\else
  \providecommand{\doi}{doi: \begingroup \urlstyle{rm}\Url}\fi

\bibitem[Arora et~al.(2019)Arora, Li, and Lyu]{Arora+19}
Sanjeev Arora, Zhiyuan Li, and Kaifeng Lyu.
\newblock Theoretical analysis of auto rate-tuning by batch normalization.
\newblock In \emph{International Conference on Learning Representations}, 2019.
\newblock URL \url{https://openreview.net/forum?id=rkxQ-nA9FX}.

\bibitem[Ba et~al.(2016)Ba, Kiros, and Hinton]{Ba+16}
Jimmy~Lei Ba, Jamie~Ryan Kiros, and Geoffrey~E Hinton.
\newblock Layer normalization.
\newblock \emph{arXiv preprint arXiv:1607.06450}, 2016.

\bibitem[Bjorck et~al.(2018)Bjorck, Gomes, Selman, and Weinberger]{Bjorck+18}
Nils Bjorck, Carla~P Gomes, Bart Selman, and Kilian~Q Weinberger.
\newblock Understanding batch normalization.
\newblock In S.~Bengio, H.~Wallach, H.~Larochelle, K.~Grauman, N.~Cesa-Bianchi,
  and R.~Garnett, editors, \emph{Advances in Neural Information Processing
  Systems 31}, pages 7694--7705. Curran Associates, Inc., 2018.
\newblock URL
  \url{http://papers.nips.cc/paper/7996-understanding-batch-normalization.pdf}.

\bibitem[Cai et~al.(2019)Cai, Li, and Shen]{Cai+19}
Yongqiang Cai, Qianxiao Li, and Zuowei Shen.
\newblock A quantitative analysis of the effect of batch normalization on
  gradient descent.
\newblock volume~97 of \emph{Proceedings of Machine Learning Research}, pages
  882--890, Long Beach, California, USA, 09--15 Jun 2019. PMLR.
\newblock URL \url{http://proceedings.mlr.press/v97/cai19a.html}.

\bibitem[Chizat and Bach(2020)]{ChizatBach20}
L\'ena\"ic Chizat and Francis Bach.
\newblock Implicit bias of gradient descent for wide two-layer neural networks
  trained with the logistic loss.
\newblock volume 125 of \emph{Proceedings of Machine Learning Research}, pages
  1305--1338. PMLR, 09--12 Jul 2020.
\newblock URL \url{http://proceedings.mlr.press/v125/chizat20a.html}.

\bibitem[Dauphin et~al.(2017)Dauphin, Fan, Auli, and Grangier]{Dauphin+17}
Yann~N. Dauphin, Angela Fan, Michael Auli, and David Grangier.
\newblock Language modeling with gated convolutional networks.
\newblock volume~70 of \emph{Proceedings of Machine Learning Research}, pages
  933--941, International Convention Centre, Sydney, Australia, 06--11 Aug
  2017. PMLR.
\newblock URL \url{http://proceedings.mlr.press/v70/dauphin17a.html}.

\bibitem[Du et~al.(2018)Du, Lee, Tian, Singh, and Poczos]{Du+18}
Simon Du, Jason Lee, Yuandong Tian, Aarti Singh, and Barnabas Poczos.
\newblock Gradient descent learns one-hidden-layer {CNN}: Don’t be afraid of
  spurious local minima.
\newblock volume~80 of \emph{Proceedings of Machine Learning Research}, pages
  1339--1348, Stockholmsmässan, Stockholm Sweden, 10--15 Jul 2018. PMLR.
\newblock URL \url{http://proceedings.mlr.press/v80/du18b.html}.

\bibitem[Dukler et~al.(2020)Dukler, Gu, and Mont{\'u}far]{Dukler+20}
Yonatan Dukler, Quanquan Gu, and Guido Mont{\'u}far.
\newblock Optimization theory for relu neural networks trained with
  normalization layers.
\newblock In \emph{International Conference on Machine Learning}, 2020.

\bibitem[Gunasekar et~al.(2018{\natexlab{a}})Gunasekar, Lee, Soudry, and
  Srebro]{Gunasekar+18a}
Suriya Gunasekar, Jason Lee, Daniel Soudry, and Nathan Srebro.
\newblock Characterizing implicit bias in terms of optimization geometry.
\newblock volume~80 of \emph{Proceedings of Machine Learning Research}, pages
  1832--1841, Stockholmsmässan, Stockholm Sweden, 10--15 Jul
  2018{\natexlab{a}}. PMLR.
\newblock URL \url{http://proceedings.mlr.press/v80/gunasekar18a.html}.

\bibitem[Gunasekar et~al.(2018{\natexlab{b}})Gunasekar, Lee, Soudry, and
  Srebro]{Gunasekar+18b}
Suriya Gunasekar, Jason~D Lee, Daniel Soudry, and Nati Srebro.
\newblock Implicit bias of gradient descent on linear convolutional networks.
\newblock In S.~Bengio, H.~Wallach, H.~Larochelle, K.~Grauman, N.~Cesa-Bianchi,
  and R.~Garnett, editors, \emph{Advances in Neural Information Processing
  Systems 31}, pages 9461--9471. Curran Associates, Inc., 2018{\natexlab{b}}.
\newblock URL
  \url{http://papers.nips.cc/paper/8156-implicit-bias-of-gradient-descent-on-linear-convolutional-networks.pdf}.

\bibitem[Hieber et~al.(2018)Hieber, Domhan, Denkowski, Vilar, Sokolov, Clifton,
  and Post]{Hiebler+18}
Felix Hieber, Tobias Domhan, Michael Denkowski, David Vilar, Artem Sokolov, Ann
  Clifton, and Matt Post.
\newblock Sockeye: A toolkit for neural machine translation, 2018.

\bibitem[Jacot et~al.(2018)Jacot, Gabriel, and Hongler]{Jacot+18}
Arthur Jacot, Franck Gabriel, and Clement Hongler.
\newblock Neural tangent kernel: Convergence and generalization in neural
  networks.
\newblock In S.~Bengio, H.~Wallach, H.~Larochelle, K.~Grauman, N.~Cesa-Bianchi,
  and R.~Garnett, editors, \emph{Advances in Neural Information Processing
  Systems 31}, pages 8571--8580. Curran Associates, Inc., 2018.
\newblock URL
  \url{http://papers.nips.cc/paper/8076-neural-tangent-kernel-convergence-and-generalization-in-neural-networks.pdf}.

\bibitem[Ji and Telgarsky(2019{\natexlab{a}})]{JiTelgarsky19b}
Ziwei Ji and Matus Telgarsky.
\newblock The implicit bias of gradient descent on nonseparable data.
\newblock volume~99 of \emph{Proceedings of Machine Learning Research}, pages
  1772--1798, Phoenix, USA, 25--28 Jun 2019{\natexlab{a}}. PMLR.
\newblock URL \url{http://proceedings.mlr.press/v99/ji19a.html}.

\bibitem[Ji and Telgarsky(2019{\natexlab{b}})]{JiTelgarsky19c}
Ziwei Ji and Matus Telgarsky.
\newblock A refined primal-dual analysis of the implicit bias,
  2019{\natexlab{b}}.

\bibitem[Ji and Telgarsky(2019{\natexlab{c}})]{JiTelgrasky19a}
Ziwei Ji and Matus Telgarsky.
\newblock Gradient descent aligns the layers of deep linear networks.
\newblock In \emph{International Conference on Learning Representations},
  2019{\natexlab{c}}.
\newblock URL \url{https://openreview.net/forum?id=HJflg30qKX}.

\bibitem[Ji and Telgarsky(2020)]{JiTelgarsky20}
Ziwei Ji and Matus Telgarsky.
\newblock Directional convergence and alignment in deep learning.
\newblock \emph{arXiv preprint arXiv:2006.06657}, 2020.

\bibitem[Kim et~al.(2018)Kim, Jun, and Zhang]{Jin-Hwa+18}
Jin-Hwa Kim, Jaehyun Jun, and Byoung-Tak Zhang.
\newblock Bilinear attention networks.
\newblock In S.~Bengio, H.~Wallach, H.~Larochelle, K.~Grauman, N.~Cesa-Bianchi,
  and R.~Garnett, editors, \emph{Advances in Neural Information Processing
  Systems}, volume~31, pages 1564--1574. Curran Associates, Inc., 2018.
\newblock URL
  \url{https://proceedings.neurips.cc/paper/2018/file/96ea64f3a1aa2fd00c72faacf0cb8ac9-Paper.pdf}.

\bibitem[Kohler et~al.(2019)Kohler, Daneshmand, Lucchi, Hofmann, Zhou, and
  Neymeyr]{Kohler+19}
Jonas Kohler, Hadi Daneshmand, Aurelien Lucchi, Thomas Hofmann, Ming Zhou, and
  Klaus Neymeyr.
\newblock Exponential convergence rates for batch normalization: The power of
  length-direction decoupling in non-convex optimization.
\newblock volume~89 of \emph{Proceedings of Machine Learning Research}, pages
  806--815. PMLR, 16--18 Apr 2019.
\newblock URL \url{http://proceedings.mlr.press/v89/kohler19a.html}.

\bibitem[LeCun et~al.(2010)LeCun, Cortes, and Burges]{Lecun+10}
Yann LeCun, Corinna Cortes, and CJ~Burges.
\newblock Mnist handwritten digit database.
\newblock \emph{ATT Labs [Online]. Available:
  http://yann.lecun.com/exdb/mnist}, 2, 2010.

\bibitem[Li et~al.(2019)Li, Wu, Weinberger, and Belongie]{Li+19}
Boyi Li, Felix Wu, Kilian~Q Weinberger, and Serge Belongie.
\newblock Positional normalization.
\newblock In H.~Wallach, H.~Larochelle, A.~Beygelzimer, F.~d`Alch\'{e} Buc,
  E.~Fox, and R.~Garnett, editors, \emph{Advances in Neural Information
  Processing Systems 32}, pages 1622--1634. Curran Associates, Inc., 2019.
\newblock URL
  \url{http://papers.nips.cc/paper/8440-positional-normalization.pdf}.

\bibitem[Luo et~al.(2019)Luo, Wang, Shao, and Peng]{Luo+19}
Ping Luo, Xinjiang Wang, Wenqi Shao, and Zhanglin Peng.
\newblock Towards understanding regularization in batch normalization.
\newblock In \emph{International Conference on Learning Representations}, 2019.
\newblock URL \url{https://openreview.net/forum?id=HJlLKjR9FQ}.

\bibitem[Lyu and Li(2020)]{LyuLi20}
Kaifeng Lyu and Jian Li.
\newblock Gradient descent maximizes the margin of homogeneous neural networks.
\newblock In \emph{International Conference on Learning Representations}, 2020.
\newblock URL \url{https://openreview.net/forum?id=SJeLIgBKPS}.

\bibitem[Moroshko et~al.(2020)Moroshko, Gunasekar, Woodworth, Lee, Srebro, and
  Soudry]{Moroshko+20}
Edward Moroshko, Suriya Gunasekar, Blake Woodworth, Jason~D. Lee, Nathan
  Srebro, and Daniel Soudry.
\newblock Implicit bias in deep linear classification: Initialization scale vs
  training accuracy, 2020.

\bibitem[Muresan(2015)]{muresan2015concrete}
M.~Muresan.
\newblock \emph{A Concrete Approach to Classical Analysis}.
\newblock CMS Books in Mathematics. Springer New York, 2015.
\newblock ISBN 9780387789330.
\newblock URL \url{https://books.google.co.in/books?id=N8rBgtIu\_qgC}.

\bibitem[Nacson et~al.(2019{\natexlab{a}})Nacson, Gunasekar, Lee, Srebro, and
  Soudry]{Nacson+19b}
Mor~Shpigel Nacson, Suriya Gunasekar, Jason Lee, Nathan Srebro, and Daniel
  Soudry.
\newblock Lexicographic and depth-sensitive margins in homogeneous and
  non-homogeneous deep models.
\newblock volume~97 of \emph{Proceedings of Machine Learning Research}, pages
  4683--4692, Long Beach, California, USA, 09--15 Jun 2019{\natexlab{a}}. PMLR.
\newblock URL \url{http://proceedings.mlr.press/v97/nacson19a.html}.

\bibitem[Nacson et~al.(2019{\natexlab{b}})Nacson, Lee, Gunasekar, Savarese,
  Srebro, and Soudry]{Nacson+19a}
Mor~Shpigel Nacson, Jason Lee, Suriya Gunasekar, Pedro Henrique~Pamplona
  Savarese, Nathan Srebro, and Daniel Soudry.
\newblock Convergence of gradient descent on separable data.
\newblock volume~89 of \emph{Proceedings of Machine Learning Research}, pages
  3420--3428. PMLR, 16--18 Apr 2019{\natexlab{b}}.
\newblock URL \url{http://proceedings.mlr.press/v89/nacson19b.html}.

\bibitem[Nacson et~al.(2019{\natexlab{c}})Nacson, Srebro, and
  Soudry]{Nacson+19c}
Mor~Shpigel Nacson, Nathan Srebro, and Daniel Soudry.
\newblock Stochastic gradient descent on separable data: Exact convergence with
  a fixed learning rate.
\newblock volume~89 of \emph{Proceedings of Machine Learning Research}, pages
  3051--3059. PMLR, 16--18 Apr 2019{\natexlab{c}}.
\newblock URL \url{http://proceedings.mlr.press/v89/nacson19a.html}.

\bibitem[Qiao et~al.(2019)Qiao, Wang, Liu, Shen, and Yuille]{Qiao+19}
Siyuan Qiao, Huiyu Wang, Chenxi Liu, Wei Shen, and Alan Yuille.
\newblock Rethinking normalization and elimination singularity in neural
  networks.
\newblock \emph{arXiv preprint arXiv:1911.09738}, 2019.

\bibitem[Qiao et~al.(2020)Qiao, Wang, Liu, Shen, and Yuille]{Qiao+20}
Siyuan Qiao, Huiyu Wang, Chenxi Liu, Wei Shen, and Alan Yuille.
\newblock Micro-batch training with batch-channel normalization and weight
  standardization, 2020.

\bibitem[Roburin et~al.(2020)Roburin, de~Mont-Marin, Bursuc, Marlet, P{\'e}rez,
  and Aubry]{Roburin+20}
Simon Roburin, Yann de~Mont-Marin, Andrei Bursuc, Renaud Marlet, Patrick
  P{\'e}rez, and Mathieu Aubry.
\newblock Spherical perspective on learning with batch norm.
\newblock \emph{arXiv preprint arXiv:2006.13382}, 2020.

\bibitem[Salimans and Kingma(2016)]{SalimansKnigma16}
Tim Salimans and Durk~P Kingma.
\newblock Weight normalization: A simple reparameterization to accelerate
  training of deep neural networks.
\newblock In D.~D. Lee, M.~Sugiyama, U.~V. Luxburg, I.~Guyon, and R.~Garnett,
  editors, \emph{Advances in Neural Information Processing Systems 29}, pages
  901--909. Curran Associates, Inc., 2016.
\newblock URL
  \url{http://papers.nips.cc/paper/6114-weight-normalization-a-simple-reparameterization-to-
  accelerate-training-of-deep-neural-networks.pdf}.

\bibitem[Santurkar et~al.(2018)Santurkar, Tsipras, Ilyas, and
  Madry]{Santurkar+18}
Shibani Santurkar, Dimitris Tsipras, Andrew Ilyas, and Aleksander Madry.
\newblock How does batch normalization help optimization?
\newblock In S.~Bengio, H.~Wallach, H.~Larochelle, K.~Grauman, N.~Cesa-Bianchi,
  and R.~Garnett, editors, \emph{Advances in Neural Information Processing
  Systems 31}, pages 2483--2493. Curran Associates, Inc., 2018.
\newblock URL
  \url{http://papers.nips.cc/paper/7515-how-does-batch-normalization-help-optimization.pdf}.

\bibitem[Sokolic et~al.(2017)Sokolic, Giryes, Sapiro, and
  Rodrigues]{Sokolic+17}
Jure Sokolic, Raja Giryes, Guillermo Sapiro, and Miguel R.~D. Rodrigues.
\newblock Robust large margin deep neural networks.
\newblock \emph{IEEE Transactions on Signal Processing}, 65\penalty0
  (16):\penalty0 4265–4280, Aug 2017.
\newblock ISSN 1941-0476.
\newblock \doi{10.1109/tsp.2017.2708039}.
\newblock URL \url{http://dx.doi.org/10.1109/TSP.2017.2708039}.

\bibitem[Soudry et~al.(2018)Soudry, Hoffer, and Srebro]{Soudry+18}
Daniel Soudry, Elad Hoffer, and Nathan Srebro.
\newblock The implicit bias of gradient descent on separable data.
\newblock In \emph{International Conference on Learning Representations}, 2018.
\newblock URL \url{https://openreview.net/forum?id=r1q7n9gAb}.

\bibitem[Wu et~al.(2020)Wu, Dobriban, Ren, Wu, Li, Gunasekar, Ward, and
  Liu]{Wu+20}
Xiaoxia Wu, Edgar Dobriban, Tongzheng Ren, Shanshan Wu, Zhiyuan Li, Suriya
  Gunasekar, Rachel Ward, and Qiang Liu.
\newblock Implicit regularization and convergence for weight normalization.
\newblock In H.~Larochelle, M.~Ranzato, R.~Hadsell, M.~F. Balcan, and H.~Lin,
  editors, \emph{Advances in Neural Information Processing Systems}, volume~33,
  pages 2835--2847. Curran Associates, Inc., 2020.
\newblock URL
  \url{https://proceedings.neurips.cc/paper/2020/file/1de7d2b90d554be9f0db1c338e80197d-Paper.pdf}.

\bibitem[Zhang et~al.(2017)Zhang, Bengio, Hardt, Recht, and Vinyals]{Zhang+17}
Chiyuan Zhang, Samy Bengio, Moritz Hardt, Benjamin Recht, and Oriol Vinyals.
\newblock Understanding deep learning requires rethinking generalization.
\newblock In \emph{International Conference on Learning Representations}, 2017.
\newblock URL \url{https://openreview.net/forum?id=Sy8gdB9xx}.

\end{thebibliography}
